%% file: skh2014.tex
\documentclass[twoside]{article}
\pdfoutput=1 %
\usepackage[accepted]{aistats2015}

\usepackage{hyperref}
\usepackage{url}

\hypersetup{
  plainpages=false,
  colorlinks=true,   %
  linkcolor=blue,    %
  anchorcolor=blue,  %
  citecolor=blue,    %
  filecolor=blue,    %
  pagecolor=blue,    %
  urlcolor=blue,     %
  pdfview=FitH,               %
  pdfstartview=FitH,          %
  pdfpagelayout=SinglePage    %
}

\usepackage[pdftex]{graphicx}
\usepackage{subfigure}
\usepackage{amssymb,amsmath,color, bm}
\usepackage{booktabs,multirow,rotating,float} %
\usepackage{enumitem,lipsum}
\usepackage{amsthm}

\usepackage[round]{natbib}
\usepackage{multibib} %
\newcites{sup}{Supplementary References}

\usepackage{algcompatible}
\usepackage{algorithm}
\usepackage{lindsten,abbreviations}
\renewcommand\setY{\mathcal{Y}}
\newcommand\ssm{SSM\xspace}

\def \E{{\mathbb E}} %
\def \V { \mathbb{V} } %
 
\def \F {\mathcal{F} } %
\def \H {\mathcal{H} } %
\def \X { \mathcal{X}  } %
\def \M {\mathcal{M} } %

\newcommand{\stepsize}{\gamma} %
\newcommand{\phat}{\hat{p}}
\newcommand{\wb}{\bm{w}}

\newcommand{\MMD}{\mathop{\rm MMD}}
\newcommand{\innerProd}[2]{\langle #1 , #2 \rangle}

\newcommand{\idm}{I}
\newcommand{\rb}{\mathbb{R}}
\newcommand{\BEAS}{\begin{eqnarray*}}
\newcommand{\EEAS}{\end{eqnarray*}}
\newcommand{\BEA}{\begin{eqnarray}}
\newcommand{\EEA}{\end{eqnarray}}
\newcommand{\BEQ}{\begin{equation}}
\newcommand{\EEQ}{\end{equation}}
\newcommand{\BIT}{\begin{itemize}}
\newcommand{\EIT}{\end{itemize}}
\newcommand{\BNUM}{\begin{enumerate}}
\newcommand{\ENUM}{\end{enumerate}}
\newcommand{\BA}{\begin{array}}
\newcommand{\EA}{\end{array}}

\renewenvironment{proof}{\par\noindent{\bf Proof\ }}{\hfill\BlackBox\\[2mm]}
\newenvironment{proofsketch}{\par\noindent{\bf Proof sketch.~}}{\hfill\BlackBox\\[2mm]}

\newtheorem{definition}{Definition}
\newtheorem{theorem}[definition]{Theorem}

\theoremstyle{remark}
\newtheorem{remark}{Remark}

\newcommand{\BlackBox}{\rule{1.5ex}{1.5ex}}  %

\newcommand{\kernel}{\kappa}
\newcommand{\teleConstant}{\chi}

\begin{document}

\runningtitle{Sequential Kernel Herding: Frank-Wolfe Optimization for Particle Filtering}

\twocolumn[

\aistatstitle{Sequential Kernel Herding: \\ Frank-Wolfe Optimization for Particle Filtering}

\aistatsauthor{ Simon Lacoste-Julien \And Fredrik Lindsten \And Francis Bach }

\aistatsaddress{ 
INRIA - Sierra Project-Team \\ 
\'Ecole Normale Sup\'erieure, Paris, France  
\And
Department of Engineering \\ 
University of Cambridge 
\And 
INRIA - Sierra Project-Team \\ 
\'Ecole Normale Sup\'erieure, Paris, France  
} ]

\begin{abstract}
Recently, the Frank-Wolfe optimization algorithm was suggested as a procedure to obtain adaptive quadrature rules for integrals of functions in a reproducing kernel Hilbert space (RKHS) with a potentially faster rate of convergence than Monte Carlo integration (and ``kernel herding'' was shown to be a special case of this procedure). In this paper, we propose to replace the random sampling step in a particle filter by Frank-Wolfe optimization. By optimizing the position of the particles, we can obtain better accuracy than random or quasi-Monte Carlo sampling.  In applications where the evaluation of the emission probabilities is expensive (such as in robot localization), the additional computational cost to generate the particles through optimization can be justified. Experiments on standard synthetic examples as well as on a robot localization task indicate indeed an improvement of accuracy over random and quasi-Monte Carlo sampling.
\end{abstract}

\vspace{-4mm}
\section{Introduction}
In this paper, we explore a way to combine ideas from \emph{optimization} with \emph{sampling} to get better approximations in probabilistic models. We consider 
state-space models (\ssm{s}, also referred to as general state-space hidden Markov models),
as they constitute an important class of models in engineering, econometrics and other areas involving time series and dynamical systems. A discrete-time,
nonlinear \ssm can be written as
\vspace{-3mm}
\begin{equation} \label{eq:ssm}
    x_{t} \, |\, x_{1:(t-1)} \sim p(x_t | x_{t-1}); \quad
    y_t \,|\, x_{1:t} \sim p(y_t | x_t),
\end{equation}
where $x_t\in\X$ denotes the latent state variable and $y_t\in\setY$ the observation at time $t$.
Exact state inference in \ssm{s} is possible, essentially, only when the model is linear and Gaussian or when the state-space
$\X$ is a finite set. For solving the inference problem beyond these restricted model classes,
sequential Monte Carlo methods, \ie particle filters (PFs), have emerged as a key tool; see \eg,
\citet{DoucetJ:2011,CappeMR:2005,DoucetGA:2000}. %
However, since these methods are based on Monte Carlo integration they are inherently affected
by sampling variance, which can degrade the performance of the estimators.

Particular challenges arise in the case when the \emph{observation likelihood} $p(y_t | x_t)$
is computationally expensive to evaluate.
For instance, this is common in robotics applications where the observation model relates the sensory input of the
robot, which can comprise vision-based systems, laser rangefinders, synthetic aperture radars, \etc.
For such systems, simply evaluating the observation function for a fixed value of $x_t$ can therefore
involve computationally expensive operations, such as image processing, point-set registration, and related tasks.
This poses difficulties for particle-filtering-based solutions for two reasons:
\emph{(1)} the computational bottleneck arising from the likelihood evaluation implies that we cannot simply
increase the number of particles to improve the accuracy, and
\emph{(2)} this type of ``complicated'' observation models will typically not allow for
adaptation of the proposal distribution used within the filter, in the spirit of \citet{PittS:1999},
leaving us with the standard---but inefficient---\emph{bootstrap proposal} as the only viable option.
On the contrary, for these systems, the \emph{dynamical model} $p(x_t | x_{t-1})$ is often comparatively simple, \eg being
a linear and Gaussian ``nearly constant acceleration'' model \citep{RisticAG:2004}.

The method developed in this paper is geared toward this class of filtering problems.
The basic idea is that, in scenarios when the likelihood evaluation is the computational bottleneck,
we can afford to spend additional computations to improve upon the sampling of the particles.
By doing so, we can avoid excessive variance arising from simple Monte Carlo sampling from the
bootstrap proposal.

\paragraph{Contributions.}
We build on the optimization view from~\citet{Bach2012} of kernel herding~\citep{Chen2010} to approximate the
integrals appearing in the Bayesian filtering recursions.
We make use of the Frank-Wolfe (FW) quadrature to approximate,
in particular, mixtures of Gaussians which often arise
in a particle filtering context as the mixture
over past particles in the distribution over the next state.
We use this approach within a filtering framework and
prove theoretical convergence results for the resulting
method, denoted as \emph{sequential kernel herding} (SKH), giving
one of the first explicit better convergence rates than for a particle filter. 
Our preliminary experiments show that SKH can give
better accuracy than a standard particle filter or a quasi-Monte Carlo
particle filter.

\section{Adaptive quadrature rules with Frank-Wolfe optimization}

\subsection{Approximating the mean element for integration in a RKHS}
We consider the problem of approximating integrals of functions belonging to a reproducing kernel Hilbert space (RKHS) $\H$ with respect to a \emph{fixed} distribution~$p$
over some set $\X$. We can think of the elements of $\H$ as being real-valued functions on $\X$, with pointwise evaluation given from the reproducing property by $f(x) = \innerProd{f}{\Phi(x)}$, where $\Phi: \X \to \H$ is the feature map from the state-space~$\X$ to the RKHS. Let $\kernel: \X^2 \to \reals$ be the associated positive definite kernel. We briefly review here the setup from~\citet{Bach2012}, which generalized the one from~\citet{Chen2010}. We want to approximate integrals $\E_p[f]$ for $f \in \H$ using a set of $n$ points $x^{(1)}, \ldots, x^{(n)} \in \X$ associated with positive weights $w^{(1)}, \ldots, w^{(n)}$ which sum to 1:
\vspace{-2mm}
\begin{align} 
	\E_p[f] &\approx \sum_{i=1}^n w^{(i)} f(x^{(i)}) = \E_{\hat{p}}[f], \label{eq:MCsum} %
\end{align}
where $\hat{p} := \sum_{i=1}^n w^{(i)} \delta_{x^{(i)}}$ is the associated empirical distribution defined by these points and $\delta_{x}(\cdot)$ is a point mass distribution at~$x$. 
If the points~$x^{(i)}$ are independent samples from~$p$, then this Monte Carlo estimate (using weights of~$1/n$) is unbiased with a variance of~$\V_p[f]/n$, where~$\V_p[f]$ is the variance of $f$ with respect to $p$. By using the fact that $f$ belongs to the RKHS~$\H$, we can actually choose a better set of points with lower error. It turns out that the worst-case error of estimators of the form~\eqref{eq:MCsum} can be analyzed in terms of their approximation distance to the \emph{mean element} $\mu(p) := \E_p [ \Phi ] \in \H$ \citep{Smola2007,Sriperumbudur2010}. Essentially, by using Cauchy-Schwartz inequality and the linearity of the expectation operator, we can obtain:
\vspace{-1mm}
\begin{multline} \label{eq:IntegralError}
  \sup_{\substack{
  		f \in  \H \\ 
  		\|f\|_\H \leq 1}
  		} 
  |\E_p[f] - \E_{\hat{p}}[f]| = \|\mu(p) - \mu(\hat{p})  \|_\H  \\[-3mm]
	 	 =: \MMD(p, \hat{p}), \quad
\end{multline} 
and so by bounding $\MMD(p, \hat{p})$, we can bound the error of approximating the expectation for all $f \in \H$, with $\|f\|_\H$ as a proportionality constant. $\MMD(p,\hat{p})$ is thus a central quantity for developing good quadrature rules given by~\eqref{eq:MCsum}. In the context of RKHSs, $\MMD(p,q)$ can be called the \emph{maximum mean discrepancy} \citep{gretton2012MMD} between the distributions $p$ and $q$, and acts a pseudo-metric on the space of distributions on $\X$. If $\kernel$ is a \emph{characteristic} kernel (such as the standard RBF kernel), then $\MMD$ is in fact a metric, i.e. $\MMD(p,q) = 0 \implies p = q$. We refer the reader to~\citet{Sriperumbudur2010} for the regularity conditions needed for the existence of these objects and for more details.

\subsection{Frank-Wolfe optimization for adaptive quadrature}
For getting a good quadrature rule $\hat{p}$, our goal is thus to minimize $\| \mu(\hat{p}) - \mu(p)\|_\H$. We note that $\mu(p)$ lies in the \emph{marginal polytope} $\M \subset \H$, defined as the closure of the convex-hull of $\Phi(\X)$. We suppose that $\Phi(x)$ is uniformly bounded in the feature space, that is, there is a finite $R$ such that $\|\Phi(x)\|_\H \leq R$ $\forall x \in \X$. This means that~$\M$ is a closed bounded convex subset of~$\H$, and we could in theory optimize over it. This insight was used by~\citet{Bach2012} who considered using the Frank-Wolfe optimization algorithm to optimize the convex function $J(g) := \frac{1}{2} \| g - \mu(p) \|_\H^2$ over~$\M$ to obtain adaptive quadrature rules. The Frank-Wolfe algorithm (also called conditional gradient) \citep{frank56FW} is a simple first-order iterative constrained optimization algorithm for optimizing \emph{smooth} functions over \emph{closed bounded convex} sets like~$\M$ (see~\citet{Dunn1980} for its convergence analysis on general infinite dimensional Banach spaces). At every iteration, the algorithm finds a good feasible search \emph{vertex} of~$\M$ by minimizing the \emph{linearization} of~$J$ at the current iterate $g_k$: $\bar{g}_{k+1} = \argmin_{g \in \M} \innerProd{J'(g_k)}{g}$. The next iterate is then obtained by a suitable convex combination of the search vertex $\bar{g}_{k+1}$ and the previous iterate $g_k$: $g_{k+1} = (1-\stepsize_k) g_k + \stepsize_k \bar{g}_{k+1}$ for a suitable step-size $\stepsize_k$ from a fixed schedule (e.g. $1/(k+1)$) or by using line-search. A crucial property of this algorithm is that the iterate $g_k$ is thus a convex combination of the \emph{vertices} of~$\M$ visited so far. This provides a \emph{sparse} expansion for the iterate, and makes the algorithm suitable to high-dimensional optimization (or even infinite) -- this explains in part the regain of interest in machine learning in the last decade for this old optimization algorithm (see~\citet{jaggi13FW} for a recent survey). In our setup where~$\M$ is the convex hull of~$\Phi(\X)$, the vertices of~$\M$ are thus of the form $\bar{g}_{k+1} = \Phi(x^{(k+1)})$ for some $x^{(k+1)} \in \X$. Running Frank-Wolfe on~$\M$ thus yields $g_k = \sum_{i = 1}^k w_k^{(i)} \Phi(x^{(i)}) = \E_{\hat{p}} [\Phi ]$ for some weighted set of points $\{w_k^{(i)}, x^{(i)}\}_{i=1}^k$. The iterate $g_k$ thus corresponds to a quadrature rule $\hat{p}$ of the form of~\eqref{eq:MCsum} and $g_k = \E_{\hat{p}} [\Phi ]$, and this is the relationship that was explored in~\citet{Bach2012}. Running Frank-Wolfe optimization with the step-size of $\stepsize_k = 1/(k+1)$ reduces to the kernel herding algorithm proposed by~\citet{Chen2010}. See also~\citet{Huszar2012} for an alternative approach with negative weights.

\begin{figure*}[t!]
	\vspace{-2mm}
	\begin{minipage}[t]{\columnwidth}
	\begin{algorithm}[H]
	  \caption{FW-Quad($p$, $\H$, $N$): Frank-Wolfe adaptive quadrature}
	  \label{alg:FWquad}
	\begin{algorithmic}[1]
	\STATEx \textbf{Input:} distribution $p$, RKHS $\H$ which defines kernel $\kernel(\cdot, \cdot)$ and state-space $\X$, number of samples $N$
	\STATE Let $g_0 = 0$.
	\FOR{ $k = 0\ldots N-1$}
		\STATE Solve $x^{(k+1)} = {\displaystyle\argmin_{x \in \X} \innerProd{g_k - \mu_p}{\Phi(x)}}$
		\STATEx That is:
		\STATEx $x^{(k+1)} = {\displaystyle\argmin_{x \in \X} \sum_{i=1}^k w_k^{(i)} \, ( \kernel(x^{(i)}, x)-\mu_p(x))}$.
		\STATE Option (1): Let $\stepsize_k = \frac{1}{k+1}.$
		\STATE Option (2): Let $\stepsize_k = \frac{\innerProd{g_k-\mu_p}{g_k - \Phi( x^{(k+1)}) }} {\| g_k - \Phi(x^{(k+1)}) \|^2} \,\,$ (LS)
		\STATE Update $g_{k+1} = (1-\stepsize_k) g_k + \stepsize_k \Phi( x^{(k+1)})$
		\STATEx \quad i.e. $w_{k+1}^{(k+1)} = \stepsize_k$; 
		\STATEx \quad and $w_{k+1}^{(i)} = (1-\stepsize_k) w_{k}^{(i)}$ for $i = 1\ldots k$
	\ENDFOR
	\STATE \textbf{Return:} $\phat = \sum_{i=1}^N w_N^{(i)} \delta_{x^{(i)}}$
	\end{algorithmic}
	\end{algorithm}

\end{minipage}
\hfill
\begin{minipage}[t]{\columnwidth}
\begin{algorithm}[H]
  \caption{Particle filter template (joint predictive distribution form) --- SKH alg. by changing step 3}
  \label{alg:PF}
\begin{algorithmic}[1]
\STATEx \textbf{Input:} \ssm $p(x_t|x_{t-1})$, 
\STATEx \quad\quad $o_t(x_t) := p(y_t|x_t)$ for $t \in 1:T$.
\STATEx Maintain $\hat{p}_t(x_{1:t}) \! = \!\sum_{i=1}^N \! w_t^{(i)} \! \delta_{x_{1:t}^{(i)}}\!(x_{1:t})$ 
during algorithm as approximation of $p(x_t, x_{1:(t-1)} | y_{1:(t-1)})$.
\STATE Let $\tilde{p}_1(x_1) := p(x_1)$
\FOR{t=1\, \ldots, T}
\STATE Sample: get $\hat{p}_t = \mathrm{SAMPLE}(\tilde{p}_{t}, N)$
\STATEx \quad\quad [For SKH, use $\hat{p}_t = \textrm{FW-Quad}(\tilde{p}_t, \H_t, N)$]
\STATE Include observation and normalize:
\STATEx \quad\quad $\hat{W}_t = \E_{\hat{p}_t}[o_t]$; \quad $\hat{r}_t(x_{1:t}) := \frac{1}{\hat{W}_t} o_t(x_t) \hat{p}_t(x_{1:t})$.
\STATE Propagate approximation forward: 
\STATEx \quad\quad $\tilde{p}_{t+1}(x_{t+1}, x_{1:t}) := p(x_{t+1} | x_t) \hat{r}_t(x_{1:t})$
\ENDFOR
\STATE \textbf{Return} Filtering distribution $\hat{r}_T$; predictive distribution $\hat{p}_{T+1}$; normalization constants $\hat{W}_1, \ldots, \hat{W_T}$.
\end{algorithmic}
\end{algorithm}
 \end{minipage}
\hfill
\end{figure*}

Algorithm~\ref{alg:FWquad} presents the Frank-Wolfe optimization algorithm to solve $\min_{g \in \M} J(g)$ in the context of getting quadrature rules (we also introduce the shorthand notation $\mu_p := \mu(p)$). We note that to evaluate the quality $\MMD(\hat{p}, p)$ of this adaptive quadrature rule, we need to be able to evaluate $\mu_p(x) = \int_{x' \in \X} p(x') \kernel(x',x) dx'$ efficiently. This is true only for specific pairs of kernels and distributions, but fortunately this is the case when $p$ is a mixture of Gaussians and $\kernel$ is a Gaussian kernel. This insight is central to this paper; we explore this case more specifically 
in Section~\ref{sub:MoG}. To find the next quadrature point, we also need to (approximately) optimize $\mu_p(x)$ over $\X$ (step~3 of Algorithm~\ref{alg:FWquad}, called the FW vertex search). In general, this will yield a non-convex optimization problem, and thus cannot be solved with guarantees, even with gradient descent. In our current implementation, we approach step~3 by doing an exhaustive search over $M$ random samples from $p$ precomputed when FW-Quad is called. We thus follow the idea from the kernel herding paper~\citep{Chen2010} to choose the best $N$ ``super-samples'' out of a large set of samples $M$. 
Thanks to the fact that convergence guarantees for Frank-Wolfe optimization can still be given when using an approximate FW vertex search, 
we show in Appendix~\ref{app:rateRandomSearch} of the 
supplementary material that this procedure either adds a $O(1/M^{1/4})$ term or a $O(1/\sqrt{M})$ term to the worst-case $\MMD(\hat{p}, p)$ error.

In our description of Algorithm~\ref{alg:FWquad}, a preset number~$N$ of particles (iterations) was used. Alternatively, we could use a variable number of iterations with the terminating criterion test $\| g_k - \mu(p)\|_\H \leq \epsilon$ which can be \emph{explicitly computed during the algorithm} and provides the $\MMD$ error bound on the returned quadrature rule. Option (2) on line~5 chooses the step-size $\stepsize_k$ by analytic line-search (hereafter referred as the FW-LS version) while option (1) chooses the kernel herding step-size $\stepsize_k = 1/(k+1)$ (herafter referred as the FW version) which always yields uniform weights: $w_k^{(i)} = 1/k$ for all $i \leq k$. A third alternative is to re-optimize $J(g)$ over the convex hull of the previously visited vertices; this is called the fully corrective version~\citep{jaggi13FW} of the Frank-Wolfe algorithm (hereafter referred as FCFW). In 
this case: $(w_{k+1}^{(1)}, \ldots, w_{k+1}^{(k+1)}) = \argmin_{\wb \in \Delta_{k+1}} \wb^\top \bm{K}_{k+1} \wb - 2 \bm{c}_{k+1}^\top \wb$, where $\Delta_{k+1}$ is the $(k+1)$-dimensional probability simplex, $\bm{K}_{k+1}$ is the kernel matrix on the $(k+1)$ vertices: $(\bm{K}_{k+1})_{ij} = \kernel(x^{(i)}, x^{(j)})$ and $(\bm{c}_{k+1})_i = \mu_p(x^{(i)})$ for $i = 1, \ldots, {(k+1)}$. This is a convex quadratic problem over the simplex. A slightly modified version of the FCFW is called the min-norm point algorithm and can be more efficiently optimized using specific purpose active-set algorithms --- see~\citet[\S 9.2]{bach13sub} for more details. We refer the reader to~\citet{Bach2012} for more details on the rate of convergence of Frank-Wolfe quadrature assuming that the FW vertex is found with guarantees. We summarize them as follows: if $\H$ is infinite dimensional, then FW-Quad gives the same $O(1/\sqrt{N})$ rate for the MMD error as standard random sampling, for all FW methods. On the other hand, if a ball of non-zero radius centered at $\mu_p$ lies within $\M$, then faster rates than random sampling are possible: FW gives a $O(1/N)$ rate whereas FW-LS and FCFW gives exponential convergence rates (though in practice, we often see differences not explained by the theory between these methods).

\subsection{Example: mixture of Gaussians} \label{sub:MoG}

We describe here in more details the Frank-Wolfe quadrature when $p$ is a mixture of Gaussians $p(x) = \sum_{i=1}^K \pi_i \N(x | \mu_i, \Sigma_i)$ for $\X = \reals^d$ and $\kernel$ is the Gaussian kernel $\kernel_\sigma(x,x') := \exp(-\frac{1}{2\sigma^2}\|x-x'\|^2)$. In this case,  $\mu_p(x) = \sum_{i=1}^K \pi_i (\sqrt{2 \pi}\sigma)^d \N(x | \mu_i, \Sigma_i + \sigma^2 \idm_d)$. 
We thus need to optimize a difference of mixture of Gaussian bumps in step~3 of Algorithm~\ref{alg:FWquad}, a non-convex optimization problem that we approximately solve by exhaustive search over $M$ random samples from $p$.

\section{Sequential kernel herding}
\subsection{Sequential Monte Carlo}
Consider again the \ssm in \eqref{eq:ssm}. The joint probability density function for a
sequence of latent states $x_{1:T} := (x_1,\ldots,x_T)$ and observations $y_{1:T}$ factorizes as $p(x_{1:T}, y_{1:T}) = \prod_{t=1}^T p(x_t|x_{t-1}) p(y_t | x_t)$, with $p(x_1|x_0) := p(x_1)$ denoting the prior density on the initial state.
We would like to do approximate inference in this \ssm. In particular, we could be interested in computing the joint filtering distribution $r_t(x_{1:t}) := p(x_{1:t} | y_{1:t})$ or the joint predictive distribution $p_{t+1}(x_{t+1}, x_{1:t}) := p(x_{t+1}, x_{1:t}| y_{1:t})$. In particle filtering methods, we approximate these distributions with empirical distributions from weighted particle sets $\{w_t^{(i)}, x_{1:t}^{(i)}\}_{i=1}^N$ as in~\eqref{eq:MCsum}. 
We note that it is easy to marginalize $\hat{p}$ with a simple weight summation, and so we will present the algorithm as getting an approximation for the \emph{joint} distributions $r_t$ and $p_t$ defined above, with the understanding that the marginal ones are easy to obtain afterwards. 
In the terminology of particle filtering, $x_t^{(i)}$ is the particle at time $t$, whereas $x_{1:t}^{(i)}$ is the \emph{particle trajectory}. While principally the PF provides an approximation
of the full joint distribution $r_t(x_{1:t})$, it is well known that this approximation deteriorates
for any marginal of $x_s$ for $s \ll t$ \citep{DoucetJ:2011}. Hence, the PF
is typically only used to approximate marginals of $x_s$ for $s \lesssim t$ (fixed-lag smoothing) or $s = t$ (filtering), or for prediction. 

Algorithm~\ref{alg:PF} presents the bootstrap particle filtering algorithm~\citep{GordonSS:1993} from the point of view of propagating an approximate posterior distribution forward in time~\citep[see e.g.][]{fearnhead05SQMC}. We describe it as propagating an approximation $\phat_t(x_{1:t})$ of the joint predictive distribution one time step forward with the model dynamics to obtain $\tilde{p}_{t+1}(x_{t+1}, x_{1:t})$ (step 5), and then randomly sampling from it (step 3) to get the new predictive approximation $\hat{p}_{t+1}(x_{t+1}, x_{1:t})$.
As $\phat_t$ is an empirical distribution, $\tilde{p}_{t+1}$ is a mixture distribution (the mixture components are coming from the particles at time~$t$): 
\vspace{-2mm}
\begin{multline} \label{eq:transitionMixture}
\tilde{p}_{t+1}(x_{t+1}, x_{1:t})= \\ 
	\frac{1}{\hat{W}_t} \sum_{i=1}^N \underbrace{p(y_t | x_{t}^{(i)}) w_t^{(i)}}_{\textrm{mixture weight}}
  \underbrace{p(x_{t+1} | x_{t}^{(i)}) }_{\textrm{mixture component}} \!\!\!\!\delta_{x_{1:t}^{(i)}}(x_{1:t}).
\end{multline}
We denote the conditional normalization constant at time $t$ by $W_t := p(y_t | y_{1:(t-1)})$ and the global normalization constant by $Z_t := p(y_{1:t}) = \prod_{u=1}^t W_u$. $\hat{W}_t$ is the particle filter approximation to $W_t$ and is obtained by summing the un-normalized mixture weights in~\eqref{eq:transitionMixture}; see step~4 in Algorithm~\ref{alg:PF}. Randomly sampling from~\eqref{eq:transitionMixture} is equivalent to first sampling a mixture component according to the mixture weight (i.e., choosing a past particle $x_{1:t}^{(i)}$ to propagate),
and then sampling its next extension state $x_{t+1}^{(i)}$ with probability $p(x_{t+1}|x_{t}^{(i)})$. The standard bootstrap particle filter is thus obtained by maintaining uniform weight for the predictive distribution ($w_t^{(i)}=\frac{1}{N}$) and randomly sampling from~\eqref{eq:transitionMixture} to obtain the particles at time $t+1$. This gives an unbiased estimate of $\tilde{p}_{t+1}$: $\E_{\tilde{p}_{t+1}}[\hat{p}_{t+1}] = \tilde{p}_{t+1}$. Lower variance estimators can be obtained by using a different resampling mechanism for the particles than this multinomial sampling scheme, such as stratified resampling~\citep{CarpenterCF:1999} and are usually used in practice instead.

One way to improve the particle filter is thus to replace the random sampling stage of step 3 with different sampling mechanisms with lower variance or better approximation properties of the distribution $\tilde{p}_{t+1}$ that we are trying to approximate.
As we obtain the normalization constants $W_t$ by integrating the observation probability, it seems natural to look for particle point sets with better integration properties.
By replacing random sampling with a quasi-random number sequence, 
we obtain the already proposed sequential quasi-Monte Carlo scheme~\citep{Philomin2000,Ormoneit2001,gerber2014sqmc}. %
The main contribution of our work is to instead propose to use Frank-Wolfe quadrature in step 3 of the particle filter to obtain better (adapted) point sets.

\subsection{Sequential kernel herding}
In the sequential kernel herding (SKH) algorithm, we simply replace step~3 of Algorithm~\ref{alg:PF} with $\hat{p}_t = \textrm{FW-Quad}(\tilde{p}_t, \H_t, N)$. As mentioned in the introduction, many dynamical models used in practice assume Gaussian transitions. Therefore, we will put particular emphasis on the case when (more generally) $p(x_t|x_{1:(t-1)}, y_{1:(t-1)})$ is a mixture of Gaussians, with parameters for the mixture components that can be arbitrary functions of the state history $x_{1:(t-1)}, y_{1:(t-1)}$, and is thus still fairly general. We thus consider the Gaussian kernel for the FW-Quad procedure as then we can compute the required quantities analytically. An important subtle point is which Hilbert space $\H_t$ to consider. In this paper, we focus on the \emph{marginalized} filtering case, i.e. we are interested in $p(x_t | y_{1:t})$ only. Thus we are only interested in functions of $x_t$, which is why we define our kernel at time $t$ to only depend on $x_t$ and not the past histories. For simplicity, we also assume that $\H_t = \H$ for all $t$ (we use the same kernel for each time step). Even though the algorithm can maintain the distribution on the whole history $\hat{p}_t(x_{1:t})$, the past histories $x_{1:(t-1)}$ are marginalized out when computing the mean map, for example $\mu(\tilde{p}_t) = \E_{\tilde{p}_t(x_{1:t})}[\Phi(x_t)]$. During the SKH algorithm, we can still track the particle histories by keeping track from which mixture component in~\eqref{eq:transitionMixture} $x_t$ was coming from, but the past history is not used in the computation of the kernel and thus does not appear as a repulsion term in step 3 of Algorithm~\ref{alg:FWquad}. We leave it as future work to analyze what kind of high-dimensional kernel on past histories would make sense in this context, and to analyze its convergence properties. The particle histories are useful in the Rao-Blackwellized extension that we present in Appendix~\ref{sec:RB} and use in the robot localization experiment of Section~\ref{sec:uav}.

\subsection{Convergence theory}
In this section, we give sufficient conditions to guarantee that SKH is consistent as $N$ goes to infinity. Let $p_t$ here denote the \emph{marginalized} predictive instead of the joint.
Let $F_t$ be the forward transformation operator on signed measures that takes the predictive distribution $p_t$ on $x_t$ and yields the un-normalized marginalized predictive distribution $F_t p_t$ on $x_{t+1}$ in the \ssm. Thus for a measure $\nu$, we get $(F_t \nu)(\cdot) := \int_{\X_t} p(\cdot | x_t) p(y_t | x_t) d\nu(x_t)$. 
We also have that $p_{t+1} = \frac{1}{W_t}F_t p_t$.

For the following theorem, $\F_t$ is a function space on $\X_{t+1}$ defined (depending on $\H_{t+1}$) as all functions for which the following semi-norm is finite:\footnote{In general, the integral on $\X_{t+1}$ should be with respect to the base measure for which the conditional density $p(x_{t+1}|x_t)$ is defined. All proofs are in the supplementary material.}
$$
\|f\|_{\F_t} := \sup_{\|h\|_{\H_{t+1}}=1} \bigg|\int_{\X_{t+1}} f(x_{t+1}) h(x_{t+1}) dx_{t+1} \bigg|.
$$

\begin{theorem}[Bounded growth of the mean map] \label{thm:Ct}
Suppose that the function $f_t : (x_{t+1},x_t) \mapsto p(y_t|x_t) p(x_{t+1}|x_t)$ is in the tensor product function space $\F_t \otimes \H_t$ with the following defined nuclear norm: $\| f_t \|_{\F_t \otimes \H_t} := \inf \sum_{i} \|\alpha_i\|_{\F_t} \|\beta_i\|_{\H_t}$, where the infimum is taken over all the possible expansions such that $f_t(x_{t+1},x_t) = \sum_i \alpha_i(x_{t+1}) \beta_i (x_t)$ for all $x_t, x_{t+1}$. Then for any finite signed Borel measure $\nu$ on $\X_t$, we have:
$$
\| \mu(F_t \nu) \|_{\H_{t+1}} \leq \| f_t \|_{\F_t \otimes \H_t} \, \| \mu(\nu) \|_{\H_t} .
$$
\end{theorem}

\begin{theorem}[Consistency of SKH] \label{thm:rate}
Suppose that for all $1 \leq t \leq T$, $f_t$ is in $\F_t \otimes \H_t$ as defined in Theorem~\ref{thm:Ct} and $o_t$ is in $\H_t$. Then we have:\footnote{We use the convention that the empty sum is $0$ and the empty product is $1$.}
\begin{multline*}
\| \mu(\hat{p}_T) - \mu(p_T) \|_{\H_T} \leq \\ 
	\hat{\epsilon}_T + \left(R \frac{\|o_{T-1}\|_{\H_{T-1}}}{W_{T-1}} + \rho_{T-1}\right)
 	\sum_{t=1}^{T-1} \teleConstant_t \, \hat{\epsilon}_t \left(\prod_{k=t}^{T-2} \rho_k \right) , 
\end{multline*}
where $\rho_t := \frac{\| f_t \|_{\F_t \otimes \H_t}}{W_t}$, $\teleConstant_t := \prod_{k=1}^{t-1} \frac{\hat{W_k}}{W_k}$ and $\hat{\epsilon}_t$ is the FW error reported at time $t$ by the algorithm: $\hat{\epsilon}_t := \| \mu(\hat{p}_t) - \mu(\tilde{p}_t) \|_{\H_t}$.
\end{theorem}

We note that $\teleConstant_t \approx 1$ as we expect the errors on $W_k$ to go in either direction, and thus to cancel each other over time (though in the worst case it could grow exponentially in $t$). If $\hat{\epsilon}_t \leq \epsilon$ and $\rho_t \leq \rho$, we basically have $\| \mu(\hat{p}_T) - \mu(p_T) \| = O(\rho^T \epsilon)$ if $\rho > 1$; $O(T \epsilon)$ if $\rho = 1$; and $O(\epsilon)$ if $\rho < 1$ (a contraction). The exponential dependence in $T$ is similar as for a standard particle filter for general distributions; see~\citet{douc2014stability} though for conditions to get a contraction for the PF.

Importantly, for a fixed $T$ it follows that the rates of convergence for Frank-Wolfe in $N$ translates to rates of errors for integrals of functions in $\H$ with respect to the predictive distribution $p_T$. Thus if we suppose that $\H$ is finite dimensional, that $p_t$ has full support on $\X$ for all $t$ and that the kernel $\kernel$ is continuous, then by Proposition~1 in~\citet{Bach2012}, we have that the faster rates for Frank-Wolfe hold and in particular we could obtain an error bound of $O(1/N)$ with $N$ particles. As far as we know, this is the first explicit faster rates of convergence as a function of the number of particles than the standard $O(\frac{1}{\sqrt{N}})$ for Monte Carlo particle filters. In contrast, \citet[Theorem 7]{gerber2014sqmc} showed a $o(\frac{1}{\sqrt{N}})$ rate for the randomized version of their SQMC algorithm (note the little-o).\footnote{The rate holds on the approximation of integrals of continuous bounded functions.} Note that the theorem does not depend on how the error of $\epsilon$ is obtained on the mean maps of the distribution; and so if one could show that a QMC point set could also achieve a faster rate for the error on the \emph{mean maps} (rather than on the distributions itself as is usually given), then their rates would translate also to the global rate by Theorem~\ref{thm:rate}.\footnote{We also note that a simple computation shows that for a Monte Carlo sample of size $N$, $\E \| \mu(\hat{p}) - \mu(p) \|_\H^2 \leq \frac{(R^2- \| \mu(p) \|^2)}{N}$.}

\input{expts.tex}

\section{Conclusion}
\vspace{-3mm}
We have developed a method for Bayesian filtering problems using a combination
of optimization and particle filtering. The method has been demonstrated
to provide improved performance over both random sampling and quasi-Monte Carlo methods.
The proposed method is modular and it can be used with different
types of particle filtering techniques, such as the Rao-Blackwellized particle filter.
Further investigating this possibility for other classes of particle filters
is a topic for future work. Future work also includes a deeper analysis of the
convergence theory for the method in order to develop practical guidelines for the
choice of the kernel bandwidth.

\clearpage

\subsubsection*{Acknowledgements} We thank Eric Moulines for useful discussions. This work was partially supported by the MSR-Inria Joint Centre, a grant by the European Research Council (SIERRA project 239993) and by the Swedish Research Council (project \emph{Learning of complex dynamical systems} number 637-2014-466).

\bibliography{skh_refs}
\bibliographystyle{abbrvnat} %

\clearpage

\input{supplement.tex}

\end{document}

%% file: expts.tex
\vspace{-1mm}
\section{Experiments}

 \begin{figure}[tb]
   \centering
   \vspace{-2mm}   
   \includegraphics[width = 0.80\columnwidth]{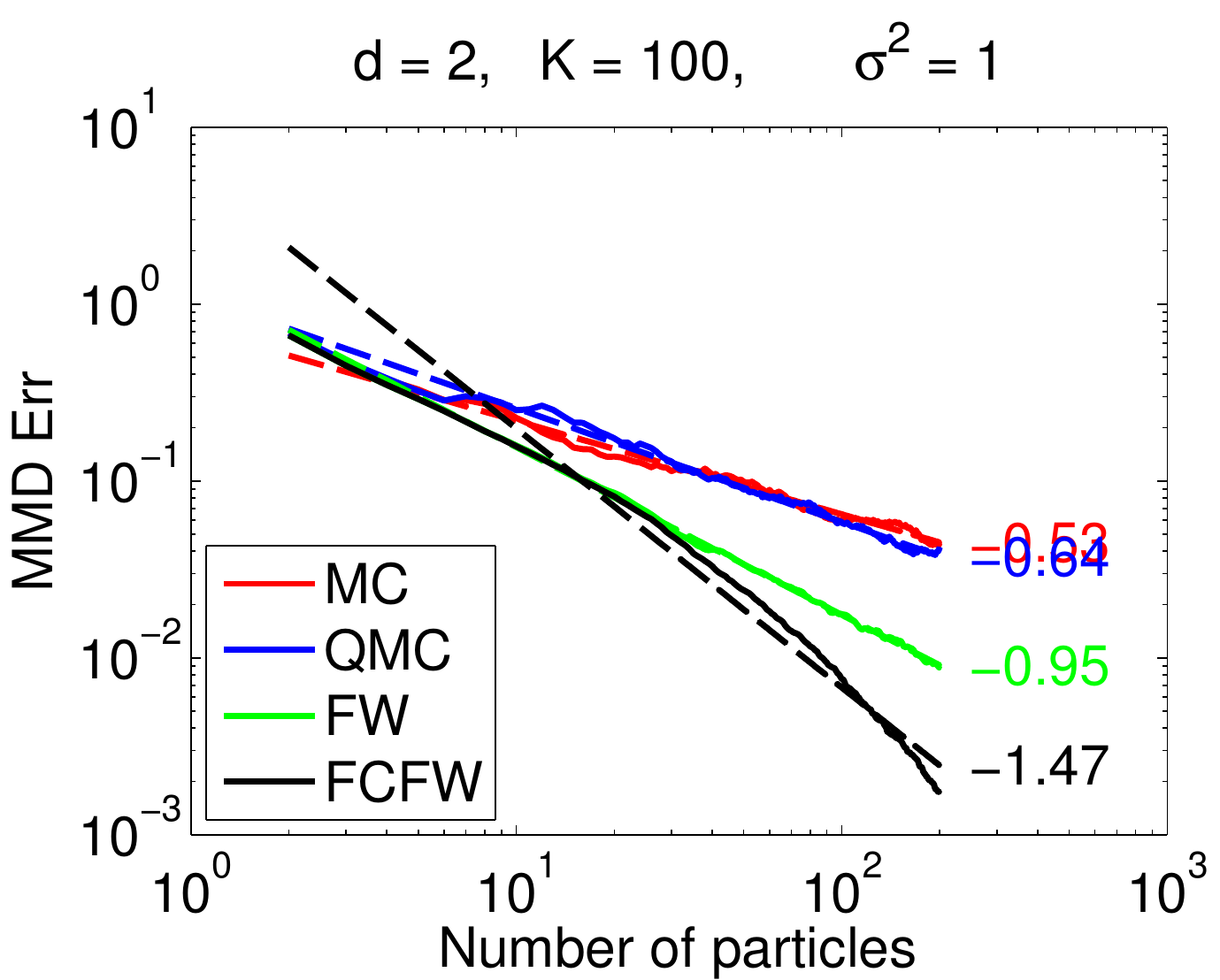} 
   \includegraphics[width = 0.80\columnwidth]{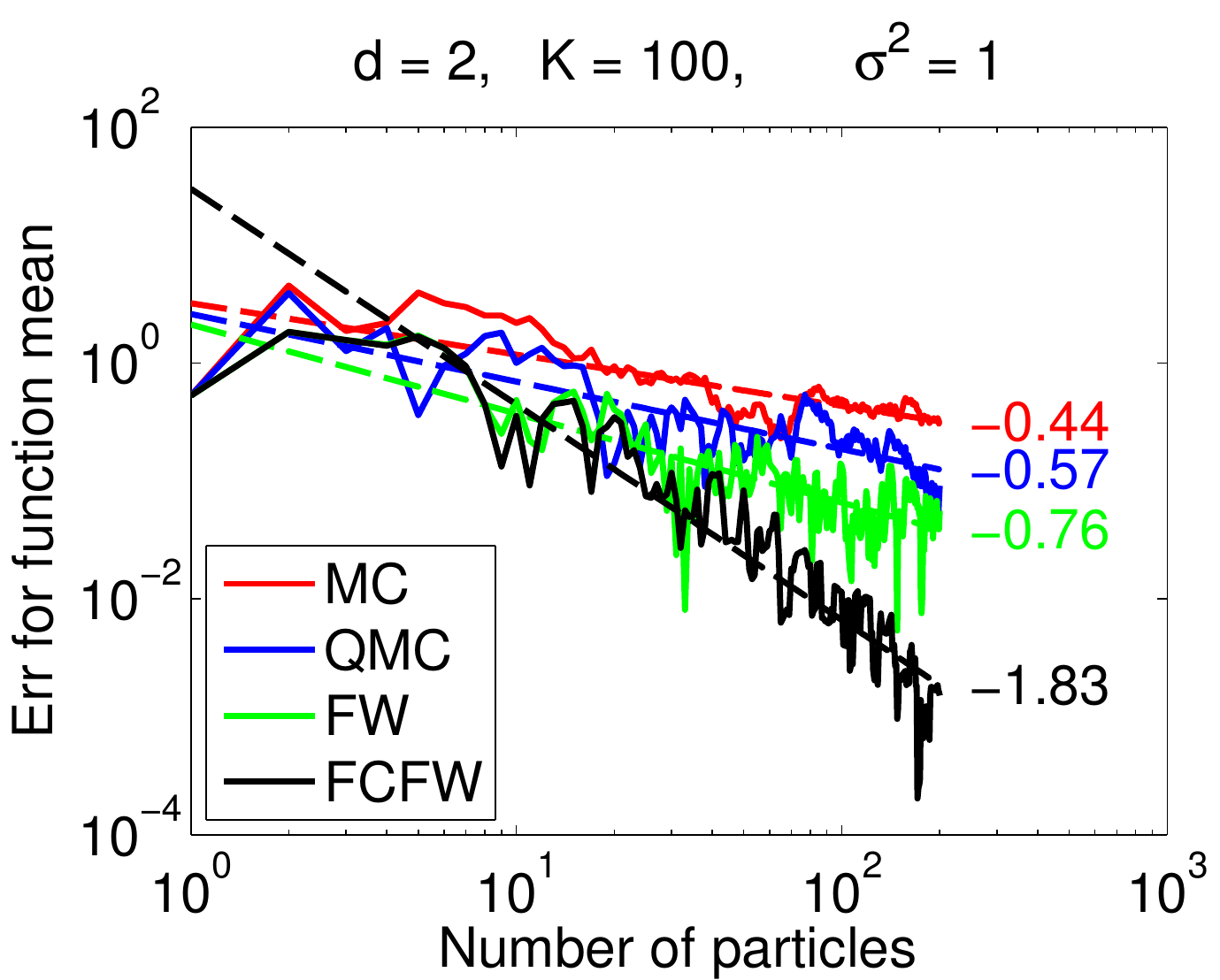}%
   \caption{Top: MMD error for different sampling schemes where $p$ is a mixture of 2d Gaussians with $K=100$ components. Bottom: error on the mean estimate for the same mixture. The dashed lines are linear fits with slopes reported next to the axes.\vspace{-2mm}\mbox{}}
   \label{fig:expts:gmm}
 \end{figure}

\vspace{-2mm}
\subsection{Sampling from a mixture of Gaussians} \label{exp:MoG}
We start by investigating the merits of different sampling schemes
for approximating mixtures of Gaussians, since this is an intrinsic step to the SKH algorithm. In Figure~\ref{fig:expts:gmm}, we give the $\MMD$ error as well as the error on the mean function in term of the number of particles $N$ for the different sampling schemes on a randomly chosen mixture of Gaussians with $K=100$ components in $d=2$ dimensions. Additional results as well as the details of the model are given in Appendix~\ref{app:MoG} of the supplementary material. In our experiments, the number of FW search points is $M=50,\!000$. We note that even though in theory all methods should have the same rate of convergence $O(1/\sqrt{N})$ for the $\MMD$ (as $\H$ is infinite dimensional), FCFW empirically improves significantly over the other methods. As $d$ increases, the difference between the methods tapers off for a fixed kernel bandwidth $\sigma^2$, but increasing $\sigma^2$ gives better results for FW and FCFW than the other schemes.

In the remaining sections, we evaluate empirically the application of kernel herding in a filtering context using the proposed SKH algorithm. 

\subsection{Particle filtering using SKH on synthetic examples}
\vspace{-1mm}
We consider first several synthetic data sets in order to assess
the improvements offered by Frank-Wolfe quadrature over standard Monte Carlo and quasi-Monte-Carlo techniques.
We generate data from four different systems
(further details on the experimental setup can be found in Appendix~\ref{app:synthetic}):
\vspace{-1.5mm}
\begin{description}[leftmargin = 2em, labelindent=0pt]
\item\vspace{-0.5mm}\textbf{Two linear Gaussian state-space (LGSS)} models of dimensions $d=3$ and $d=15$, respectively. 
\item\vspace{-0.5mm}\textbf{A jump Markov linear system (JMLS)}, consisting of 2 interacting LGSS models of dimension $d=2$.
  The switching between the models
  is governed by a \emph{hidden} 2-state Markov chain.
\item\vspace{-0.5mm}\textbf{A nonlinear benchmark} time-series model used by, among others, \cite{DoucetGA:2000,GordonSS:1993}. The model is
  of dimension $d=1$ and is given by: \\
  \vspace{-5mm}
  \begin{align*}
    x_{t+1} &= 0.5 x_t + 25 \frac{x_t}{1+x_t^2} + 8\cos(1.2t) + v_t,  \\ 
    y_t &= 0.05 x_t^2 + e_t,
  \end{align*}
  with $v_t$ and $e_t$ mutually independent standard Gaussian.
\end{description} \vspace{-2mm}
These models are ordered in increasing levels of difficulty for inference. For the LGSS models, the exact filtering distributions can be computed by a Kalman filter.
For the JMLS, this is also possible by running a mixture of Kalman filters, albeit at a computational
cost of $2^T$ (where $T$ is the total number of time steps). For the nonlinear system,
no closed form expressions are available for the filtering densities; instead we run a PF
with $\Np = 100,\!000$ particles as a reference.

We generate $30$ batches of observations for $\T=100$ time steps from all systems, except for the JMLS where we use $T=10$
(to allow exact filtering). We
run the proposed SKH filter, using both FW and FCFW optimization and compare against
a bootstrap PF (using stratified resampling \citep{CarpenterCF:1999}) and a quasi-Monte-Carlo PF based on
a Sobol-sequence point-set. All methods are run with $\Np$ varying from $20$ to $200$ particles.
We deliberately use rather few particles since, as discussed above, we believe
that this is the setting when the proposed method can be particularly useful.

\begin{figure*}[!t]
  \vspace{-2.5mm}
  \centering
  \includegraphics[width = 0.32\textwidth]{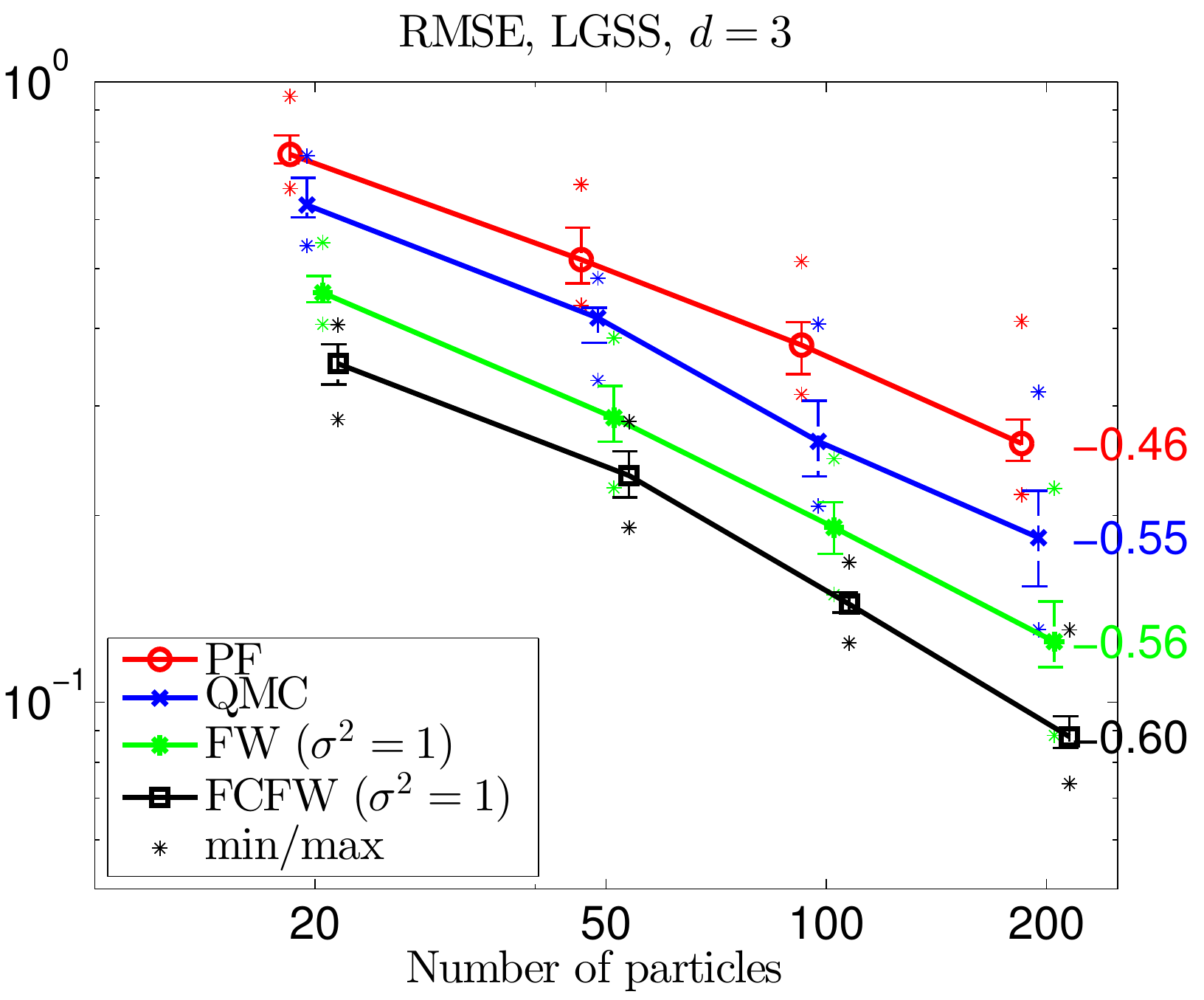} 
  \includegraphics[width = 0.32\textwidth]{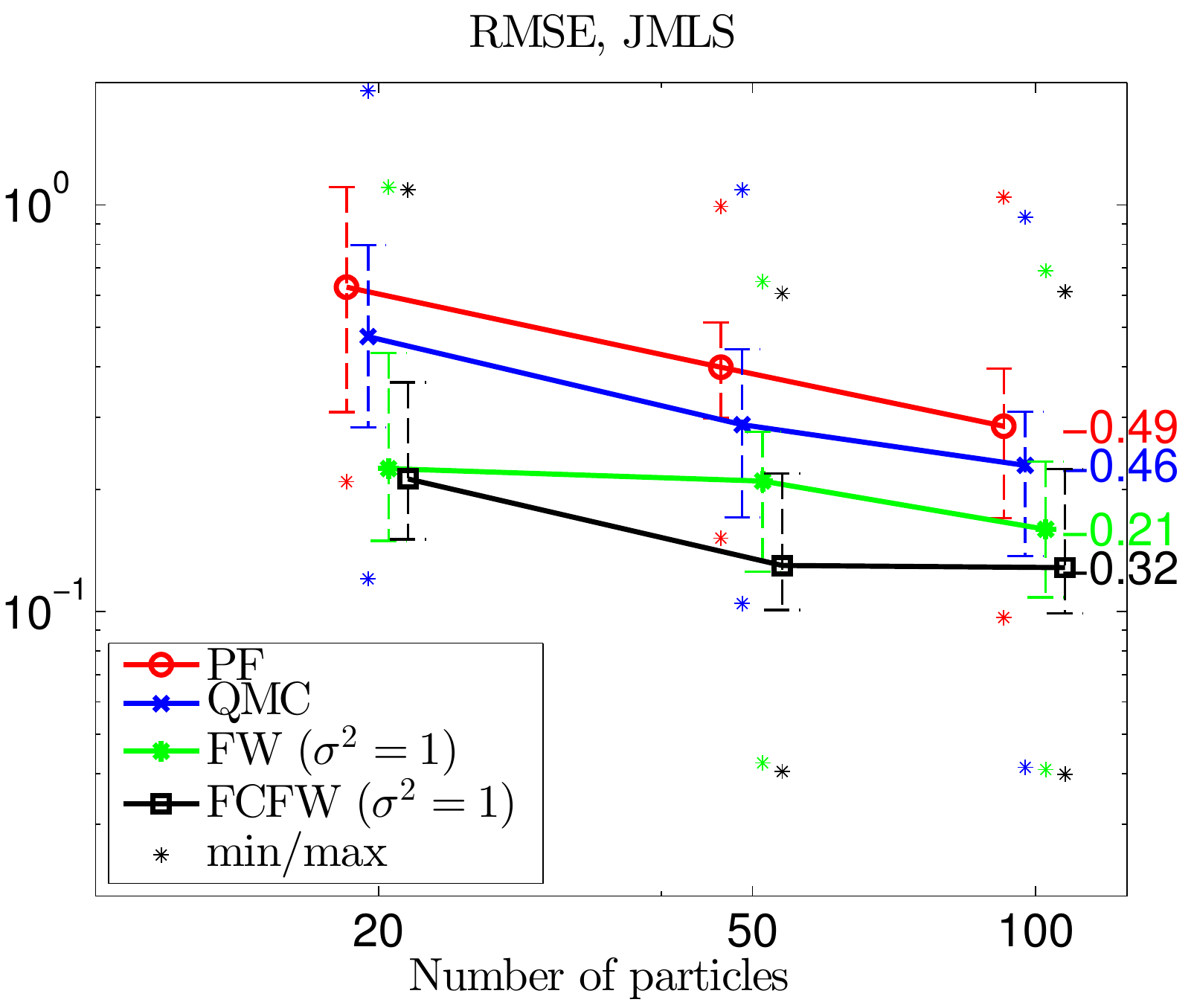} 
  \includegraphics[width = 0.32\textwidth]{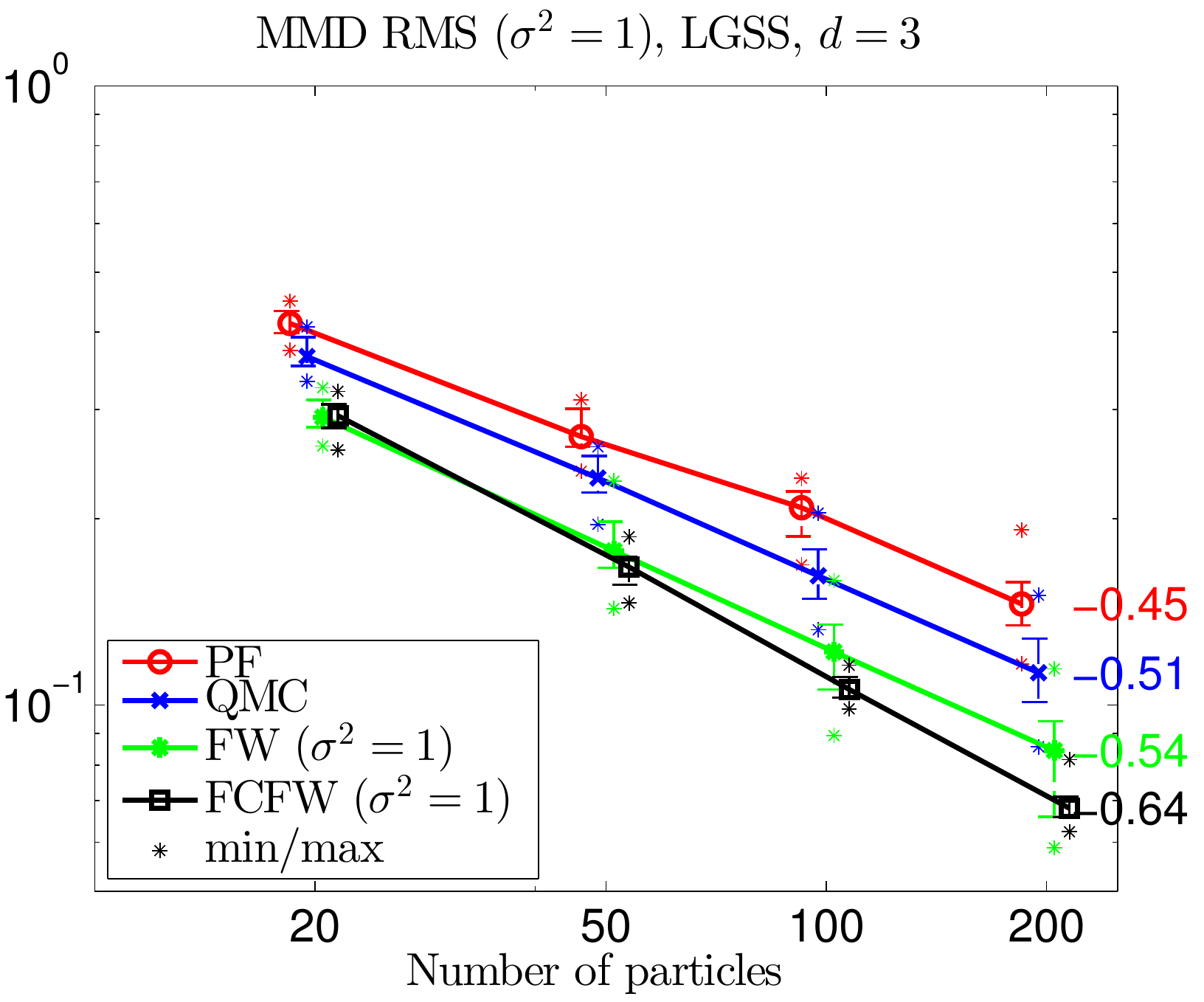}\\%
  \includegraphics[width = 0.32\textwidth]{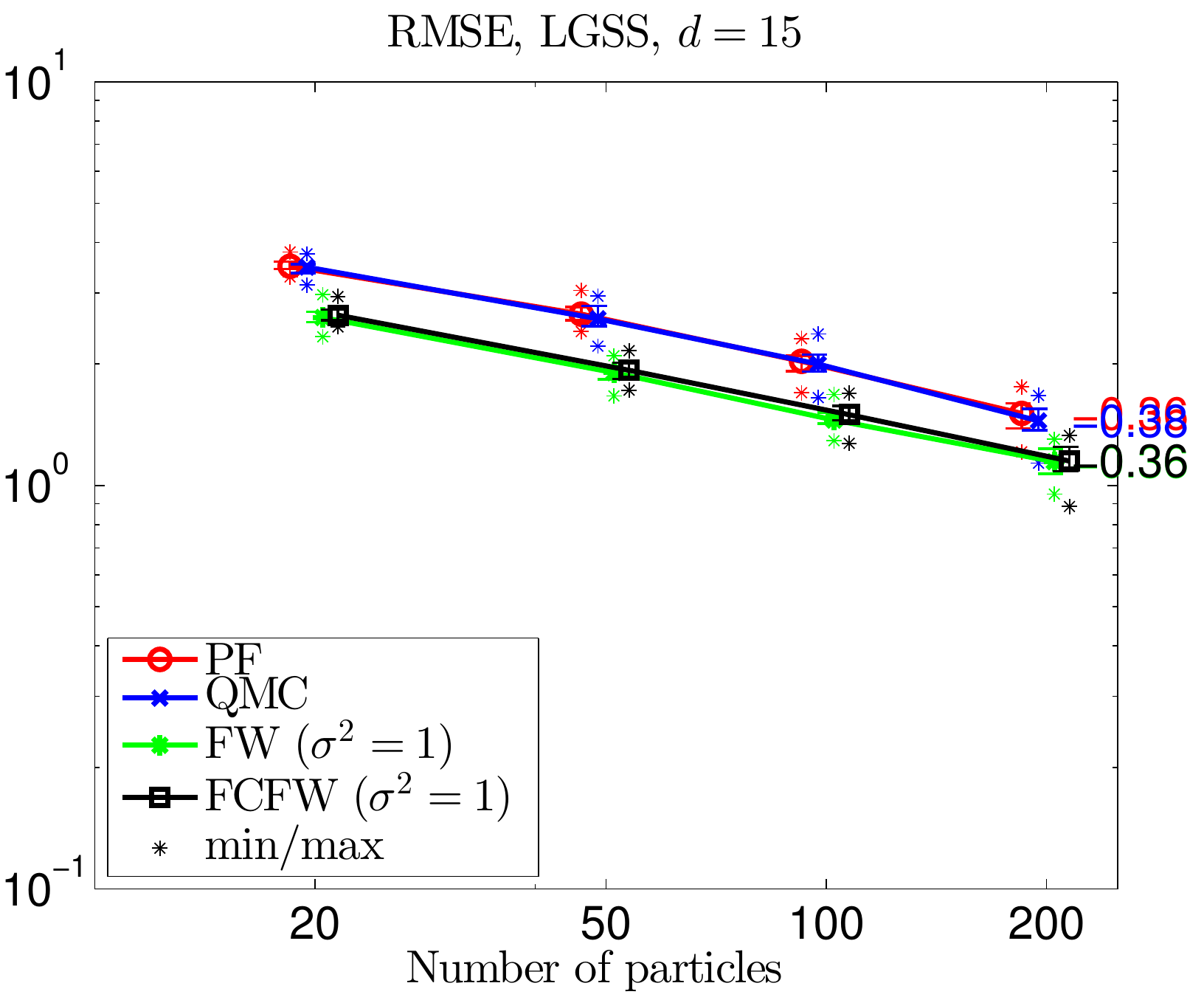} 
  \includegraphics[width = 0.32\textwidth]{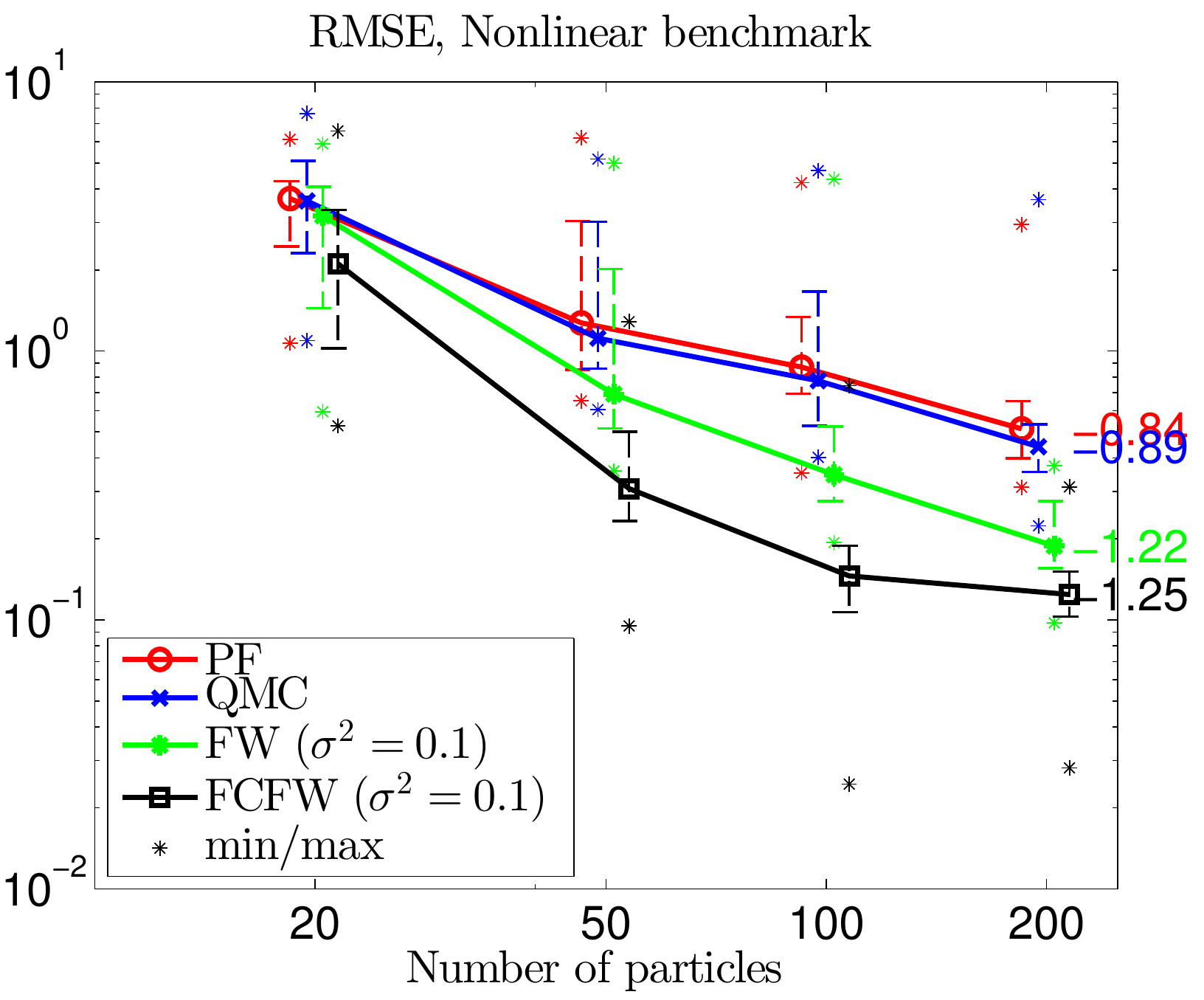} 
  \includegraphics[width = 0.32\textwidth]{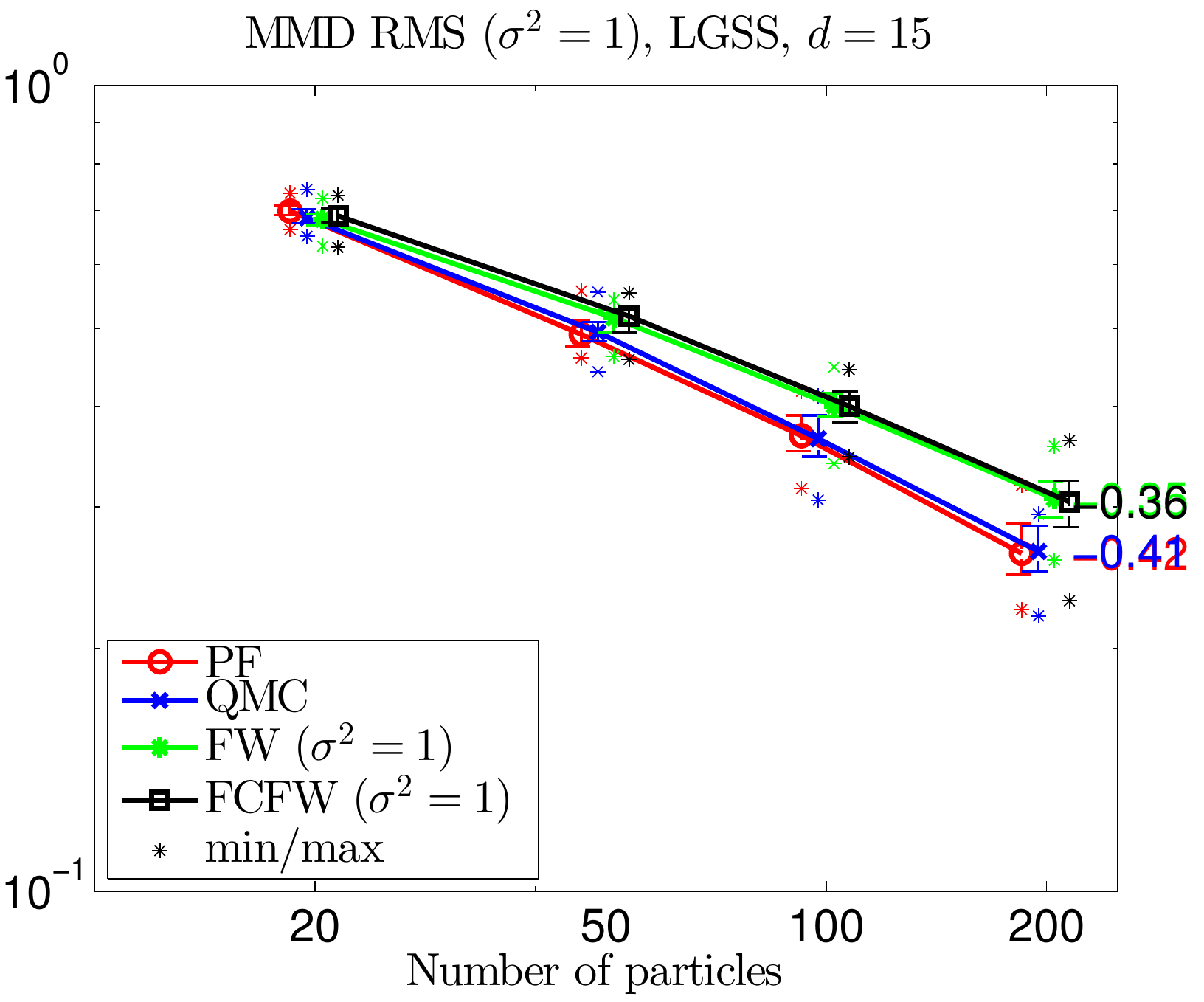}
  \caption{RMSEs (left and middle columns) for the four considered models and MMDs (right column) for the two LGSS models.}
  \label{fig:expts:syn}
\end{figure*}

To assess the performances of the different methods, we first compute the root-mean-squared errors (RMSE)
for the filtered mean-state-estimates over the~$T$ time steps, \wrt the reference filters. We report the median RMSEs
over the 30 \emph{different} data batches, along with the $25\%$ and $75\%$ quantiles, and the minimum and maximum values
in Figure~\ref{fig:expts:syn}.
The SKH algorithms were run for three different values of $\sigma^2\in\{0.01, 0.1, 1\}$.
Here, we report the results for $\sigma^2 = 1$ for the LGSS models and the JMLS, and for $\sigma^2=0.1$ for the nonlinear
benchmark model. The results for the other values are given in Appendix~\ref{app:synthetic}. The improvements
are somewhat robust to the value of $\sigma^2$, but in some cases significant
differences were observed. As can be seen, both SKH methods improve significantly
upon both QMC and the bootstrap~PF.
For the two LGSS models, we also compute the MMD
(reported in the rightmost column in Figure~\ref{fig:expts:syn}).
\subsection{Vision-based UAV Localization}\label{sec:uav}
In this section, we apply the proposed SKH algorithm to solve a filtering problem
in field robotics. We use the data and the experimental
setup described by~\citet{TornqvistSKG:2009}. %
The problem consists of estimating the full six-dimensional pose of an unmanned
aerial vehicle (UAV). 

\citet{TornqvistSKG:2009} proposed a vision-based solution,
essentially tracking interest points in the camera images over consecutive frames to estimate
the ego-motion.
This information is then fused with the inertial and barometer sensors
to estimate the pose of the UAV. The system is modelled on state-space
form, with a state vector comprising the position, velocity, acceleration, as well as the
orientation and the angular velocity of the UAV. The state is also augmented with
sensor biases, resulting in a state dimension of 22. Furthermore,
the state is augmented with the three-dimensional positions of the interest points
that are currently tracked by the vision system; this is a varying number but typically around ten.

To deal with the high-dimensional state-vector, \citet{TornqvistSKG:2009}
used a Rao-Blackwellized PF (see Appendix~\ref{sec:RB}) to solve the filtering problem, marginalizing all but 6 state components
(being the pose, \ie, the position and orientation) using a combination of Kalman
filters and extended Kalman filters. The remaining 6 state-variables were tracked using a bootstrap
particle filter with $\Np = 200$ particles; the strikingly small number of particles owing to the computational
complexity of the likelihood evaluation.

\begin{figure*}[!t]
\vspace{-3mm}
  \centering
  \includegraphics[width = 0.30\textwidth]{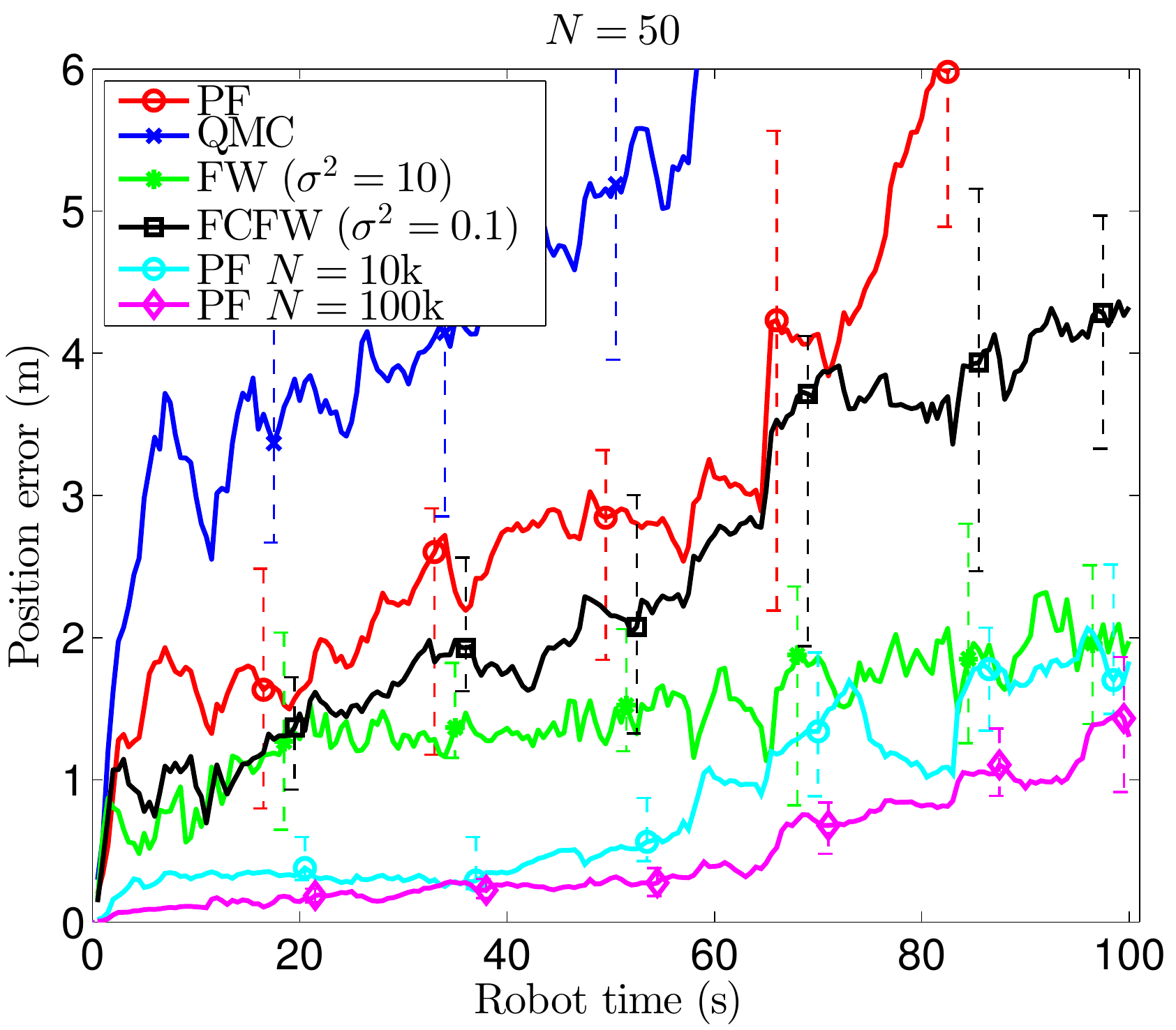} \quad
  \includegraphics[width = 0.30\textwidth]{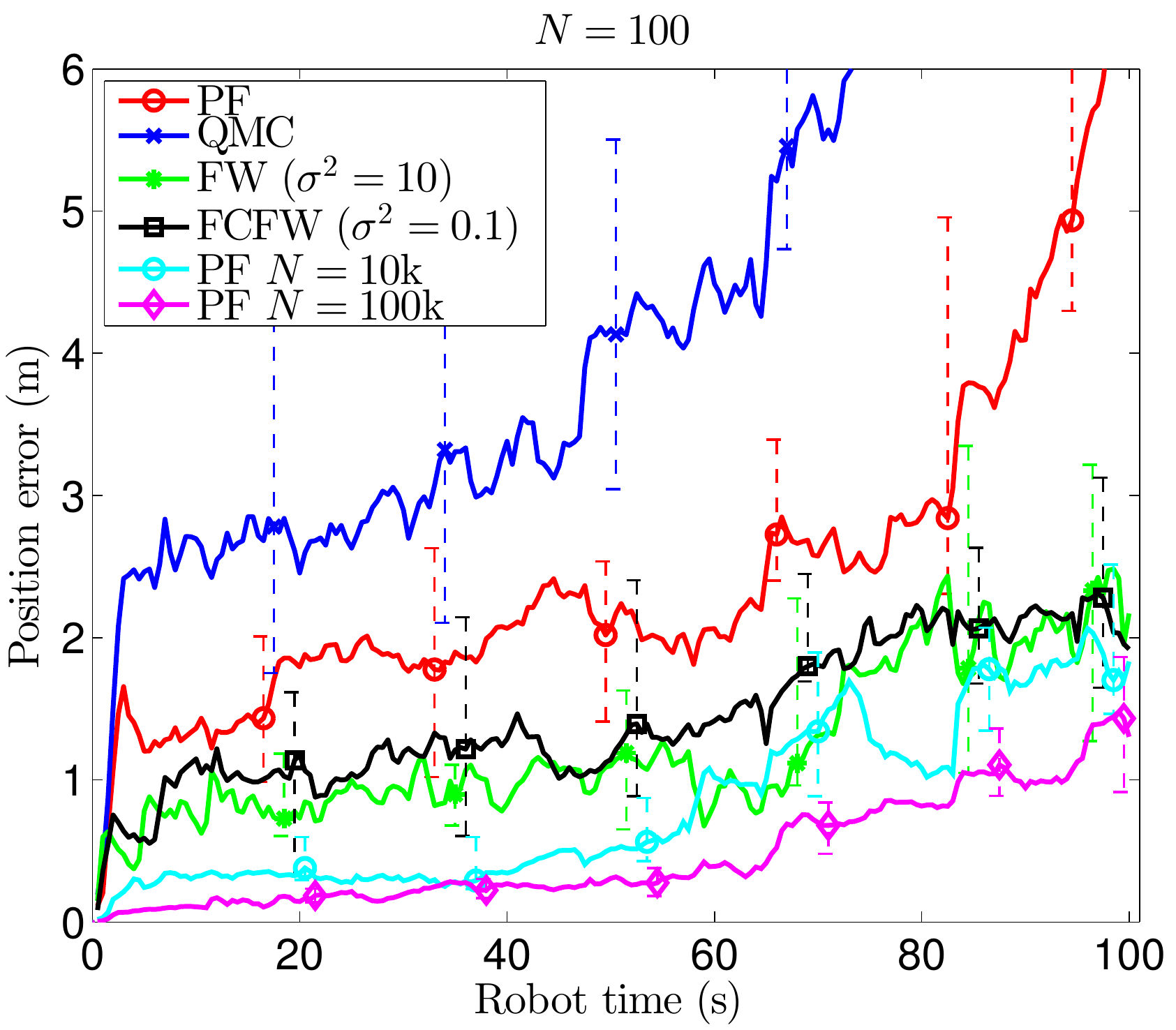} \quad
  \includegraphics[width = 0.30\textwidth]{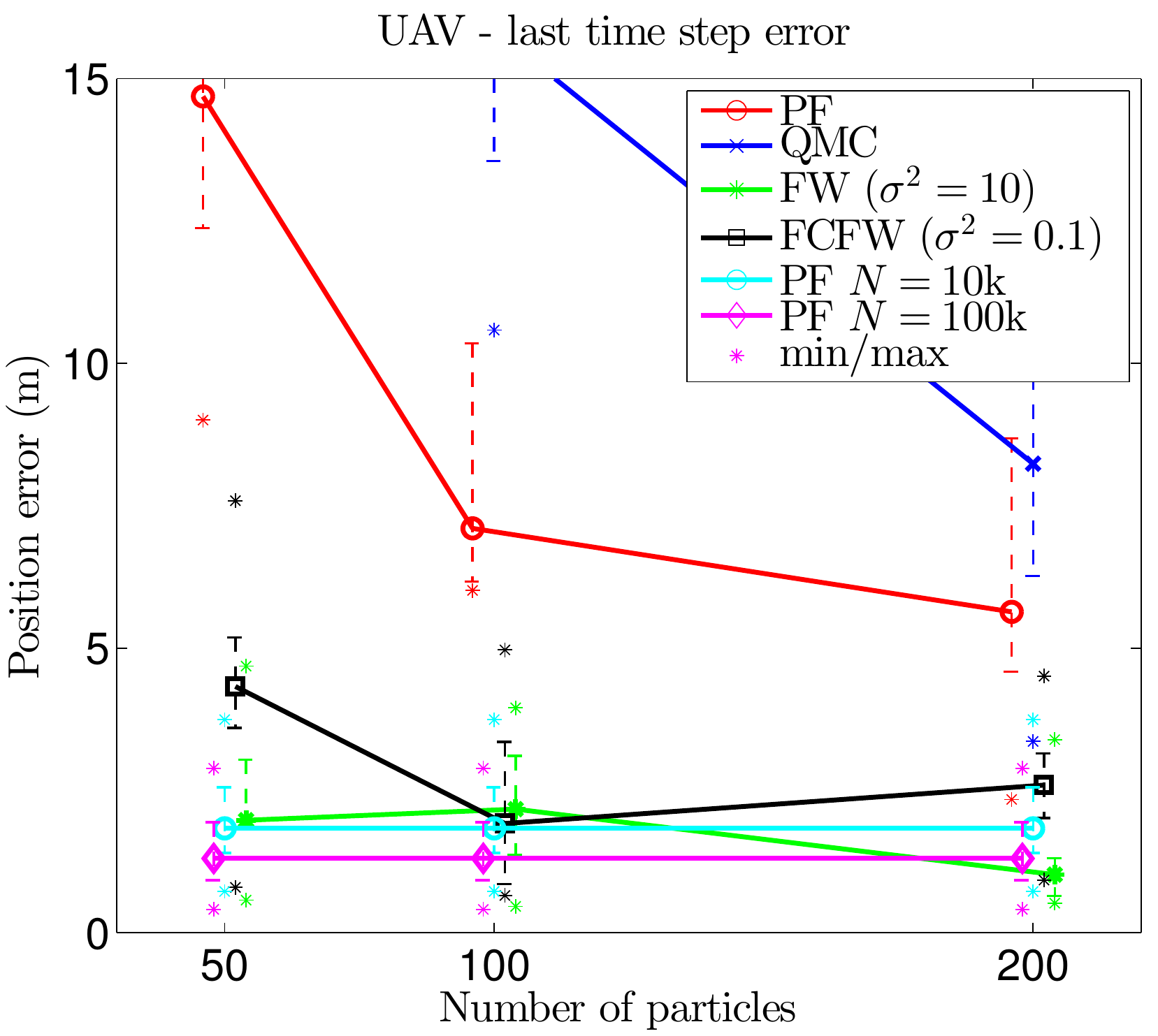}
  \caption{Median of position errors over 10 runs for each method. %
    The errors are computed relative to the mean prediction over 10 runs of a PF with 100k particles (the variation of the reference PF is also shown for PF 100k). The error bars represent the [25\%, 75\%] quantile. The rightmost plot shows the error at the last time step as a function of $\Np$. 100~s of robot time represents $2,\!000$ \ssm time steps, it \emph{does not} correspond to computation time.\vspace{-2mm}\mbox{}}
  \label{fig:uav:errors}
\end{figure*}

For the current experiment, we obtained the code and the flight-test data from
\citet{TornqvistSKG:2009}. The modularity of our approach allowed us
to simply replace the Monte Carlo simulation step within their setup with FW-Quad.
We ran SKH-FW with $\sigma^2 =10$ and SKH-FCFW with $\sigma^2 = 0.1$, as well as the bootstrap PF used in
\citet{TornqvistSKG:2009}, and a QMC-PF; all methods using $\Np = 50$, 100, and 200 particles.
We ran all methods 10 times on the same data; the variation in SKH coming from the random search points for
the FW procedure, and in QMC for starting the Sobol sequence at different points.
For comparison, we ran 10 times a reference PF with $\Np=100,\!000$ particles
and averaged the results. The median position errors for 100 seconds of robot time (there are 20 \ssm time steps per second of robot time)
are given in Figure~\ref{fig:uav:errors}.
The UAV is assumed to start at a known location at time zero, hence, all the errors are zero initially.
Note that all methods accumulate errors over time. This is natural, since there
is no absolute position reference available (\ie, the filter is unstable)
and the objective is basically to keep the error as small as possible for as long time as possible.
SKH-FW here gives the overall best results, with significant improvements over the bootstrap PF and the
QMC methods for small number of particles. SKH-FW even gives similar errors for the last time step with only $N=200$ particles as one of the \emph{reference} PFs (using $\Np = 100,\!000$ particles). See Appendix~\ref{app:FCFW} for a discussion of the role of $\sigma^2$ for FCFW.

\vspace{-3mm}
\paragraph{Runtimes.} In these experiments, we focused on investigating how optimization could improve the error per particle, as the gain in runtime depends on the exact implementation as well as the likelihood evaluation cost. We note that the FW-Quad algorithm scales as $O(N M)$ for $N$ samples and $M$ search points when using FW, by updating the objective on the $M$ search points in an online fashion (we also empirically observed this linear scaling in $N$). On the other hand, FCFW scales as $O(N^2 M)$ as the weights on the particles possibly change at each iteration, preventing the same online trick. SKH scales linearly with the number of time steps $T$ (as a standard PF). For the UAV application, the original Matlab code from~\citet{TornqvistSKG:2009} spent an average of 0.2~s per time step for $N=50$ particles (linear in the number of particles as the likelihood evaluation is the bottleneck) on a XEON E5-2620 2.10 GHz PC. The overhead of using our Matlab implementation of FW-Quad with $N=50$ is about 0.15~s per time step for FW and 0.3~s for FCFW; and 0.3~s for FW and 1.0~s for FCFW for $N=100$ (we used $M=10,\!000$ search points in this experiment). In practice, this means that SKH-FW can be run here with 50 particles in the same time as the standard PF is run with about 90 particles. But as Figure~\ref{fig:uav:errors} shows, the error for SKH-FW with 50 particles is still much lower than the PF with 200 particles.

%% file: supplement.tex
\appendix

\numberwithin{definition}{section}
\numberwithin{algorithm}{section}

\onecolumn

{
  \linewidth\hsize \toptitlebar 
  {\centering
  {\Large\bf Supplementary material \par}}
 \bottomtitlebar 
}

\section{Extension for Rao-Blackwellization} \label{sec:RB}
A common strategy for improving the efficiency of the PF is to make use of Rao-Blackwellization---this idea can be used also with SKH.
Rao-Blackwellization, here, refers to analytically marginalizing some conditionally tractable component of the state vector
and thereby reducing the dimensionality of the space on which the PF operates. Assume that the state of the system is comprised of
two components~$x_t$ and~$z_t$, where the filtering density for~$z_t$ is tractable \emph{conditionally} on the history of~$x_{1:t}$.
The typical case is that of a conditionally linear Gaussian system, in which case the aforementioned conditional filtering density
$p(z_t | x_{1:t}, y_{1:t})$ is Gaussian and computable using a Kalman filter (conditionally on $x_{1:t}$). The Rao-Blackwellized PF (RBPF)
exploits this property by factorizing:
\begin{equation}
  \label{eq:rbpf:factorization}
  p(z_t, x_{1:t} | y_{1:t}) = p(z_t | x_{1:t}, y_{1:t}) p(x_{1:t} | y_{1:t}) 
  \approx \sum_{i=1}^N w_t^{(i)} \N(z_t | \widehat z_{t}(x_{1:t}^{(i)}), \Sigma_t(x_{1:t}^{(i)})) \delta_{x_{1:t}^{(i)}}(x_{1:t}),
\end{equation}
where the conditional mean $\widehat z_t(x_{1:t}) := \E[z_t | x_{1:t}, y_{1:t}]$ and covariance matrix
$\Sigma_t(x_{1:t}) := \V(z_t| x_{1:t}, y_{1:t})$
can be computed (for a fixed trajectory $x_{1:t}$) using a Kalman filter.
The mixture approximation follows by plugging in a particle approximation of $ p(x_{1:t} | y_{1:t})$ computed
using a standard PF.   Hence, for a conditionally linear Gaussian model, the RBPF
takes the form of a Mixture Kalman filter; see \citetsup{ChenL:2000}.
Analogously to a standard PF, the SKH procedure allows us to to compute an empirical
point-mass approximation of $p(x_{1:t} | y_{1:t})$
by keeping track of the complete history of the state $x_{1:t}$.
Consequently, by \eqref{eq:rbpf:factorization} it is straightforward to employ Rao-Blackwellization
also for SKH; we use this approach in the numerical example in Section~\ref{sec:uav}.

\section{Rates for SKH when using random search points} \label{app:rateRandomSearch}

In this section, we show that we can get guarantees on the $\MMD$ error of the FW-Quad procedure when approximately finding the FW vertex in step~3 of Algorithm~\ref{alg:FWquad} using exhaustive search through $M$ random samples from $p$. This means that despite not solving step~3 exactly, the SKH procedure with $M$ random search points (under assumptions of Theorem~\ref{thm:rate}) is still consistent as long as $M$ grows to infinity.

The main idea is that the rates of convergence for the Frank-Wolfe optimization procedure still holds when the linear subproblem (step~3) is solved within accuracy of $\delta$. More specifically, if we guarantee that the FW vertex $\bar{g}_{k+1}$ that we use satisfy $\innerProd{J'(g_k)}{\bar{g}_{k+1}} \leq \min_{g \in \M} \innerProd{J'(g_k)}{g} + \delta$ during the algorithm, then the standard $O(1/k)$ rate of convergence for FW carries through but \emph{within $\delta$} of the optimal objective (i.e. up to $J(g^{*}) + \delta$). A simple modification of the argument by~\citet{jaggi13FW} (who used a \emph{shrinking} $\delta$ during the FW algorithm) can show this for the step-size of $\gamma_k = \frac{2}{k+2}$; we give the proofs for the step-size of $\gamma_k = \frac{1}{k+1}$ as well as the potential faster rate $O(1/k^2)$ for the MMD objective in Appendix~\ref{app:FWrates}.

Let $\X_M \subseteq \X$ be the set of $M$ search points, and $p_M$ be the empirical distribution for the $M$ samples from $p$. Let $\delta_M := \|\mu(p_M)-\mu(p)\|_\H$ which can be made small by increasing $M$. Consider the iteration $k$ in FW-Quad where we do exhaustive search on $\X_M$ in step~3. 
 We thus have:
\begin{align*}
\innerProd{g_k-\mu_p}{\Phi(x^{(k+1)})} =
	 \min_{x \in \X_M} \innerProd{g_k-\mu_p}{\Phi(x)}
	& = \min_{x \in \X_M} \innerProd{g_k-\mu(p_M) + \mu(p_M)-\mu(p)}{\Phi(x)} \\
	& \leq \min_{x \in \X_M} \innerProd{g_k-\mu(p_M)}{\Phi(x)} + \delta_M R_M ,
\end{align*}
where $R_M := \max_{x \in \X_M} \|\Phi(x) \|$ ($R_M \leq R$). 
We can thus interpret step~3 as approximately solving (within $\delta_M R_M$) the linear subproblem for the Frank-Wolfe optimization of $J_M(g) := \frac{1}{2} \|g-\mu(p_M)\|_\H^2$ over the marginal polytope of $\X_M$. We thus get a rate of convergence to within $\delta_M R_M$ of $\min_g J_M(g) = 0$. Finally, we have 
$$
\|g_N - \mu(p)\|_\H \leq \|g_N - \mu(p_M)\|_H + \delta_M = \sqrt{2 J_M(g_N)} + \delta_M
\leq \sqrt{2(\epsilon_N + R_M \delta_M)} + \delta_M
$$
where $\epsilon_N$ would be the error after $N$ steps of a standard (non-approximate) Frank-Wolfe procedure (e.g. $O(1/N)$, though it could be $O(1/N^2)$ if $\mu(p_M)$ is in the strict interior of the marginal polytope of $\X_M$ as we show in Appendix~\ref{app:FWrates}). Finally, we know that $\E [\delta_M] \leq R/\sqrt{M}$, and we could also obtain a high probability bound for it as well using a concentration inequality with triangular arrays. This gives the guarantee for the $\MMD$ error of the SKH procedure with $M$ random search points (with a term of $O(1/M^{1/4})$). Even though the rate is slow in $M$, the approach is motivated for problems where the bottleneck is the evaluation of the observation probability (which is only evaluated $N$ times per time step) whereas $M$ can be taken to be very large. We also note that if $\H$ is finite dimensional and the kernel $\kernel$ is continuous, then an asymptotically faster rate of $O(1/\sqrt{M})$ can be shown (see Appendix~\ref{app:FWrates}), though with a worse constant that makes the comparison for smaller $M$ less clear.

\section{Additional details on experiments}

\subsection{Mixture of Gaussians experiment} \label{app:MoG}

 \begin{figure*}[tb]
   \centering 
   \includegraphics[width = 0.32\textwidth]{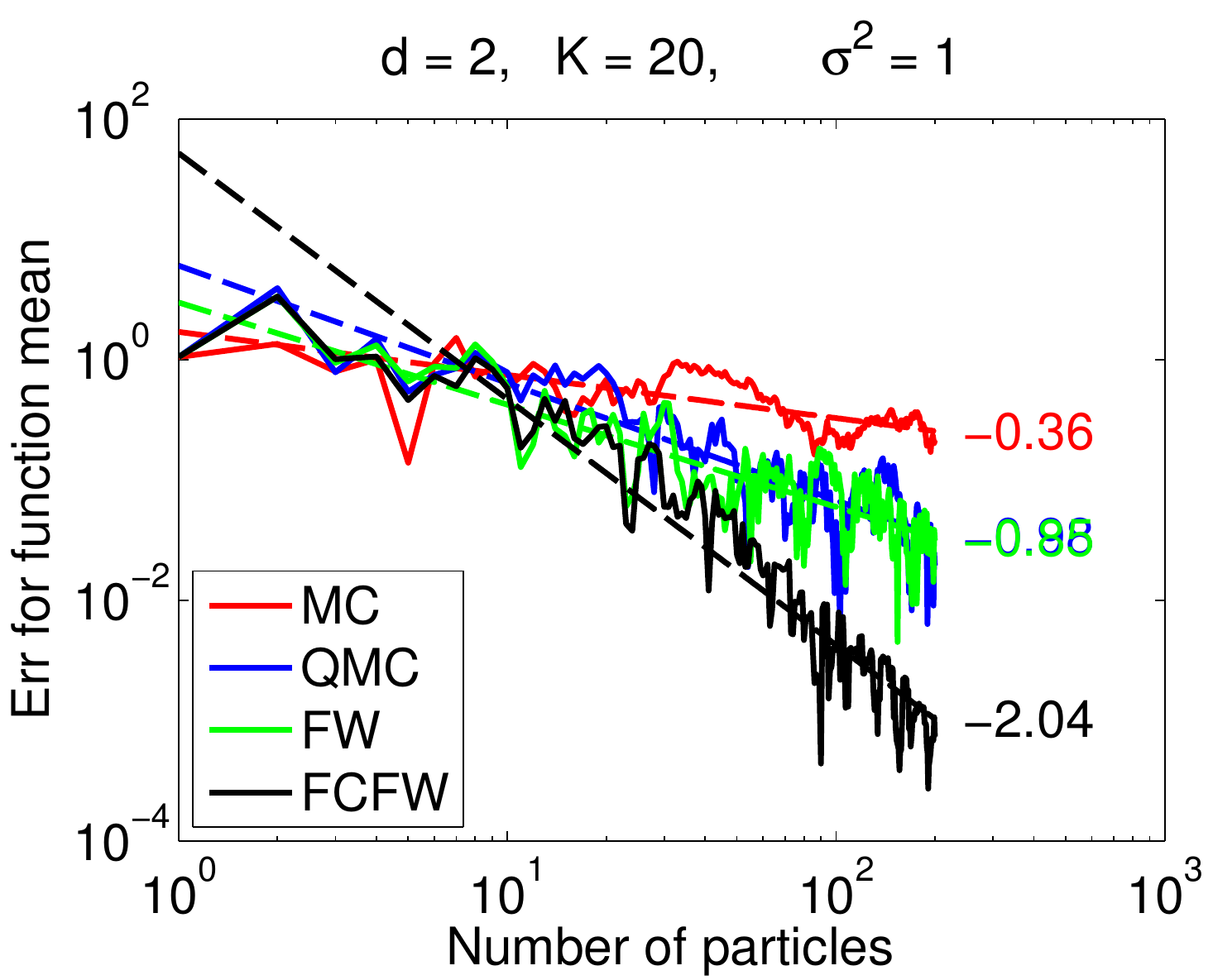} 
   \includegraphics[width = 0.32\textwidth]{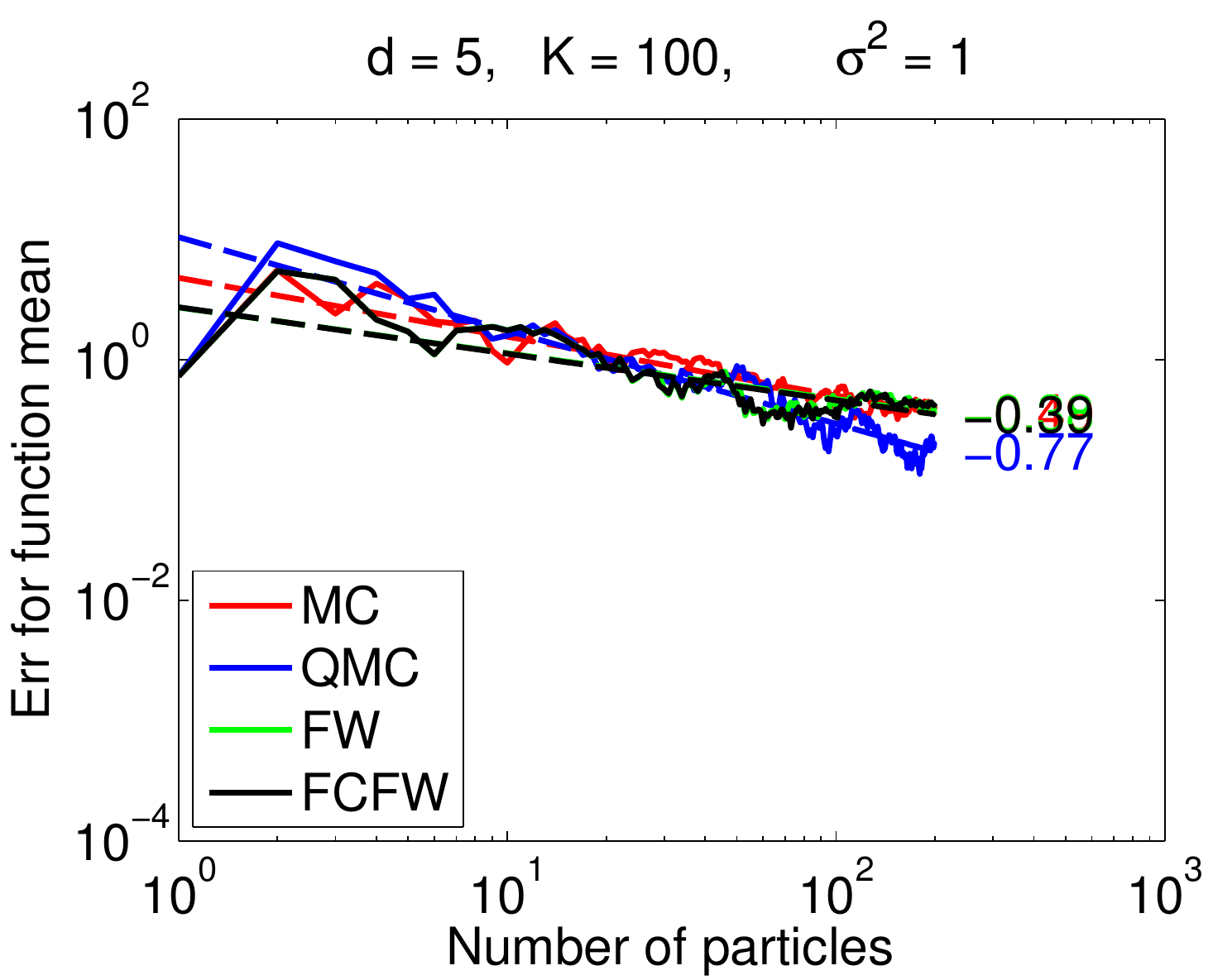} 
   \includegraphics[width = 0.32\textwidth]{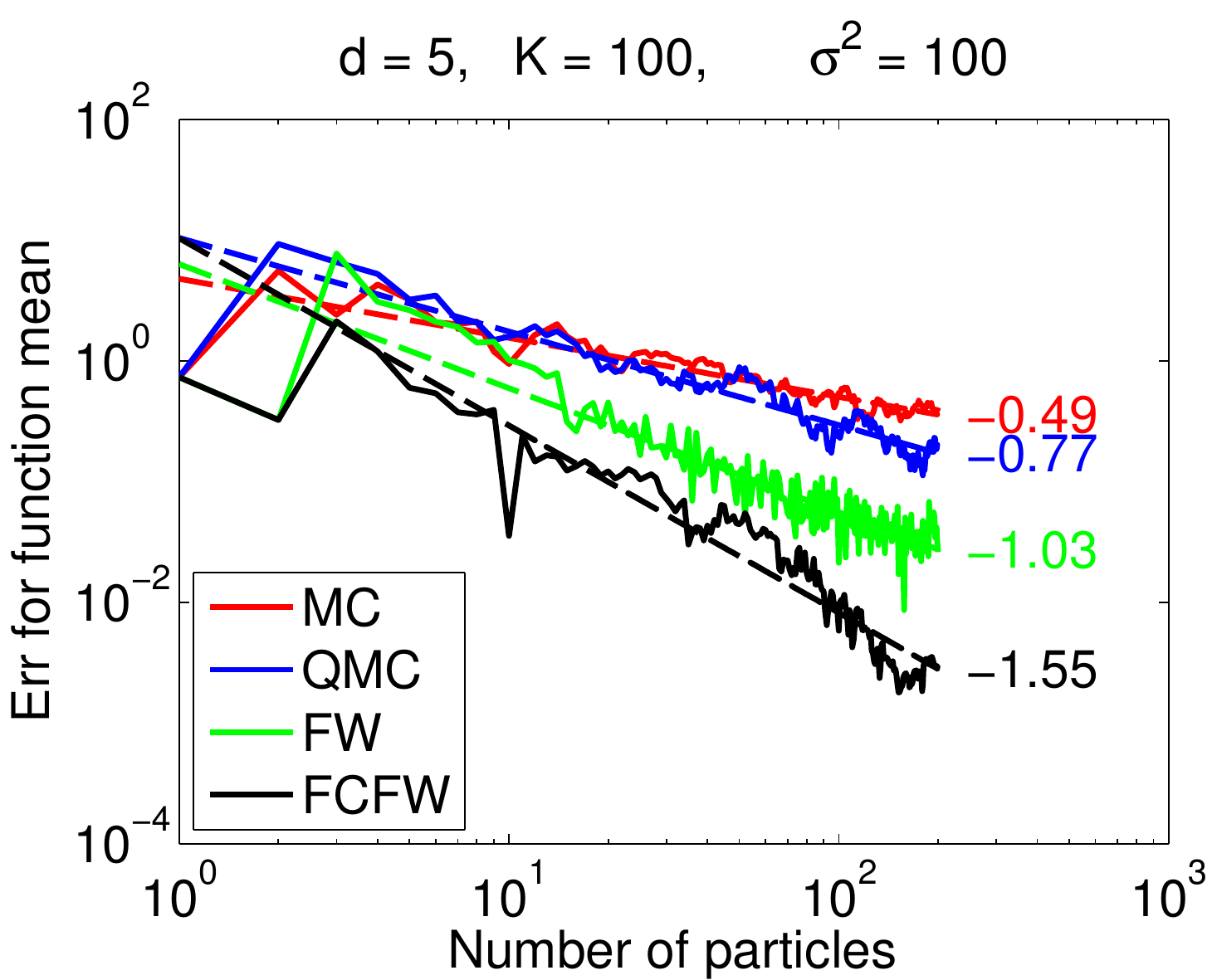}\\
   \vspace{3mm}
   \includegraphics[width = 0.32\textwidth]{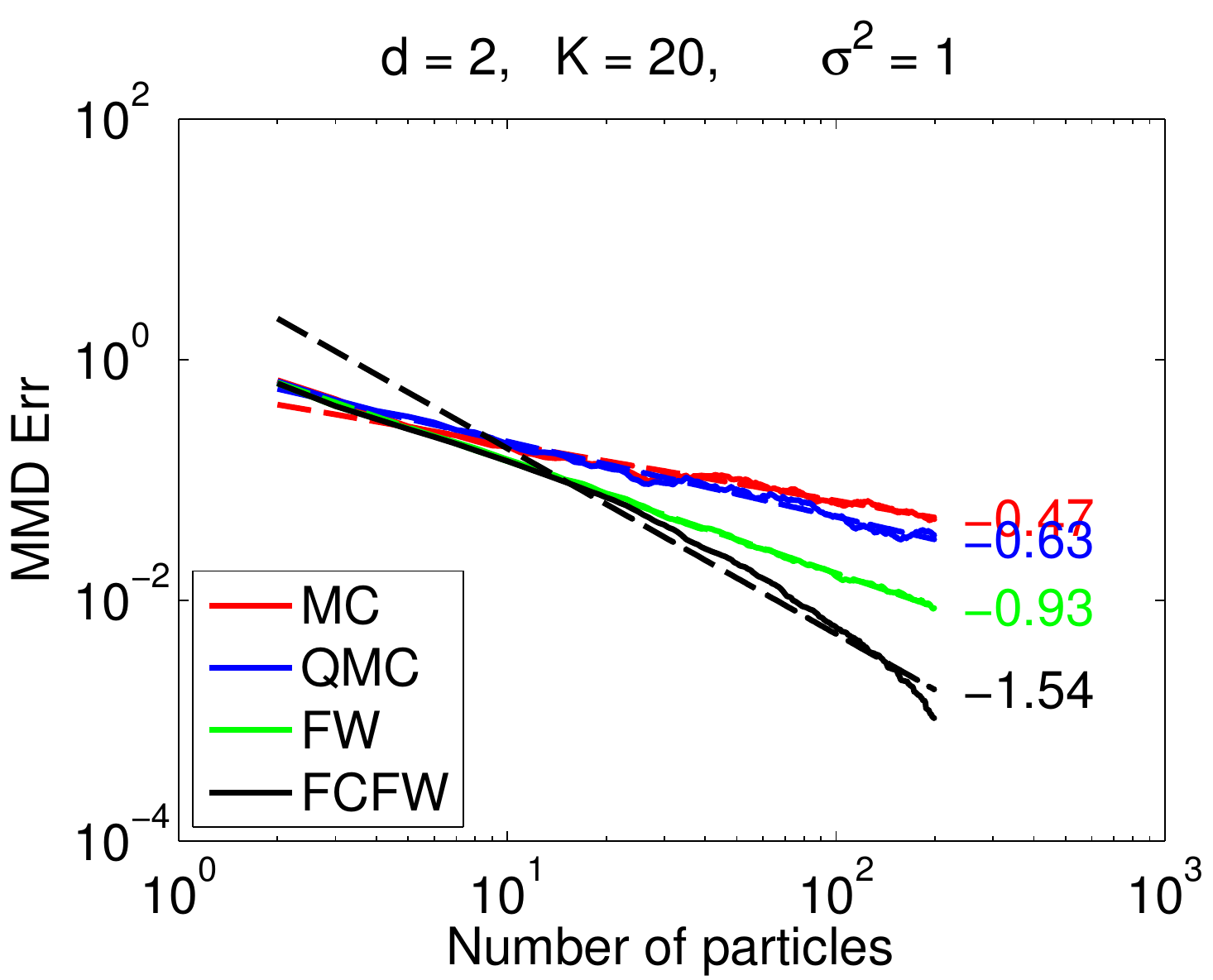} 
   \includegraphics[width = 0.32\textwidth]{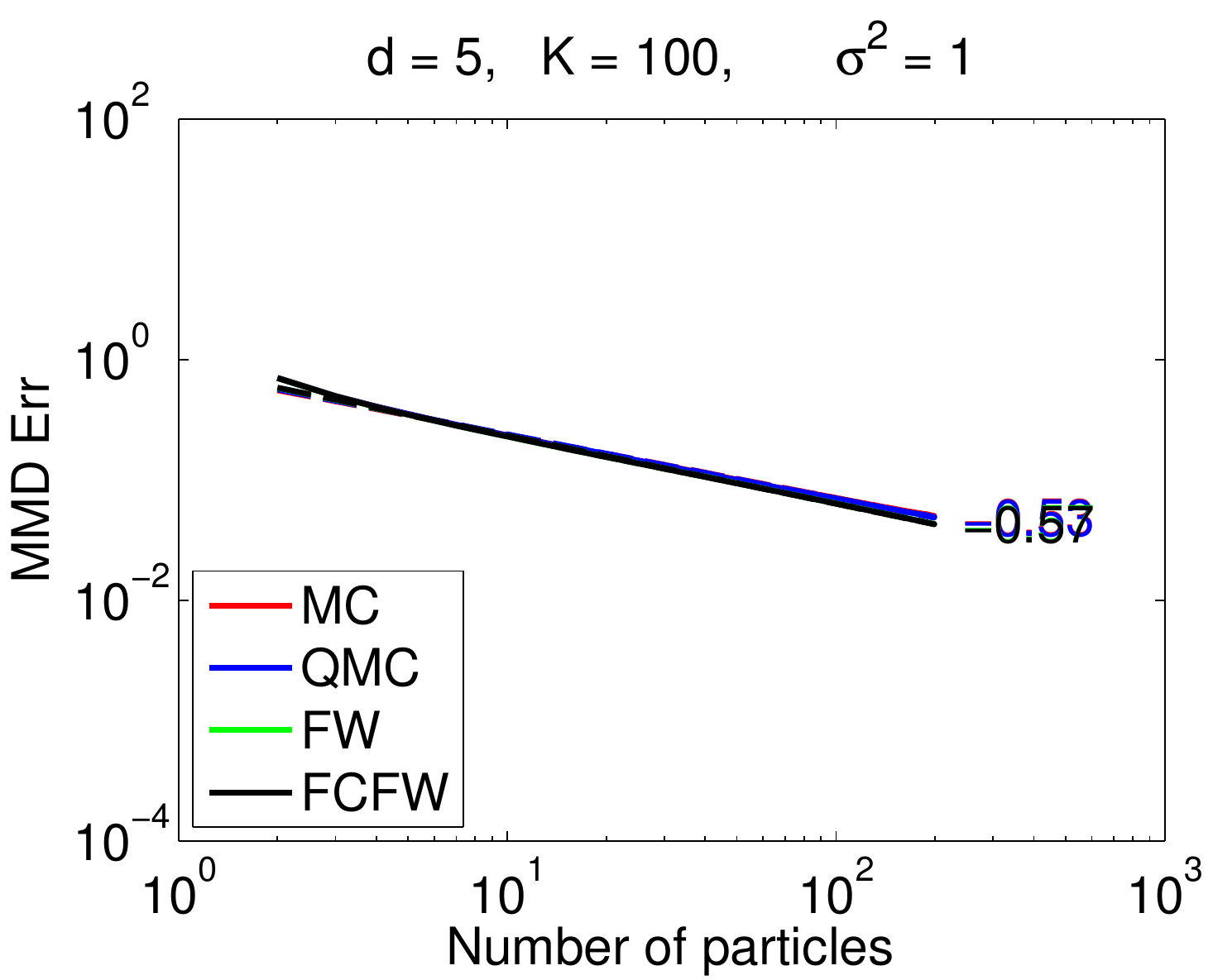} 
   \includegraphics[width = 0.32\textwidth]{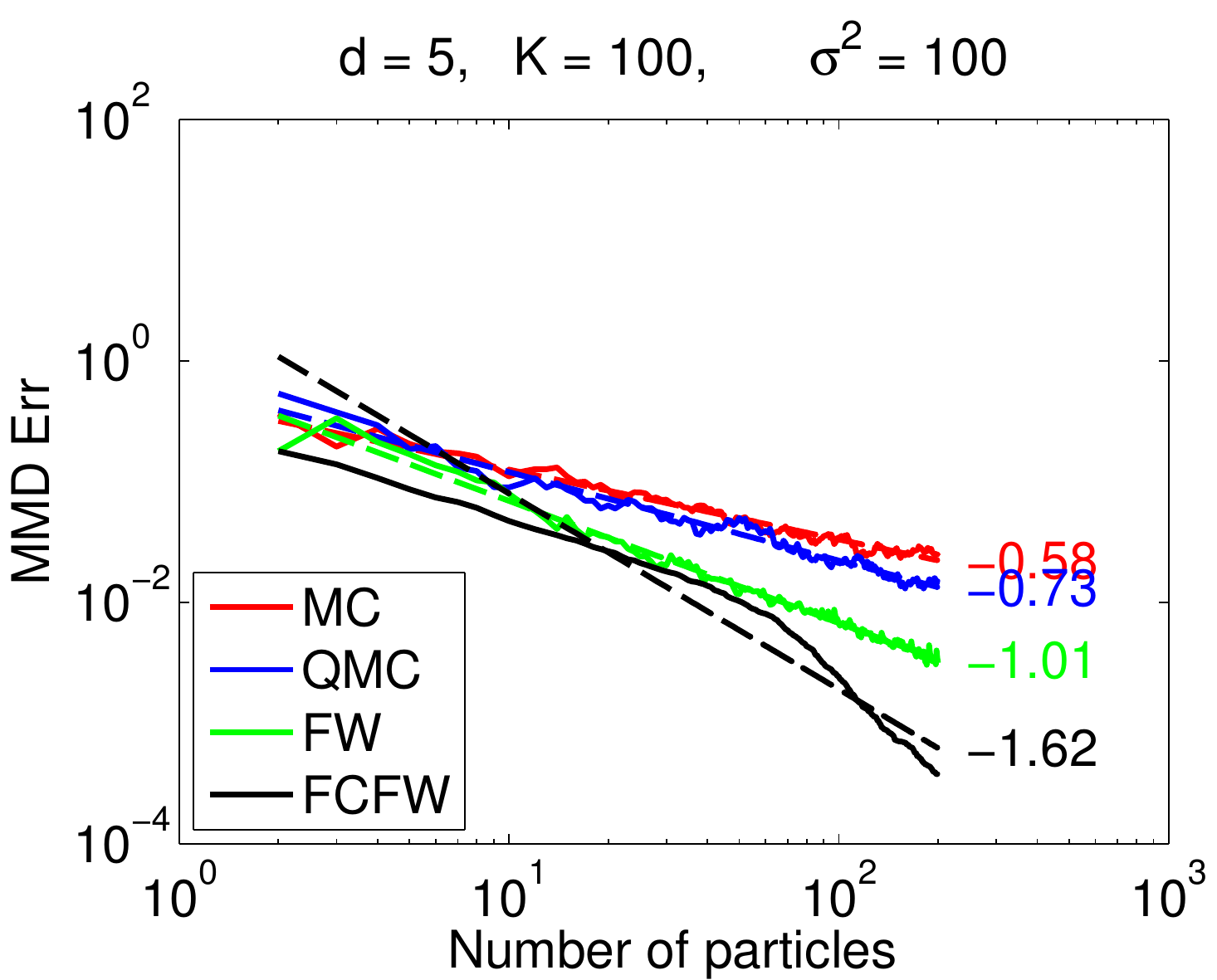}
   \caption{Error on the mean function (top row) and MMD error (bottom row) for the mixture of Gaussians experiment. The first column is for $K=20$ and $d=2$. The next two columns are for the \emph{same} mixture of Gaussians in higher dimension $d=5$ with $K=100$ components, but running FW-Quad with $\sigma^2 = 1$ (middle column) or $\sigma^2 = 100$ (last column). We see that using a higher $\sigma^2$ helps significantly in higher dimension. The dashed lines are linear fits with slopes reported next to the axes.}
   \label{fig:expts:gmm2}
 \end{figure*}

The parameters for the mixture of Gaussians $p(x) = \sum_{i=1}^K \pi_i \N(x | \mu_i, \Sigma_i)$ were randomly sampled as follows:
\begin{itemize}
\item The means $\mu_i$'s are uniformly sampled on $[-5,5]^d$.
\item $\Sigma_i = \sigma_i^2 \idm$ where $\sigma_i^2$ is uniformly sampled on $[0.1, 4.1]$.
\item $\pi_i$ are obtained by normalizing independent uniform random variables.
\end{itemize}

Figure~\ref{fig:expts:gmm2} present additional results for the mixture of Gaussians experiments. From our experiments, we make the following observations:
\begin{list}{\labelitemi}{\leftmargin=1.1em}
   \addtolength{\itemsep}{-.6\baselineskip}
\item[--] FCFW always performs best (this was observed similarly in~\citet{Bach2012} but for other pairs of distribution / kernel).
\item[--] As $d$ increases, the difference between the methods tapers off for a fixed $\sigma^2$, but increasing $\sigma^2$ gives better results for FW and FCFW than the others (see for example the last column of Figure~\ref{fig:expts:gmm2}).
\item[--] The FW-LS results are identical to FW, and so we have excluded them from the plots for clarity.
\item[--] The improvement of QMC over MC decreases as the number of mixture components $K$ increase. FW and FCFW are not affected by $K$ as much. 
\end{list}

\paragraph{QMC implementation.} To generate quasi-random samples from the mixture of Gaussians, we generate a $(d\!+\!1)$-dimensional Sobol sequence using the Matlab~\texttt{qrandstream} function. The last dimension is (naively) used to sample the mixture component by using the inverse transformation method for a discrete random variable. The first $d$ components are then used to sample from the corresponding multivariate Gaussian by transforming $d$ independent standard normals. We note that~\citet{gerber2014sqmc} argued on the importance of sorting the discrete mixture components according to their location \emph{before} choosing them with the standard inverse transformation method (in our naive implementation, the order is arbitrary and arising from how the mixture of Gaussians was stored). They propose a method for this that they called \emph{Hilbert sort}, for which they could prove nice low-discrepancy properties. This approach might reduce the sensitivity to $K$ of QMC. The worse results of QMC for the UAV experiment in Figure~\ref{fig:uav:errors} might be explained by our naive implementation.

\subsection{Synthetic data examples and additional results} \label{app:synthetic}
In this section we provide additional details and results for the synthetic
data examples. The LGSS models and the modes of the JMLS are generated randomly using the function {\tt drss} from the Matlab Control Systems Toolbox. The four models that were considered are given by:

\begin{description}
\item[LGSS, $d=3$] on the form
  \begin{align*}
    x_{t+1} &= Ax_t + v_t, & v_t&\sim \N(0,I), \\
    y_{t+1} &= Cx_t + e_t, & e_t&\sim \N(0,0.1)
  \end{align*}
  with $(A,C)$ being an observable pair. The system has poles in $-0.2825$
  and $-0.3669 \pm 0.0379i$.
\item[LGSS, $d=15$]  on the form
  \begin{align*}
    x_{t+1} &= Ax_t + v_t, & v_t&\sim \N(0,I), \\
    y_{t+1} &= Cx_t + e_t, & e_t&\sim \N(0,0.1)
  \end{align*}
  with $(A,C)$ being an observable pair. The system has poles in
  $0.2456 \pm 0.6594i$,
  $0.4833$,
  $0.3329$,
  $0.0882 \pm 0.2512i$, 
  $-0.1485$,
  $-0.8045$,
  $-0.4848$,
  $-0.5252 \pm 0.0368i$,
  $-0.6692 \pm 0.0612i$,
  $-0.6604$, and
  $-0.6680$.
\item[JMLS] on the form
  \begin{align*}
    \Prb&( r_{t+1}=\ell | r_t= k )= \Pi_{k\ell}, && \\
    x_{t+1} &= A_{r_t}x_t + F_{r_t}v_t,& v_t&\sim\N(0,I), \\
    y_t &= C_{r_t}x_t + G_{r_t} e_t, & e_t &\sim \N(0,1), \\
  \end{align*}
  with
$    \Pi =
    \begin{pmatrix}
      0.7 & 0.3 \\
      0.3 & 0.7
    \end{pmatrix}$,
  and the two system modes corresponding to observable systems
  with poles in $-0.4429$, $0.0937$, and $-0.6576$, $0.3109$, respectively.
\item[Nonlinear benchmark model] as described in the main text.
\end{description}

Additional results, for the different values of $\sigma^2 \in \{0.01, 0.1, 1\}$
are reported in Figures~\ref{fig:1}--\ref{fig:n}.

\subsubsection{Discussion of role of $\sigma^2$ for FCFW } \label{app:FCFW}

The results in Figure~\ref{fig:nlresults:sup} for the nonlinear benchmark show an interesting behavior for FCFW when the kernel bandwidth $\sigma^2$ is increasing. In particular, for $\sigma^2=1$ (rightmost plot), SKH-FCFW obtains the lowest error for all methods (and other $\sigma$'s) at $N=20$, but its error stays constant when increasing the number of particles while the other methods see their error decreasing. This phenomenon needs to be carefully studied further. Our current hypothesis is that when $\sigma^2$ is large, FCFW is too effective at myopically optimizing the MMD error for the mixture of Gaussians $\tilde{p}_t$ and yields a too small effective sample size (it sets many weights of particles to zero), thus hurting the particle filtering error. When $d$ is small or when $\sigma^2$ is large, the Gaussian kernel matrix becomes rank-deficient due to numerical precision; we thus have numerically a finite dimensional $\H$. In the case of the 1d nonlinear model, FCFW sometimes could optimize the MMD error within numerical precision (its square of the order of 1e-16) within 30 particles. FW-Quad would thus output only 30 particles even though we asked it to produce $N > 30$. This explains why increasing $N$ did not translate in a reduction of filtering errors for SKH-FCFW with $\sigma^2=1$: the effective number of particles stayed much less than $N$.

SKH-FW did not seem to suffer from this problem. This might partly explain why we were able to use the much bigger $\sigma^2=10$ for the UAV experiment with good results (Figure~\ref{fig:uav:errors}) whereas we used $\sigma^2=0.1$ for SKH-FCFW. 

\begin{figure}[!tb]
  \centering
  \includegraphics[width = 0.32\textwidth]{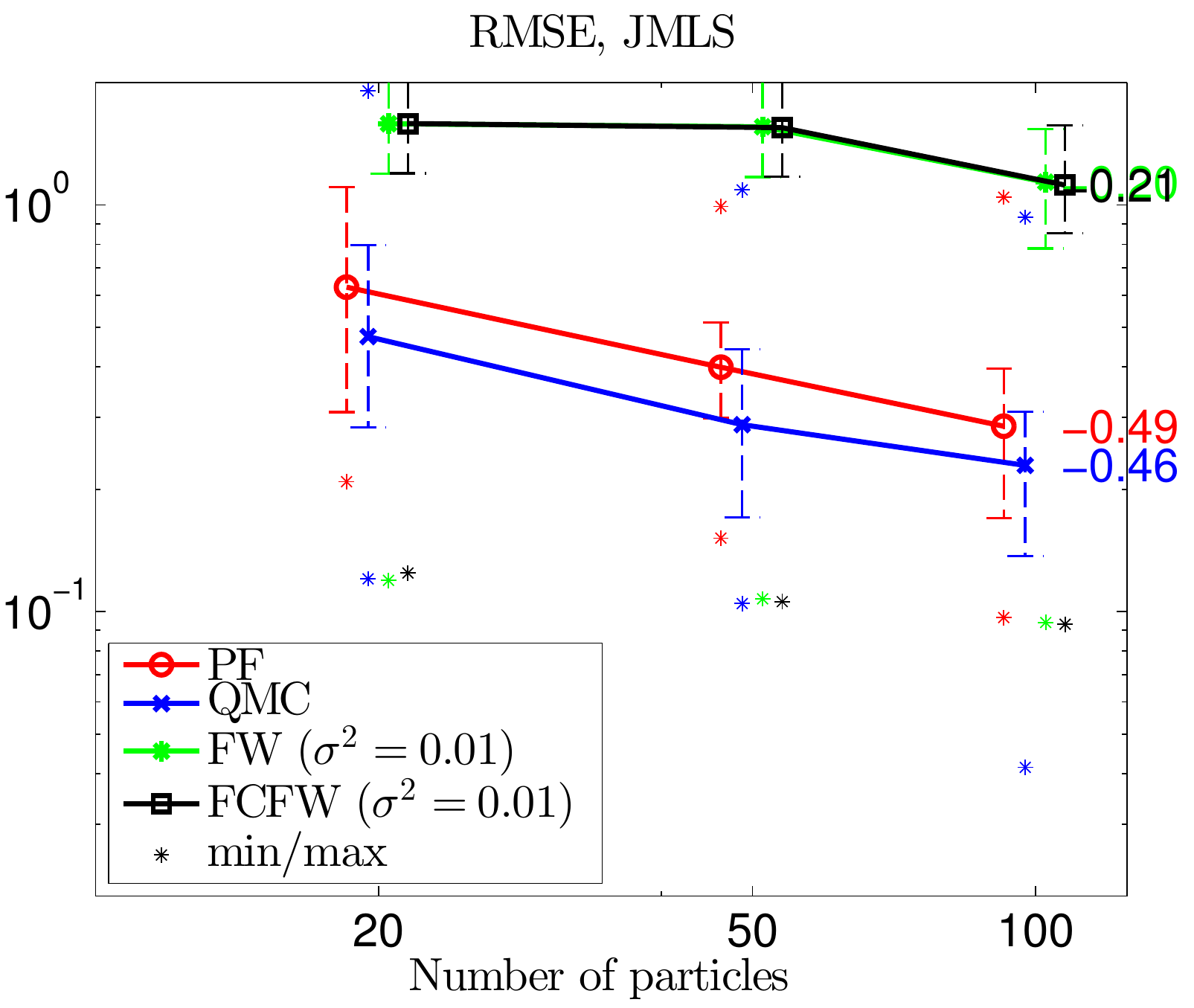} 
  \includegraphics[width = 0.32\textwidth]{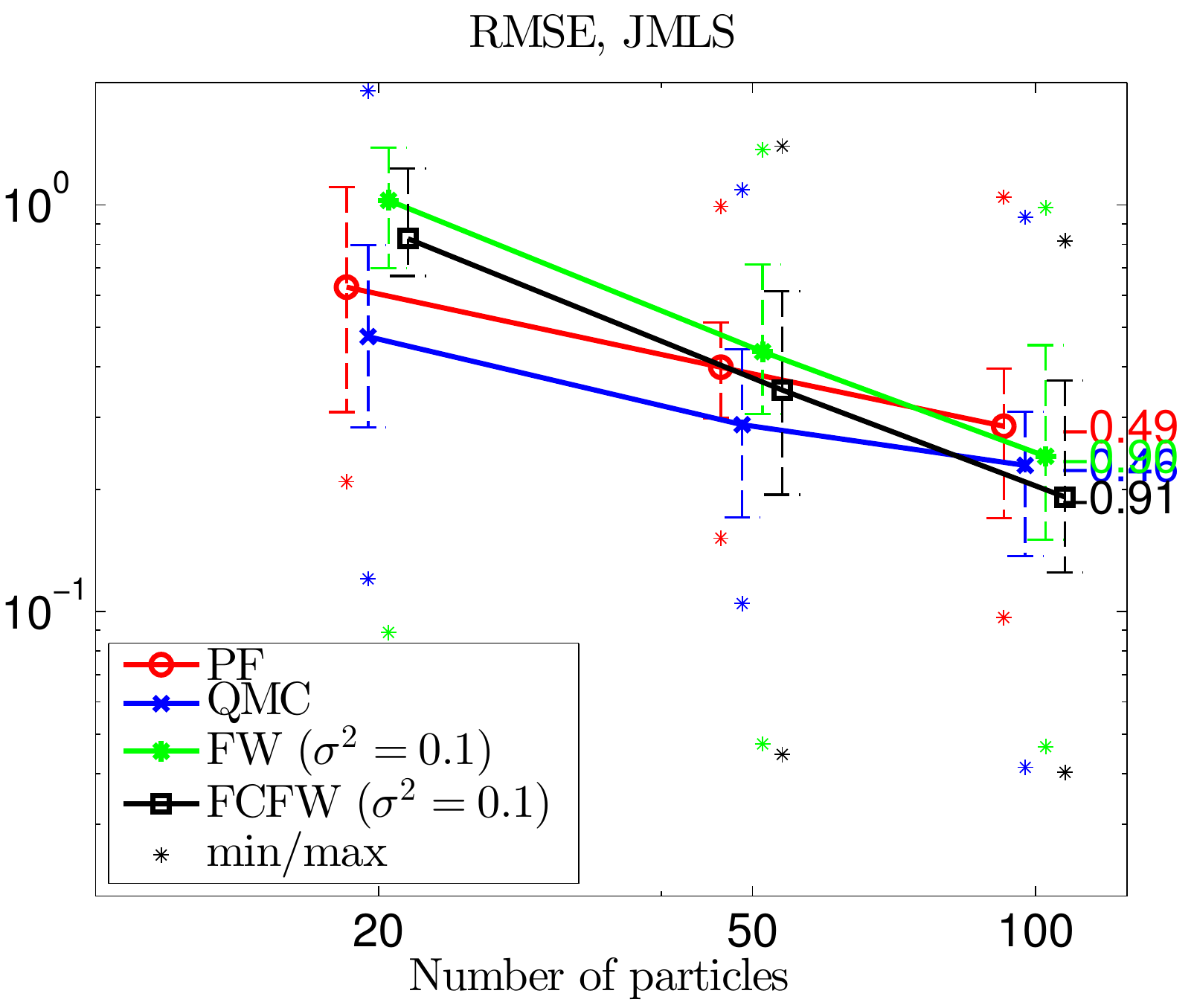} 
  \includegraphics[width = 0.32\textwidth]{./figures/zoomLog_jmls_3}
  \caption{RMSE for JMLS, using $\sigma^2 = 0.01$, $\sigma^2 = 0.1$, and $\sigma^2 = 1$ (left to right).}
  \label{fig:1}
\end{figure}

\begin{figure}[!tb]
  \centering
  \includegraphics[width = 0.32\textwidth]{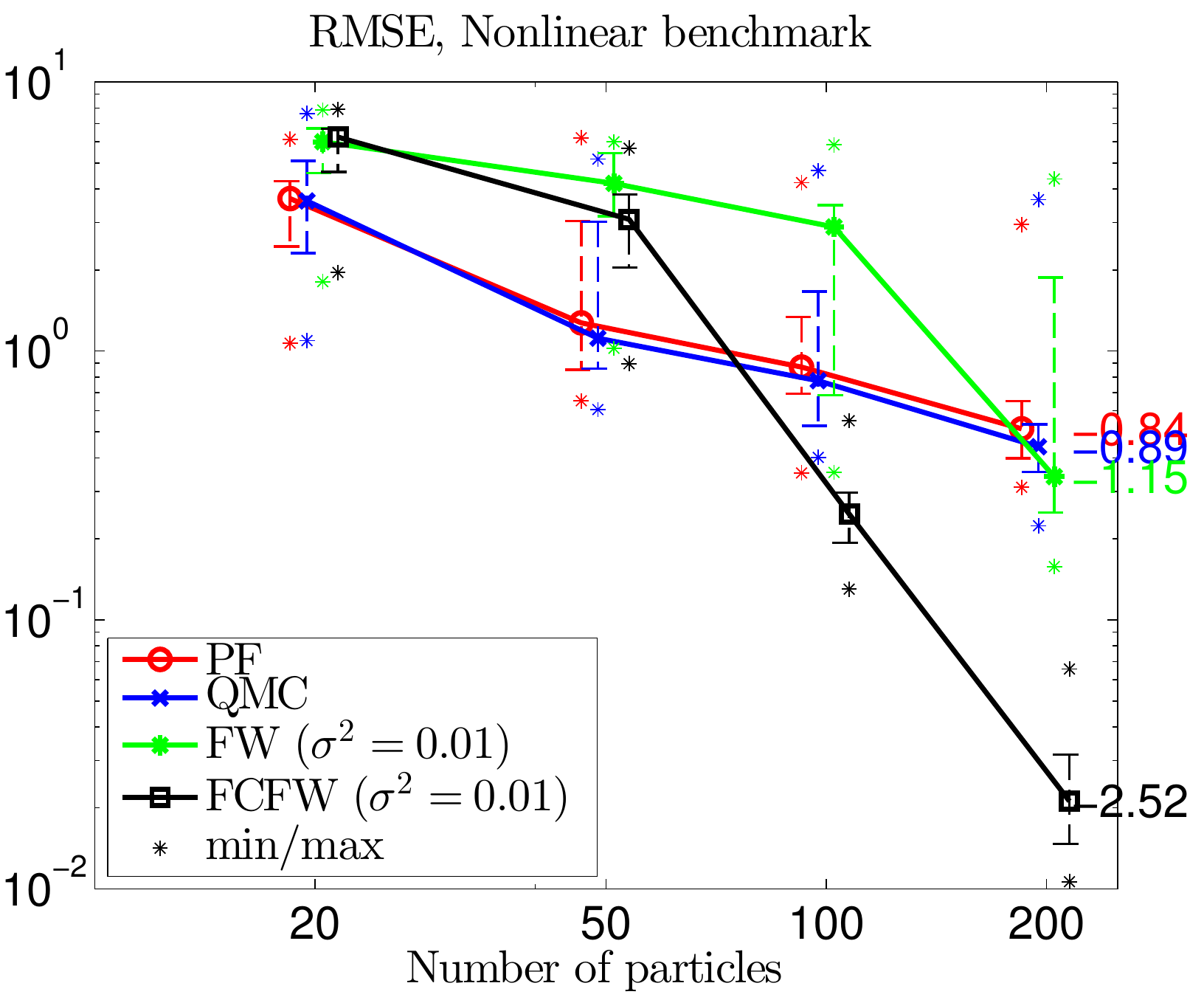} 
  \includegraphics[width = 0.32\textwidth]{./figures/zoomLog_nl_2} 
  \includegraphics[width = 0.32\textwidth]{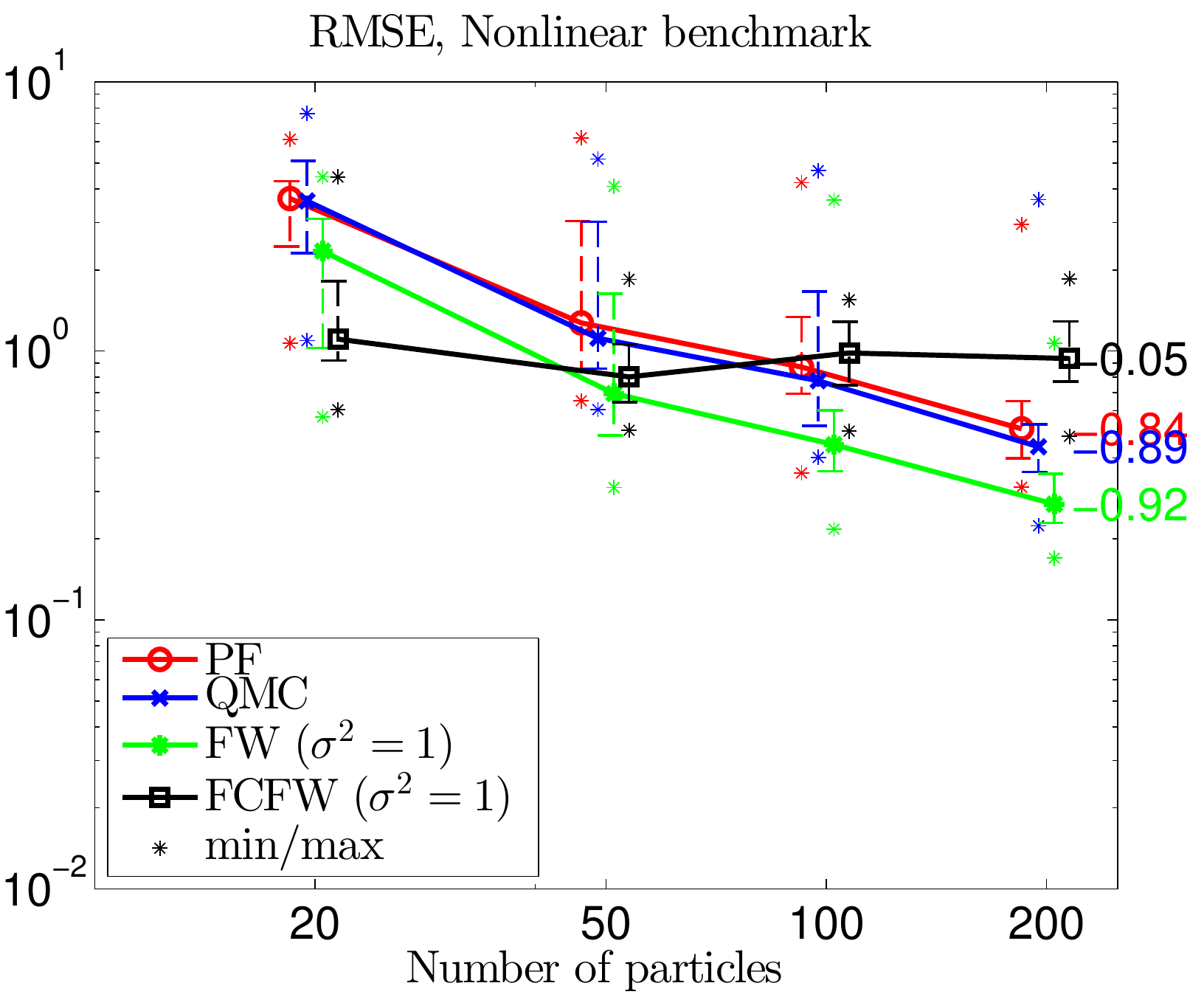}
  \caption{RMSE for nonlinear benchmark model, using $\sigma^2 = 0.01$, $\sigma^2 = 0.1$, and $\sigma^2 = 1$ (left to right).}
  \label{fig:nlresults:sup}
\end{figure}

\clearpage

\begin{figure}[p]
  \centering 
  \includegraphics[width = 0.32\textwidth]{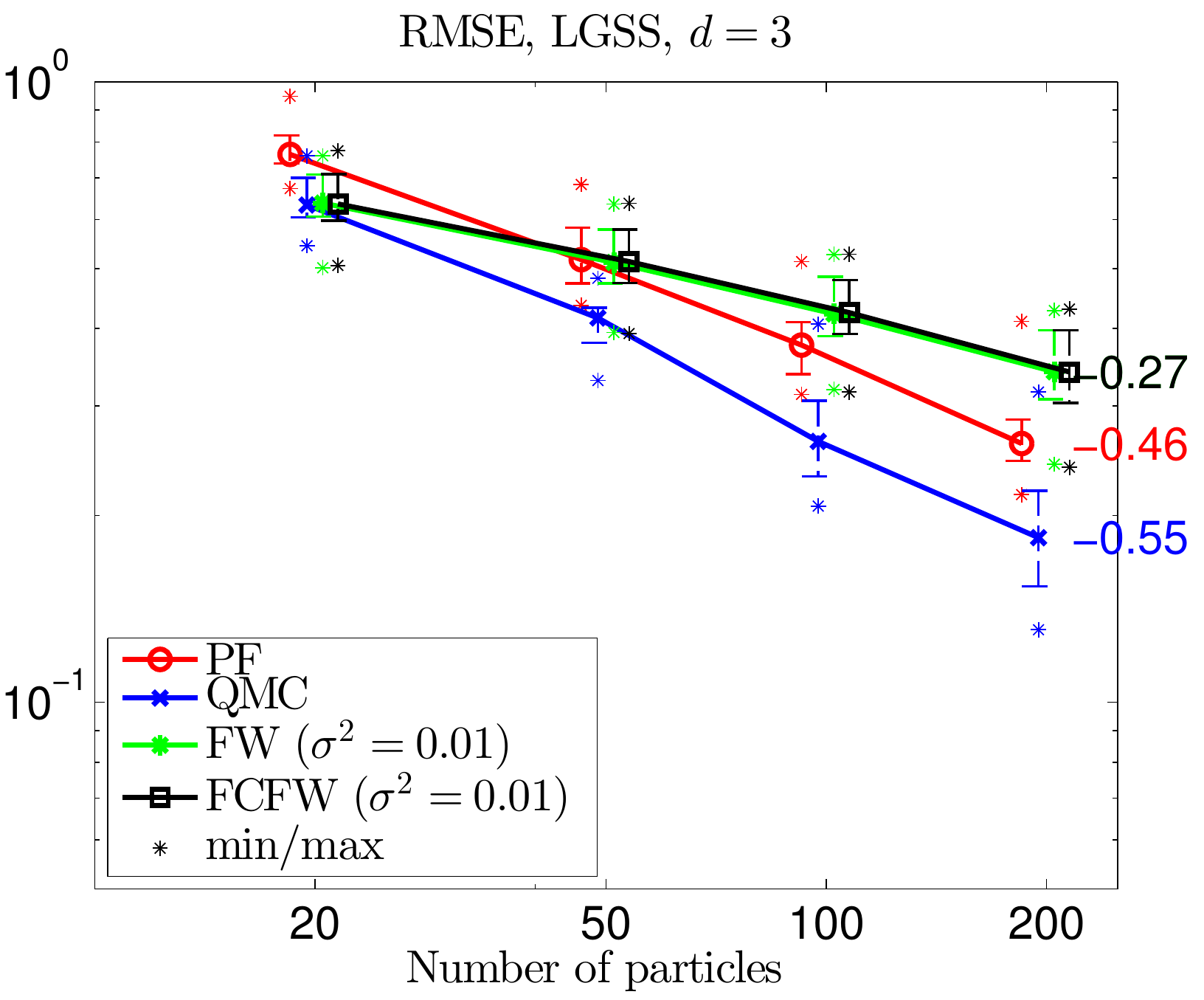} 
  \includegraphics[width = 0.32\textwidth]{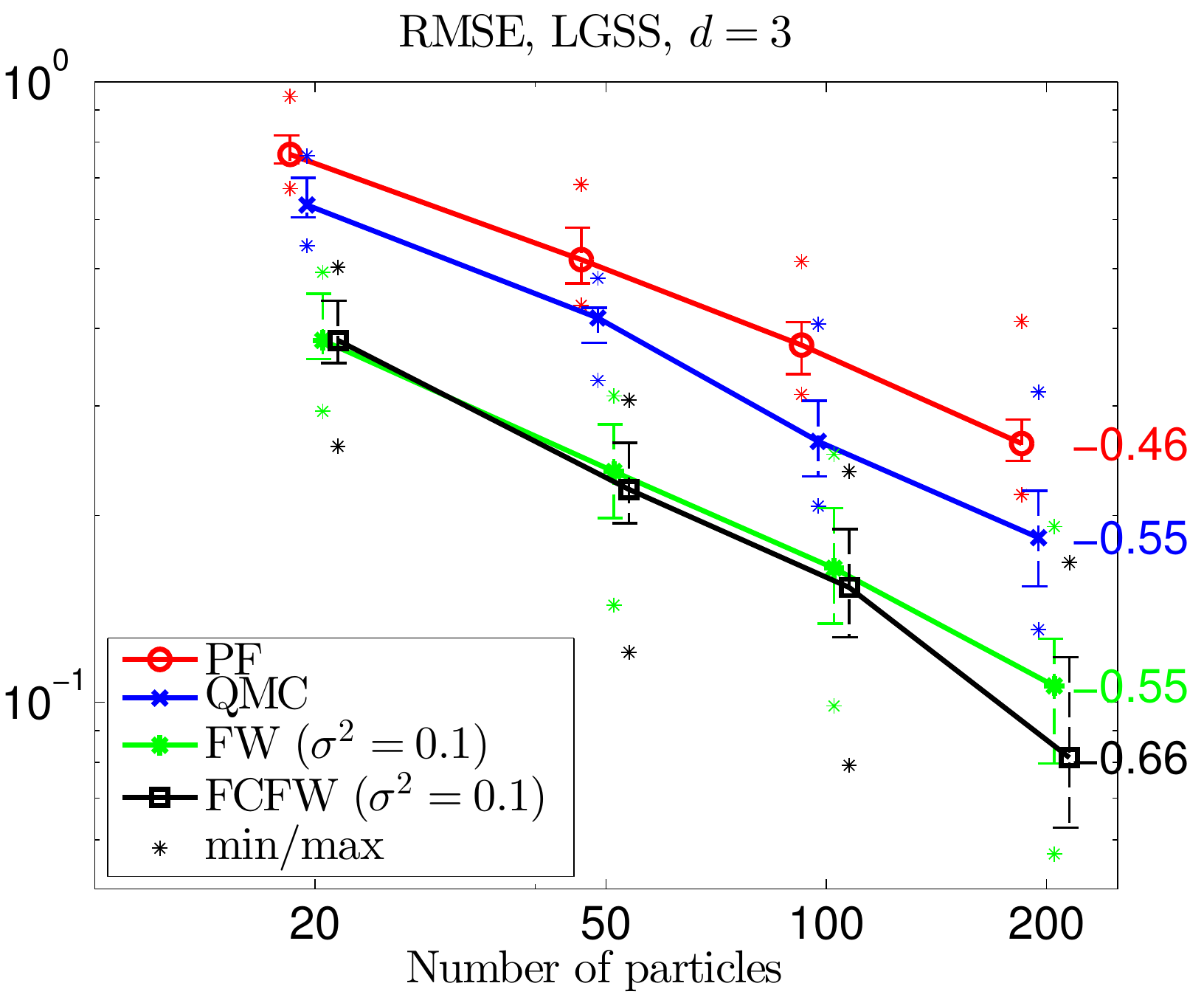} 
  \includegraphics[width = 0.32\textwidth]{./figures/zoomLog_lgss_nx3_3}\\
  \vspace{3mm}
  \includegraphics[width = 0.32\textwidth]{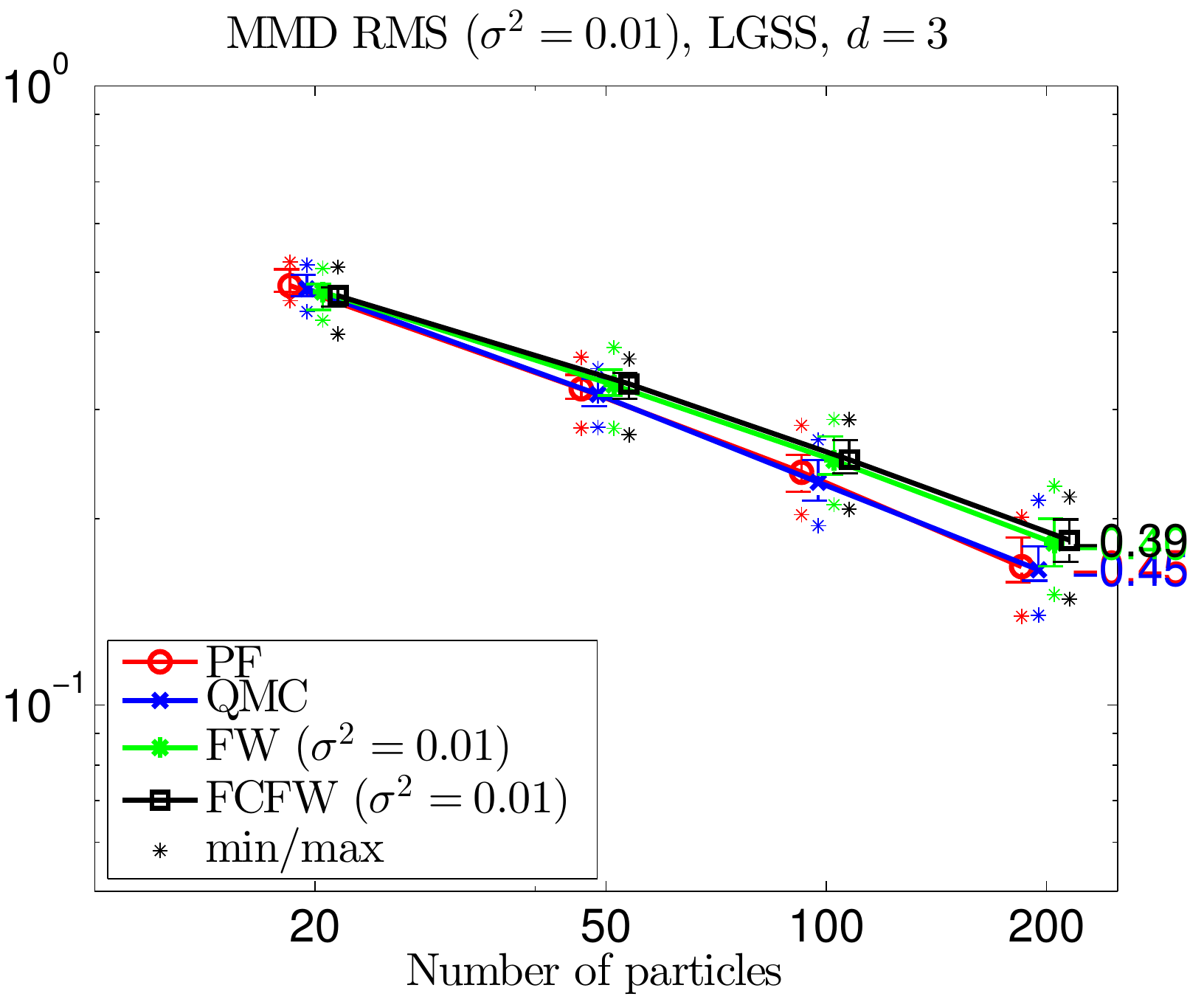} 
  \includegraphics[width = 0.32\textwidth]{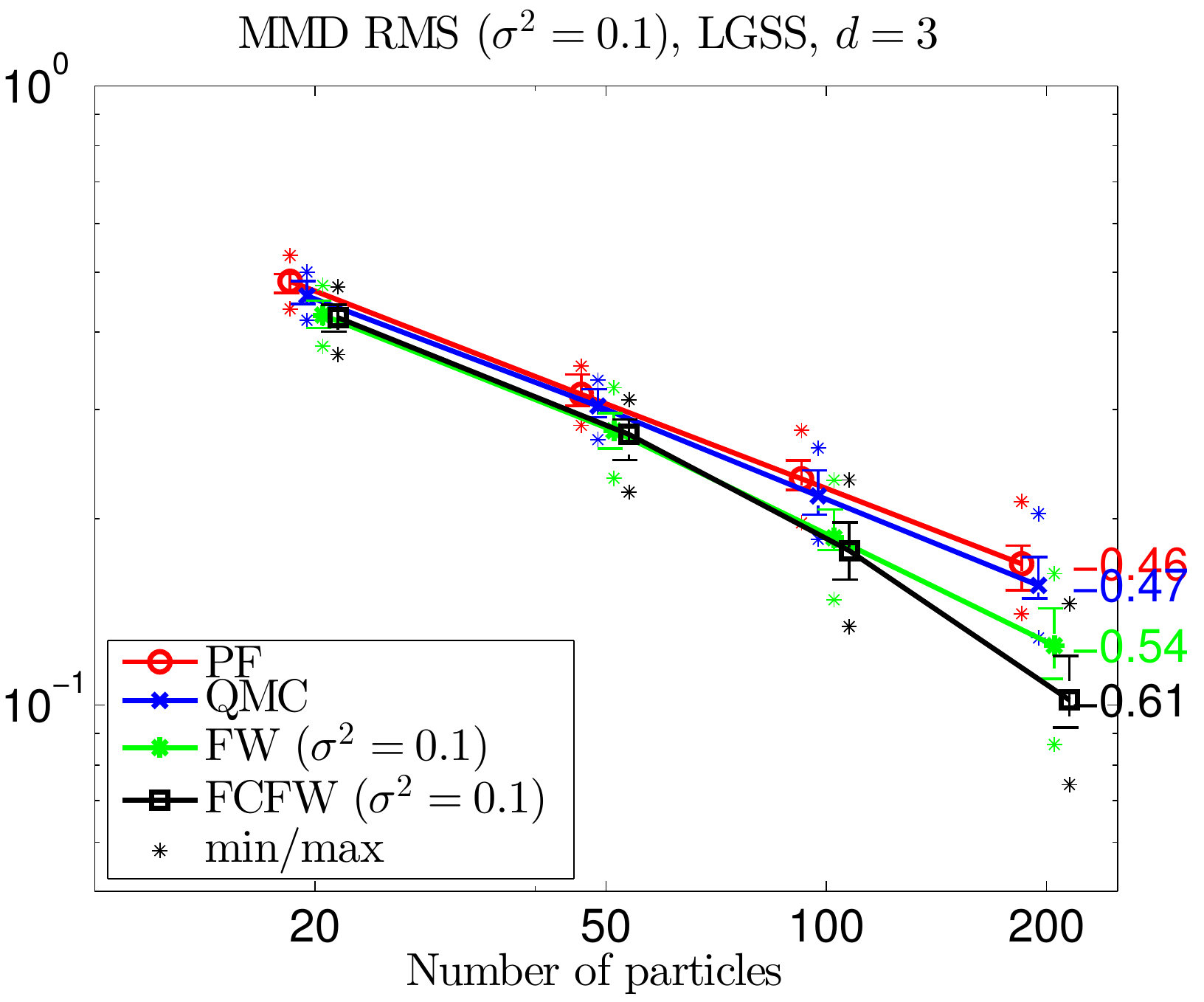} 
  \includegraphics[width = 0.32\textwidth]{./figures/zoomLog_lgss_nx3_6}
  \caption{RMSE (top row) and MMD (bottom row) for LGSS ($d = 3$), using $\sigma^2 = 0.01$, $\sigma^2 = 0.1$, and $\sigma^2 = 1$ (left to right). Note that the MMD definition depends on $\sigma^2$. This is why the MMD curves for PF and QMC are also changing with $\sigma^2$ (but their RMSE ones are not).}
\end{figure}

\begin{figure}[p]
  \centering
  \includegraphics[width = 0.32\textwidth]{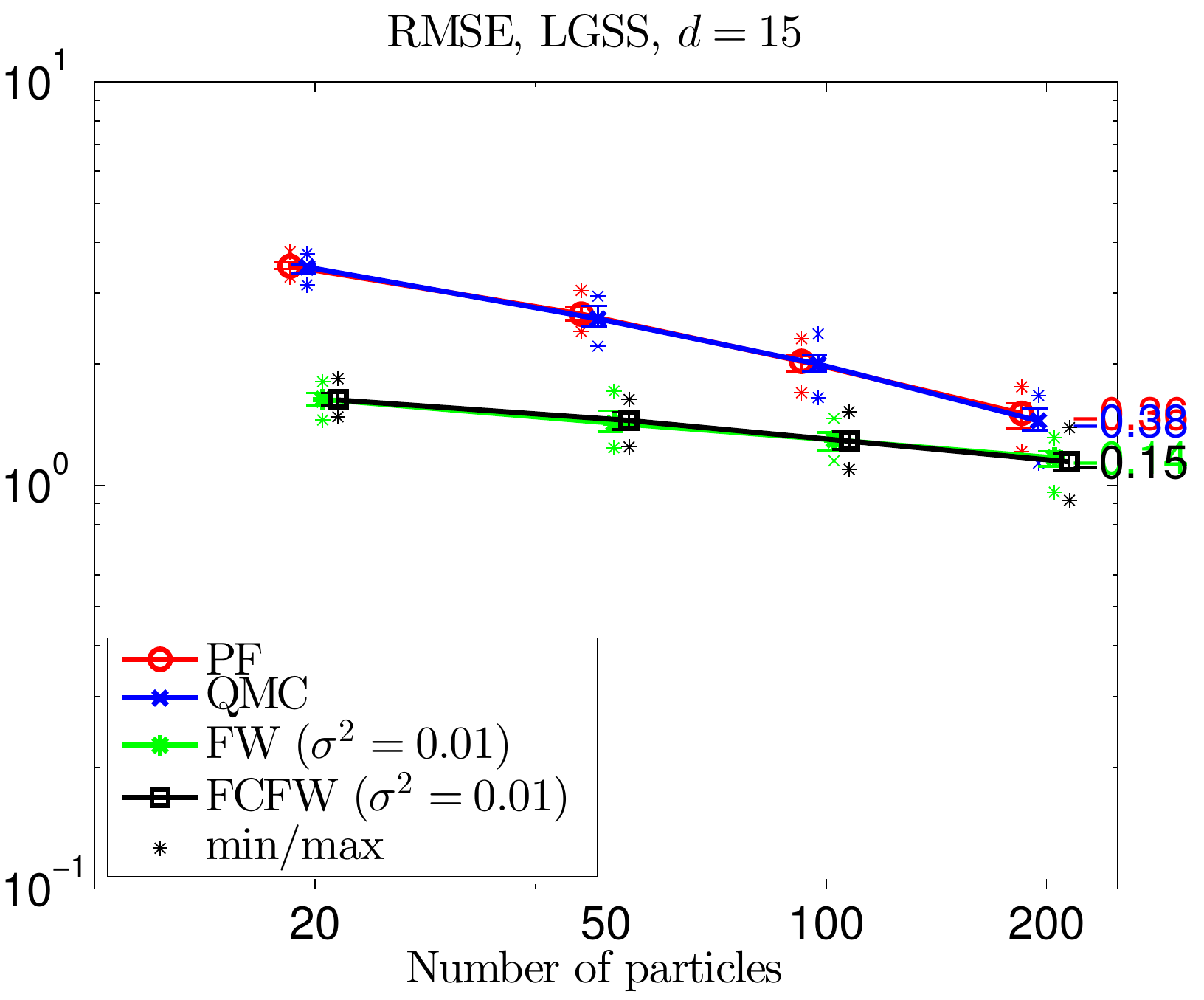} 
  \includegraphics[width = 0.32\textwidth]{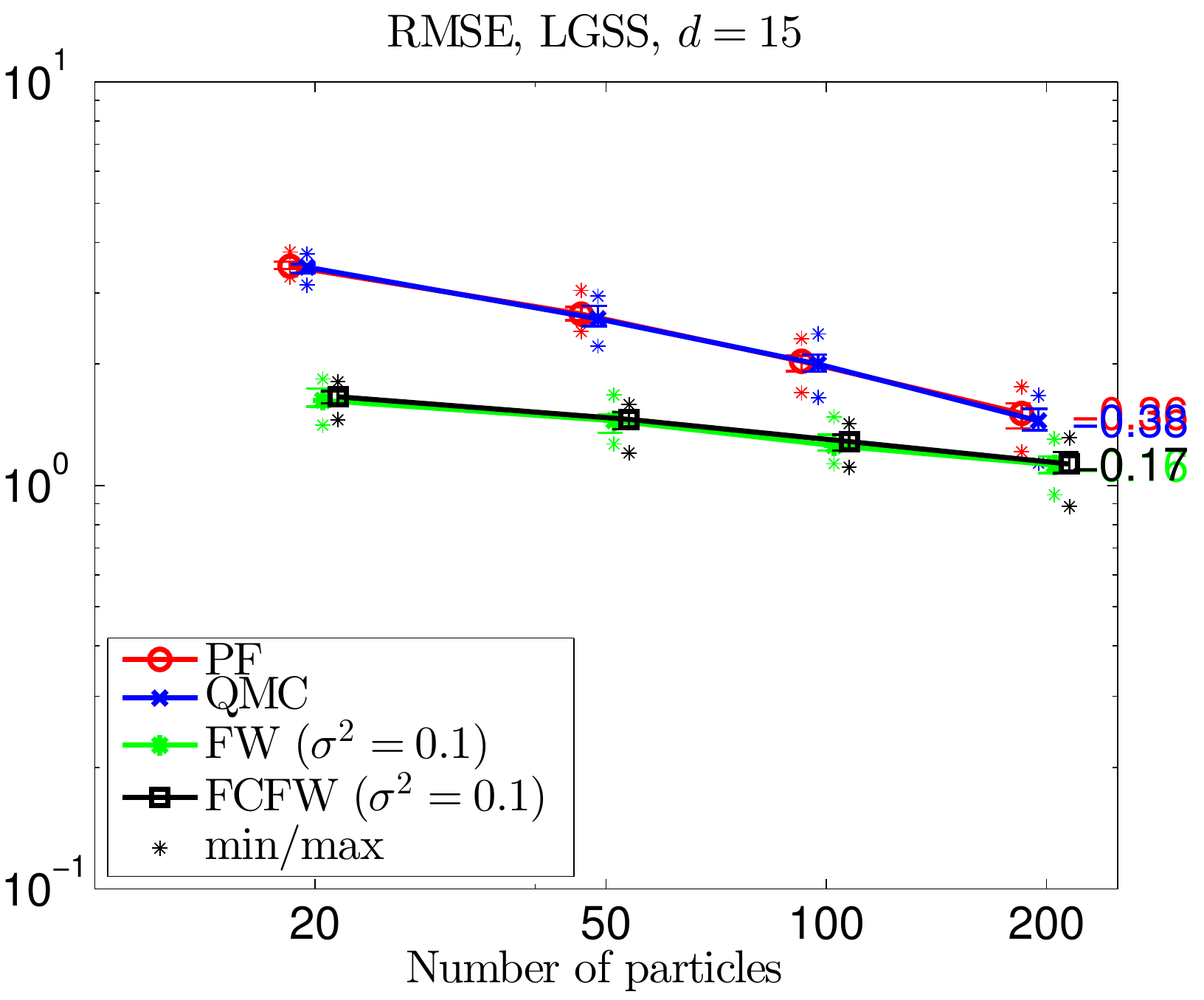} 
  \includegraphics[width = 0.32\textwidth]{./figures/zoomLog_lgss_nx15_3}\\
  \vspace{3mm}
  \includegraphics[width = 0.32\textwidth]{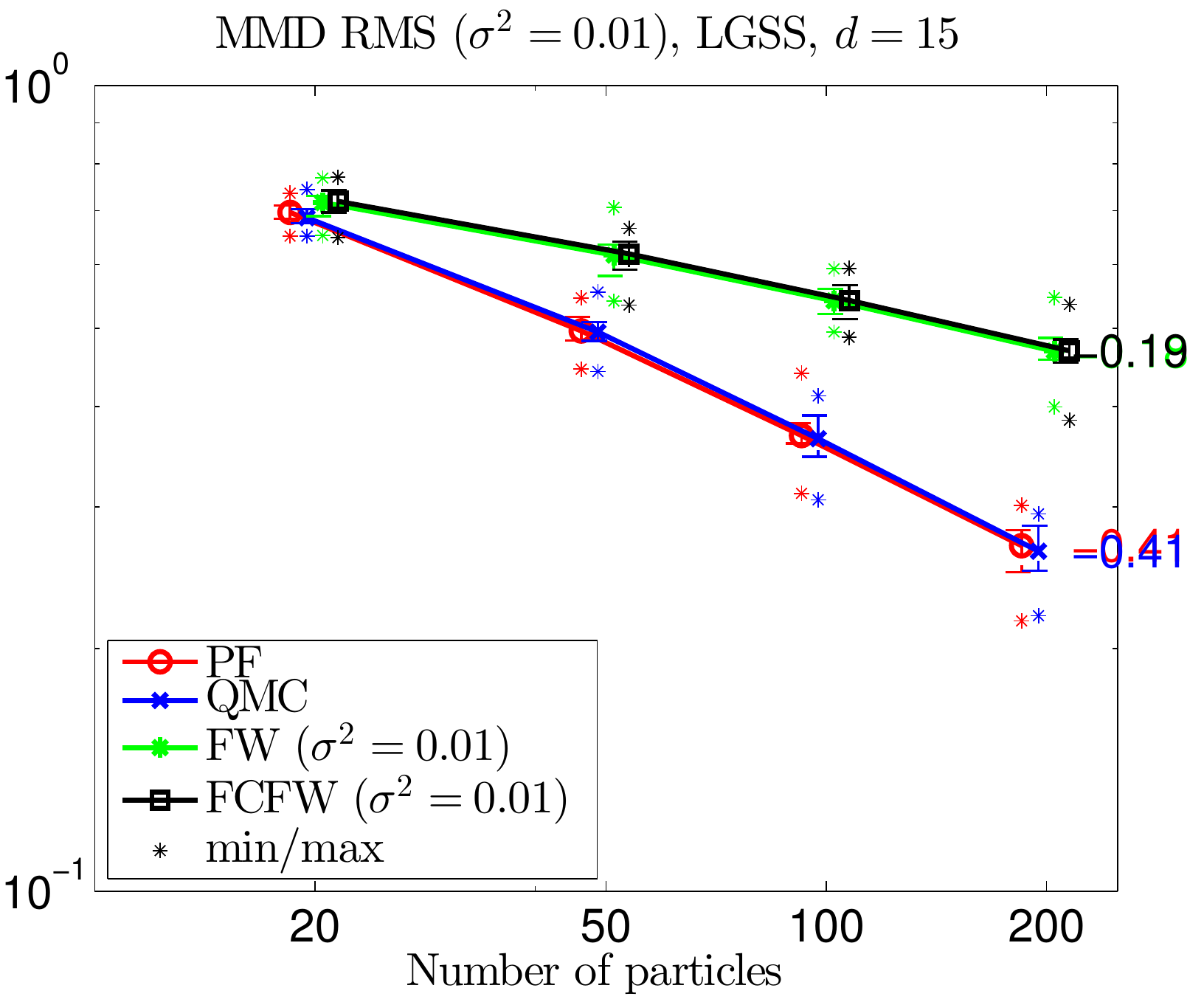} 
  \includegraphics[width = 0.32\textwidth]{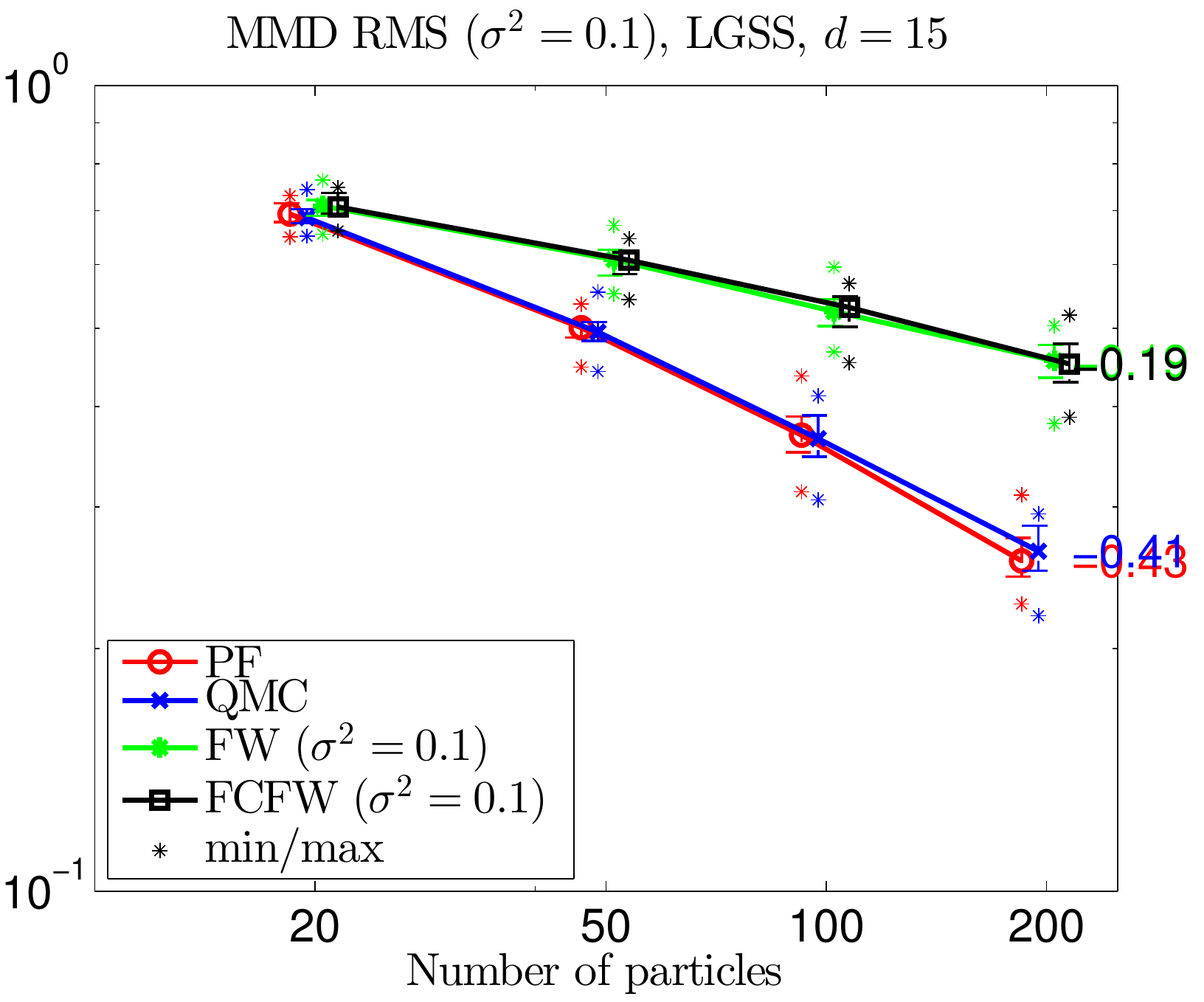} 
  \includegraphics[width = 0.32\textwidth]{./figures/zoomLog_lgss_nx15_6}
  \caption{RMSE (top row) and MMD (bottom row) for LGSS ($d = 15$), using $\sigma^2 = 0.01$, $\sigma^2 = 0.1$, and $\sigma^2 = 1$ (left to right).}
  \label{fig:n}
\end{figure}

\clearpage
\section{Proof sketch for Theorem~\ref{thm:Ct}}
\begin{proofsketch}
We assume that the function $f_t:(x_{t+1},x_{t}) \mapsto  p(y_t|x_t) p(x_{t+1}|x_t)$ is in the tensor product $\F_t \otimes \H_t$, with $\F_t$ defined as in the statement of the theorem. We consider the nuclear norm~\citepsup{jameson1987summing}:
$$ \| f_t \|_{\F_t \otimes \H_t} = \inf_{\{\alpha_i, \beta_i\}_{i=1}^\infty} \sum_i \| \alpha_i \|_{\F_t}  \| \beta_i \|_{\H_t}$$
over all possible decompositions $\{\alpha_i, \beta_i\}_{i=1}^\infty$ of $f_t$ such that, for all $x_t, x_{t+1}$
$$
f_t(x_{t+1},x_{t})  = \sum_{i} \alpha_i(x_{t+1}) \beta_i(x_t).
$$
In the following, let $\{\alpha_i, \beta_i\}_{i=1}^\infty$ be such a decomposition for $f_t$. We have
 \BEAS
 (F_t \nu) (x_{t+1}) & = &  \int \underbrace{p(x_{t+1}|x_t) p(y_t|x_t)}_{f_t(x_{t+1},x_t)} d\nu(x_t)  \in \rb \\
 \mu(F_t v) & = & \int (F_t\nu)(x_{t+1}) \Phi(x_{t+1}) dx_{t+1} \in \H_{t+1} .
 \EEAS
Now we have that $\|\mu(F_t \nu)\|_{\H_{t+1}} = \sup_{\|h\|_{\H_{t+1}} = 1} | \innerProd{h}{\mu(F_t \nu)}|$, so we consider for some $h \in \H_{t+1}$:
\BEAS
\innerProd{h}{\mu(F_t \nu)} & = & 
 	 \int (F_t\nu)(x_{t+1}) h(x_{t+1}) dx_{t+1} \mbox{ (by linearity and reproducing property)} \\
 & =  &  \int \bigg( \int f_t(x_{t+1},x_t) d\nu(x_t) \bigg) h(x_{t+1}) dx_{t+1} 
\\
 & =  &  \int  \int f_t(x_{t+1},x_t)   h(x_{t+1}) dx_{t+1} d\nu(x_t)  \mbox{ (Fubini's theorem)} \\
 & =  &  \int  \int \bigg( \sum_i \alpha_i(x_{t+1}) \beta_i(x_t) \bigg) h(x_{t+1}) dx_{t+1} d\nu(x_t) \\
   & = & \sum_i \bigg( \underbrace{\int  \beta_i(x_t) d\nu(x_t)}_{\E_\nu[\beta_i]} \bigg) \bigg(\int \alpha_i(x_{t+1}) h(x_{t+1}) dx_{t+1} \bigg)  \mbox{ (Fubini's theorem)} .
 \EEAS
By~\eqref{eq:IntegralError}, we have that $| \E_{\nu}[\beta_i]| \leq \|\beta_i\|_{\H_t} \|\mu(\nu) \|_{\H_t}$. Thus we have:
\BEAS  
   \|\mu(F_t \nu)\|_{\H_{t+1}} = \sup_{\|h\|_{\H_{t+1}} = 1} | \innerProd{h}{\mu(F_t \nu)}| & = & \sup_{\|h\|_{\H_{t+1}} = 1} \left| \sum_i \E_\nu[\beta_i] \bigg(\int \alpha_i(x_{t+1}) h(x_{t+1}) dx_{t+1} \bigg) \right| \\
  &\leq & \sum_i \underbrace{\big|\E_\nu[\beta_i]\big|}_{\leq \|\beta_i\|_{\H_t} \|\mu(\nu) \|_{\H_t}} 
  	\bigg( \underbrace{ \sup_{\|h\|_{\H_{t+1}} = 1} \left| \int \alpha_i(x_{t+1}) h(x_{t+1}) dx_{t+1} \right|}_{:=\, \|\alpha_i\|_{\F_t}} \bigg) \\
  &\leq & \left( \sum_i \|\alpha_i\|_{\F_t} \|\beta_i\|_{\H_t} \right) \|\mu(\nu) \|_{\H_t} . 
  \EEAS

This inequality was valid for any expansion $\{\alpha_i, \beta_i\}_{i=1}^\infty$ for $f_t$, and thus we can take the infimum of the upper bound over all possible expansions to get:
$$
	 \|\mu(F_t \nu)\|_{\H_{t+1}} \leq \| f_t \|_{\F_t \otimes \H_t}  \|\mu(\nu) \|_{\H_t}
$$
as we wanted to prove.
\end{proofsketch}

\section{Special case for the Gaussian kernel}

In this section, we explore what form $\|\cdot\|_\F$ takes for the Gaussian kernel. We then show that $\| f_t \|_{\F \otimes \H}$ is finite for a simple one-dimensional linear Gaussian \ssm as long as $\sigma$ is small enough (and thus $\H$ is big enough, as the size of $\H$ increases when $\sigma$ decreases for the Gaussian kernel).

For the Gaussian kernel $\kernel(x,y) = \exp( - \| x - y\|_2^2 / 2 \sigma^2) =: q(x-y)$, its Fourier transform is
$$
\hat{q}(\omega) = \int e^{ -i x^\top \omega} q(x) = (2 \pi)^{d/2} \sigma^d e^{ -  \| \omega\|^2 \sigma^2 / 2}
$$
and
$$
\| h \|_\H^2  = \frac{1}{(2\pi)^d} \int  \frac{ | \hat{h} (\omega) |^2}{ \hat{q}(\omega)} d \omega .
$$
Moreover, 
$$
\langle \alpha, h \rangle_{L^2} =  \frac{1}{(2\pi)^d} \langle \hat{\alpha}, \hat{h} \rangle_{L^2} 
=    \frac{1}{(2\pi)^d} \int 
 \frac{  \hat{h} (\omega) }{ \hat{q}(\omega)^{1/2} }  \hat{q}(\omega)^{1/2} \hat{\alpha}(\omega) d \omega .
$$
By applying Cauchy-Schwartz on $L^2$ on the RHS, this leads to
$$
|\langle \alpha, h \rangle_{L^2} |
\leqslant \| h\|_\H \bigg(
\frac{1}{(2\pi)^d} \int     | \hat{\alpha} (\omega) |^2   \hat{q}(\omega) d \omega
\bigg)^{1/2} .
$$

Thus, we can take the function space:
$$
\| \alpha \|_\F^2
=  \frac{1}{(2\pi)^d} \int     | \hat{\alpha} (\omega) |^2   \hat{q}(\omega) d \omega
=  \frac{\sigma^d}{(2\pi)^{d/2}} \int     | \hat{\alpha} (\omega) |^2    e^{ -  \| \omega\|^2 \sigma^2 / 2}  d \omega .
$$
This allows for quite peaky distributions for the dynamics, as Diracs are authorized (with constant Fourier transform).

To compute an upper bound on $\| f_t \|_{\F \otimes \H}$, we simply need to find a decomposition of
$p(y_t|x_t) p(x_{t+1}|x_t) $ as a sum of terms $\alpha_i(x_{t+1}) \beta_i (x_t)$ and bound the appropriate norms of $\alpha_i$ and $\beta_i$. 
 
 We do this for a special case in the following section.

  \subsection{Bound for one-dimensional Gaussian distribution}
We assume that $x_{t+1} \in \rb^d$ and $y_t \in \rb^m$, and that
\BEAS
p(x_{t+1}|x_t) & = &  \frac{1}{(\sqrt{2\pi}\tau)^d} e^{ - \frac{1}{2 \tau^2} \| x_{t+1} - A x_t\|_2^2} \\
p(y_t |x_t) & = &  \frac{1}{(\sqrt{2\pi}\upsilon)^m} e^{ - \frac{1}{2 \upsilon^2} \| y_t -  B x_t\|_2^2}
\EEAS

We only do the proof for $d=m=1$ and $y_t=0$ (constant observations) to make the proof simpler. We conjecture that similar results hold more generally. We use the Mehler formula for $w$ such that $\frac{2w}{1-w^2} = \frac{1}{\tau^2}$~\citepsup{abramowitz2012handbook}, where $H_n$ is the $n^{\mathrm{th}}$ Hermite polynomial ($H_n(x) := (-1)^n e^{-x^2} \frac{d^n}{dx^n} e^{-x^2}$):
\BEAS
e^{2xyw/(1-w^2)}
& =  & \sqrt{1-w^2} e^{  (x^2+y^2) w^2 / ( 1-w^2) } \sum_{n=0}^\infty \frac{1}{n!} (w/2)^n H_n(x) H_n(y) .
\EEAS
 
Thus 
\BEAS
p(x_{t+1}|x_t) p(y_t |x_t) 
& = & 
 \frac{1}{(\sqrt{2\pi}\tau) } 
 e^{ - \frac{1}{2 \tau^2}  x_{t+1}^2
  - \frac{1}{2 \tau^2}  A^2x_{t}^2 }  e^{   (x_{t+1}^2+A^2x_{t}^2) w^2 / ( 1-w^2) } 
   \frac{1}{(\sqrt{2\pi}\upsilon)} e^{  -    \frac{1}{2 \upsilon^2} B^2 x_t^2   } \\
  & & \sqrt{1-w^2} \sum_{n=0}^\infty \frac{1}{n!} (w/2)^n H_n(x_{t+1}) H_n(Ax_t) \\
  & = &  \sqrt{1-w^2} \frac{1}{(\sqrt{2\pi}\tau)}   \frac{1}{(\sqrt{2\pi}\upsilon)}   
   \sum_{n=0}^\infty \frac{1}{n!} (w/2)^n \alpha_n(x_{t+1}) \beta_n (x_t)
\EEAS
with, using $ - \frac{1}{2\tau^2} + \frac{w^2}{1-w^2} = \frac{ w^2 - w}{1-w^2} =  \frac{-w}{1+w}$:
\BEAS
\alpha_n(x_{t+1}) 
& = &  
e^{ -   x_{t+1}^2 w / (1+w) }  H_n(x_{t+1})
\\
\beta_n (x_t) & = & e^{ -   A^2 x_{t}^2 w / (1+w) }  
e^{ -    \frac{1}{2 \upsilon^2} B^2 x_t^2   } 
H_n(A x_{t}) .
\EEAS

We thus now need to compute the norms of $\alpha_n$ and $\beta_n$, by first computing the Fourier transform. We use the representation:
$$
H_n(x) =  \frac{n!}{2 i \pi} \oint e^{-t^2 + 2 xt}    \frac{dt}{t^{n+1}} ,
$$
integrating over a contour around the origin,
which leads to:
\BEAS
\hat{\alpha}_n(\omega) & = &  \int_\rb e^{- i \omega x} H_n(x) e^{-x^2 w / (1+w) } dx \\
& = &  \int_\rb e^{- i \omega x} \frac{n!}{2 i \pi} \oint e^{-t^2 + 2 xt}    e^{-x^2 w / (1+w) } \frac{dt}{t^{n+1}}   dx
\\
& = &  \frac{n!}{2 i \pi}  \oint   e^{-t^2} \bigg( \int_\rb e^{- i \omega x}   e^{+ 2 xt}    e^{-x^2 w / (1+w) }  dx \bigg) \frac{dt}{t^{n+1}}  .   
\EEAS
We may now use
$$
\int_\rb e^{-ax^2 + bx} dx = \sqrt{\frac{\pi}{a}} e^{b^2/4a}
$$
to get, using $-1+ 4(1+w)/4w =  1/w$,
\BEAS
\hat{\alpha}_n(\omega)& = &  
  \frac{n!}{2 i \pi}  \oint   e^{-t^2} \bigg(  
\sqrt{ \frac{\pi(1+w)}{w} }
\exp \bigg( \frac{ ( 2t - i \omega)^2 ( 1 + w)} { 4w} \bigg)
\bigg) \frac{dt}{t^{n+1}}    
\\
& = & \frac{n!}{2 i \pi} \exp \bigg(
- \frac{\omega^2 ( 1+ w)}{4 w}
\bigg) \sqrt{ \frac{\pi(1+w)}{w} }
  \oint   e^{ t^2/ w }  
\exp \bigg( \frac{ -i \omega t ( 1 + w)} {  w} \bigg)
 \frac{dt}{t^{n+1}}    
\\
& = &  \frac{n!}{2 i \pi} \exp \bigg(
- \frac{\omega^2 ( 1+ w)}{4 w}
\bigg) \sqrt{ \frac{\pi(1+w)}{w} }
 w^{-n/2} i^n \oint   e^{ -\tilde{t}^2  }  
\exp \bigg( \frac{  \omega \tilde{t} ( 1 + w)} {  \sqrt{w}} \bigg)
 \frac{d \tilde{t}}{ \tilde{t}^{n+1}}    
\EEAS
using the change of variable $t = i \sqrt{w} \tilde{t}$. This leads to

\BEAS
\hat{\alpha}_n(\omega) & = &  \exp \bigg(
- \frac{\omega^2 ( 1+ w)}{4 w}
\bigg) \sqrt{ \frac{\pi(1+w)}{w} }
  w^{-n/2} i^n H_n \bigg( \frac{     \omega ( 1 + w)} {  2\sqrt{w}} \bigg)
\\
| \hat{\alpha}_n(\omega)|
& \leqslant & 
  \exp \bigg(
- \frac{\omega^2 ( 1+ w)}{4 w}
\bigg) \sqrt{ \frac{\pi(1+w)}{w} }
  w^{-n/2} C \exp  \bigg( \frac{     \omega^2 ( 1 + w)^2} {  8 w} \bigg)  ( 2^n n! \sqrt{\pi} )^{1/2} 
  \\
  & \leqslant & 
  \exp \bigg(
- \frac{  \omega^2 (1-w^2)}{8 w}
\bigg) \sqrt{ \frac{\pi(1+w)}{w} }
  w^{-n/2} C  ( 2^n n! \sqrt{\pi} )^{1/2} 
  \EEAS
  using $H_n(x) \leqslant C \exp(x^2/2) ( 2^n n! \sqrt{\pi} )^{1/2} $, with $C =  \pi^{-1/4}$, and $  - \frac{1+w}{4w} + \frac{(1+w)^2}{8w}= \frac{1+w}{4w} \big(
  -1 + \frac{1+w}{2}
  \big) = - \frac{1-w^2}{8w}$.
  
  We thus have
  \BEAS
  \| \alpha_n \|_\F^2 &  = &  \frac{\sigma}{\sqrt{2 \pi}} \int | \hat{\alpha}_n(\omega)|^2 e^{-\omega^2 \sigma^2/2} d \omega \\
  & \leqslant &   w^{-n} 2^n n! \sqrt{\pi} C^2 { \frac{\pi(1+w)}{w} }
  \frac{\sigma}{\sqrt{2 \pi}} \int \exp \bigg(
- \frac{   \omega^2 (1-w^2)}{4 w}
\bigg)  e^{-\omega^2 \sigma^2/2} d \omega 
\\
 & = &   w^{-n} 2^n n! \sqrt{\pi} C^2 { \frac{\pi(1+w)}{w} }
  \frac{\sigma}{\sqrt{2 \pi}} \frac{\sqrt{ 2\pi}}
	  {  \sigma^2 + ( 1 - w^2)/2w}  \\
 & = &  (  w^{-n} 2^n n! )  \times C(w,\sigma)
  \EEAS
  Moreover,
  \BEAS
  \hat{\beta}_n(\omega)
  & = & \int_\rb e^{- i \omega x} H_n(Ax) e^{-x^2 ( A^2 w/(1+w) + B^2/2 \upsilon^2) } 
  dx \\
  & = &  \frac{n!}{2 i \pi}  \int_\rb e^{- i \omega x}   
   \oint   e^{-t^2}   e^{+ 2 A xt} \frac{dt}{t^{n+1}}
   e^{-x^2 ( A^2 w/(1+w) + B^2/2 \upsilon^2) } 
   dx \\
   & = &   \frac{n!}{2 i \pi}    
   \oint   e^{-t^2}  \bigg( \int_\rb e^{- i \omega x}     e^{+ 2 A xt} 
   e^{-x^2 ( A^2 w/(1+w) + B^2/2 \upsilon^2) } 
dx\bigg) \frac{dt}{t^{n+1}} \\
  & = &   \frac{n!}{2 i \pi}    
   \oint   e^{-t^2}  \bigg(   \sqrt{\frac{\pi}{A^2 w/(1+w) + B^2/2\upsilon^2}}  
   \exp \bigg(
   \frac{  [2At   - i \omega]^2}{4( A^2 w/(1+w) + B^2/2\upsilon^2)}
   \bigg)
   \bigg) \frac{dt}{t^{n+1}} \\
 & = &   \frac{n!}{2 i \pi}   
 \exp \bigg(
   \frac{  - \omega^2   }{4( A^2 w/(1+w) + B^2/2\upsilon^2)}
   \bigg)
   \sqrt{\frac{\pi}{A^2 w/(1+w) + B^2/2\upsilon^2}}   \\
   & & 
   \times
   \oint   e^{-t^2}  \bigg(   
   \exp \bigg(
   \frac{  4A^2t^2 - 4At i  \omega}{4( A^2 w/(1+w) + B^2/2\upsilon^2)}
   \bigg)
   \bigg) \frac{dt}{t^{n+1}} \\
 & = &   \frac{n!}{2 i \pi}     \exp \bigg(
   \frac{ - \omega^2   }{4( A^2 w/(1+w) + B^2/2\upsilon^2)}
   \bigg)
   \sqrt{\frac{\pi}{A^2 w/(1+w) + B^2/2\upsilon^2}}   \\
   & & 
   \times
   \oint   \exp\bigg(  t^2  
   \frac{A^2 - ( A^2 w/(1+w) + B^2/2\upsilon^2)}{( A^2 w/(1+w) + B^2/2\upsilon^2)}
   \bigg)  \bigg(   
   \exp \bigg(
   \frac{  -4Ait \omega}{4( A^2 w/(1+w) + B^2/2\upsilon^2)}
   \bigg)
   \bigg) \frac{dt}{t^{n+1}} \\
& = &   \frac{n!}{2 i \pi}     \exp \bigg(
   \frac{ - \omega^2   }{4( A^2 w/(1+w) + B^2/2\upsilon^2)}
   \bigg)
   \sqrt{\frac{\pi}{A^2 w/(1+w) + B^2/2\upsilon^2}}   \\
   & & 
   \times
   \oint   \exp\bigg(  t^2  
   \frac{A^2 / (1+w) - B^2/2\upsilon^2)}{( A^2 w/(1+w) + B^2/2\upsilon^2)}
   \bigg)  \bigg(   
   \exp \bigg(
   \frac{  -4At  i \omega}{4( A^2 w/(1+w) + B^2/2\upsilon^2)}
   \bigg)
   \bigg) \frac{dt}{t^{n+1}} \\
& = &   \frac{n!}{2 i \pi}   
 \exp \bigg(
   \frac{  -\omega^2   }{4( A^2 w/(1+w) + B^2/2\upsilon^2)}
   \bigg)
   \sqrt{\frac{\pi}{A^2 w/(1+w) + B^2/2\upsilon^2}}   \\
   & & 
   \times
   \oint   \exp\bigg(  t^2  
   \frac{1 - B^2(1+w)/ 2 \upsilon^2A^2)}{  w  + B^2(1+w)/2\upsilon^2A^2 }
   \bigg)  \bigg(   
   \exp \bigg(
   \frac{  -4At i \omega}{4( A^2 w/(1+w) + B^2/2\upsilon^2)}
   \bigg)
   \bigg) \frac{dt}{t^{n+1}} . \\
  \EEAS
  We can now perform the change of variable $ t = i \tilde{t} 
  \sqrt{\frac{w + B^2(1+w)/2\upsilon^2A^2)}{  1  - B^2(1+w)/2\upsilon^2A^2 }} = i \tilde{t} \sqrt{\tilde{w}}$, with 
  $\tilde{w}  >  w$ for $B>0$.

   This leads to
  \BEAS
  \hat{\beta}_n(\omega)
   & = &   \frac{n!}{2 i \pi}    
 \exp \bigg(
   \frac{  -\omega^2   }{4( A^2 w/(1+w) + B^2/2\upsilon^2)}
   \bigg)
   \sqrt{\frac{\pi}{A^2 w/(1+w) + B^2/2\upsilon^2}}   \\
   & & 
   \times \tilde{w}^{-n/2}
   \oint   \exp ( - \tilde{t}^2)   \bigg(   
   \exp \bigg(
   \frac{  4A  \tilde{t} \sqrt{\tilde{w}}    \omega}{4( A^2 w/(1+w) + B^2/2\upsilon^2)}
   \bigg)
   \bigg) \frac{d\tilde{t}}{\tilde{t}^{n+1}} \\
 & = &   
 \exp \bigg(
   \frac{  -\omega^2   }{4( A^2 w/(1+w) + B^2/2\upsilon^2)}
   \bigg)
   \sqrt{\frac{\pi}{A^2 w/(1+w) + B^2/2\upsilon^2}}   \\
   & & 
   \times \tilde{w}^{-n/2}
   H \bigg(
   \frac{  2 A   \sqrt{\tilde{w}}    \omega}{4( A^2 w/(1+w) + B^2/2\upsilon^2)}
   \bigg)\\
   |    \hat{\beta}_n(\omega)|
   & \leqslant & 
   \exp \bigg(
   \frac{  -\omega^2   }{4( A^2 w/(1+w) + B^2/2\upsilon^2)}
   \bigg)
   \sqrt{\frac{\pi}{A^2 w/(1+w) + B^2/2\upsilon^2}}   \\
   & & 
   \times \tilde{w}^{-n/2}
   (2^n n! \sqrt{\pi} )^{1/2} C \exp\bigg(
    \frac{1}{2} \bigg[\frac{  2 A  \tilde{t} \sqrt{\tilde{w}}    \omega}{4( A^2 w/(1+w) + B^2/2\upsilon^2)}\bigg]^2
   \bigg)\\
   & \leqslant & \mbox{ cst } \times   \tilde{w}^{-n/2} (2^n n! \sqrt{\pi} )^{1/2} \exp( -  \square(A,w,B) \omega^2 )
   \EEAS
 with 
 \BEAS
 \square(A,w,B) = \frac{1}{4( A^2 w/(1+w) + B^2/2\upsilon^2)} - \frac{1}{8} \frac{\tilde{w}}{( A^2 w/(1+w) + B^2/2\upsilon^2)^2} .
 \EEAS
 
 We have $ \square(A,w,0) = \frac{1-w^2}{8A^2 w}$. Thus, by continuity if $B$ if small enough, 
 $ \square(A,w,B) >0$.
Note that when we have $B=0$ and $A=1$, we recover previous results for $\alpha_n$.
 
 We thus have
 \BEAS
  \| \beta_n \|_\H^2 &  = &  \frac{1}{(2 \pi)^{3/2} \sigma } \int | \hat{\beta}_n(\omega)|^2 e^{\omega^2 \sigma^2/2} d \omega \\
  & \leqslant & (2^n \tilde{w}^{-n} n!) C(A,B,w,\sigma)
\EEAS
 as long as  $\sigma^2 < 4 \square(A,w,B) $.

Thus, for $\sigma $ small enough, we have the norm less than a constant times
$$
\sum_{n=0}^\infty (w/\tilde{w})^{n/2} < \infty
$$
since $\tilde{w} > w$. This shows that $\sum_{n = 0}^\infty \|  {\alpha}_n\|_{\F}
\|  {\beta}_n\|_{\H} < \infty$ and thus that $C_t$ is finite if the linear dependency parameter $B$ and  the kernel bandwidth $\sigma^2$ are small enough.

\section{Proof for Theorem~\ref{thm:rate}}

\begin{proof}
We recall here that $p_t(x_t) = p(x_t | y_{1:(t-1)})$ (we are in the marginalized setting). In the notation of the algorithm, we have $\tilde{p}_{t+1} = \frac{1}{\hat{W}_t} F_{t} \hat{p}_t$. Let $q_t = p_t Z_{t-1}$ be the un-normalized marginalized predictive distribution (and similarly, $\hat{q}_t = \hat{p}_t \hat{Z}_{t-1}$). We thus have $\tilde{p}_{t+1} = \frac{1}{\hat{Z}_t} F_{t} \hat{q}_t$.
We use the metric inequality (as well as the linearity of the MMD in each of its argument as it is related to the RKHS norm; so a scalar multiplication of a distribution can be taken out of the MMD):
$$
\MMD(\hat{p}_{t+1}, p_{t+1}) %
 \leq \underbrace{\MMD(\hat{p}_{t+1}, \frac{1}{\hat{Z}_t}F_t\hat{q}_t ) }_{\textrm{(I) FW error } := \,\, \hat{\epsilon}_{t+1}} + 	\underbrace{\MMD(\frac{1}{\hat{Z}_t}F_t\hat{q}_t, \frac{1}{Z_{t}}F_t\hat{q}_t)}_{\textrm{(II) Normalization error}} + 
 	\frac{1}{Z_{t}} \underbrace{\MMD(F_t\hat{q}_t, F_t q_t)}_{\textrm{(III) Initialization error}}.
$$
The term (I) is the algorithmic Frank-Wolfe error $\hat{\epsilon}_{t+1} := \MMD(\hat{p}_{t+1}, \tilde{p}_{t+1})$. The term (II) is the normalization error which can be bounded as follows:
$$
\MMD(\frac{1}{\hat{Z}_t}F_t\hat{q}_t, \frac{1}{Z_{t}}F_t\hat{q}_t) = 
	\underbrace{\|\frac{1}{\hat{Z}_t} \mu(F_t \hat{q}_t) \|_\H}_{(A) \,\, \leq R} \frac{1}{Z_t} 
	\underbrace{|Z_t - \hat{Z}_t|}_{(B) \,\, \leq \|o_t\|_\H \MMD(q_t, \hat{q}_t)} \leq  \frac{R \|o_t\|_\H}{Z_t} \MMD(\hat{q}_t,q_t) .
$$
For inequality (A), we note that $\frac{F_t\hat{p}_t}{\hat{Z}_t}= \tilde{p}_{t+1}$ which is a normalized distribution on $x_{t+1}$, this is why $\| \mu(\tilde{p}_{t+1}) \|_\H \leq R$ as $\|\Phi(x)\|_\H \leq R$ $\forall x \in \X$ by assumption. For inequality (B), we have that $|Z_t - \hat{Z}_t| = |\E_{q_t}[o_t]-\E_{\hat{q}_t}[o_t] | \leq ||o_t||_\H \MMD(q_t, \hat{q}_t)$ by~\eqref{eq:IntegralError} under the assumption that $o_t \in \H$.

Finally, the initialization error term (III) can be bounded by using Theorem~\ref{thm:Ct} (with $\nu = \hat{q}_t - q_t$):
$$
\MMD(F_t\hat{q}_t, F_t q_t) \leq C_t \, \MMD(\hat{q}_t, q_t) ,
$$
where $C_t := \| f_t \|_{F \otimes \H}$.

To control $\MMD(\hat{q}_t, q_t)$, we now only work on the un-normalized distributions:
$$
\MMD(\hat{q}_{t}, q_{t}) \leq 
	\underbrace{\MMD(\hat{q}_{t}, F_{t-1}\hat{q}_{t-1} ) }_{ :=  \, \epsilon_{t}} +
	\underbrace{\MMD(F_{t-1}\hat{q}_{t-1}, F_{t-1} q_{t-1})}_{\leq C_{t-1} \, \MMD(\hat{q}_{t-1}, q_{t-1})} \leq \sum_{u=1}^{t} \epsilon_{u} \left( \prod_{k=u}^{t-1} C_k \right),
$$
by repeating the arguments for smaller $t$'s and unrolling the recursion (and recall that $\prod_{u=t}^{t-1} (\cdot) = 1$ by convention).

Combining the three terms, we thus get:
\begin{equation} \label{eq:MMDnormalized}
\MMD(\hat{p}_{t+1}, p_{t+1}) \leq
	\hat{\epsilon}_{t+1} + (R \|o_t\|_\H + C_t) \sum_{u=1}^{t} \frac{\epsilon_{u}}{Z_t} \left( \prod_{k=u}^{t-1} C_k \right).
\end{equation}
Finally, we transform back the $\epsilon_t$ errors in the algorithmic quantities $\hat{\epsilon}_t$ that the FW algorithm measures: 
$$
\hat{\epsilon}_{t+1} = \MMD(\hat{p}_{t+1}, \frac{1}{\hat{W}_t}F_t\hat{p}_t )
= \MMD(\frac{1}{\hat{Z}_t} \hat{q}_{t+1}, \frac{1}{\hat{W}_t}F_t \frac{\hat{q}_t}{\hat{Z}_{t-1}} ) = \frac{1}{\hat{Z}_t} \MMD(\hat{q}_{t+1}, F_t\hat{q}_t ) = \frac{1}{\hat{Z}_t} \epsilon_{t+1} .
$$
And so we can rewrite:
$$
\frac{\epsilon_{u}}{Z_t} = \hat{\epsilon}_u \frac{\hat{Z}_{u-1}}{Z_t} = 
	\hat{\epsilon}_u \frac{1}{W_t} \left( \prod_{k=u}^{t-1} \frac{1}{W_k} \right) 
		\underbrace{\left( \prod_{k=1}^{u-1} \frac{\hat{W}_k}{W_k} \right)}_{:= \, \teleConstant_u} .
$$
We expect $\teleConstant_u = \prod_{k=1}^{u-1} \frac{\hat{W}_k}{W_k} \approx 1$ as the errors on the normalization constants could hopefully go in both direction and thus cancel each other, though in the worst case it could also grow with $u$. Substituting back in~\eqref{eq:MMDnormalized}, we get what we wanted to prove:
\begin{equation} \label{eq:MMDtight}
\MMD(\hat{p}_{t+1}, p_{t+1}) \leq \hat{\epsilon}_{t+1} + \frac{(R \|o_t\|_\H + C_t)}{W_t} 
	\sum_{u=1}^{t} \teleConstant_u \, \hat{\epsilon}_{u} \left( \prod_{k=u}^{t-1} \frac{C_k}{W_k} \right).
\end{equation}
\end{proof}

\begin{remark}[Bound for $\hat{Z}_t$] For parameter estimation in a HMM, one would also be interested in the quality of approximation for $Z_t$. We note that inequality (B) also gives us a bound on the relative error of our estimate $\hat{Z}_t$ for the normalization constant:
$$
\frac{|Z_t - \hat{Z}_t|}{Z_t} \leq \frac{\|o_t\|_\H}{W_t} \sum_{u=1}^{t} \teleConstant_u \, \hat{\epsilon}_{u} \left( \prod_{k=u}^{t-1} \frac{C_k}{W_k} \right).
$$

\end{remark}
\vspace{1em}

\begin{remark}[Bound for joint predictive distribution $p_t^J$]
To be more precise, we could have used the notation $\H_t$ and $\MMD_t$ to be explicit that the RKHS considered was for functions of $x_t$. For example Theorem~\ref{thm:Ct} really says that $\MMD_{t+1}(F_t\hat{q}_t, F_t q_t) \leq C_t \, \MMD_t(\hat{q}_t, q_t)$. But since $\H_t = \H$ (in the isomorphism sense) for all $t$, we did not have to worry about this. On the other hand, as $\H_t$ contains functions of $x_t$ only, we have that $\mu(p_t)$ is the same whether $p_t$ is the marginalized or the \emph{joint} predictive distributions $p_t^J$ (as for the joint, the expectation in the mean map definition will marginalize out the variables $x_{1:(t-1)}$ as they do not appear in $\H_t$). This means that if we consider the \emph{joint forward transformation} $F_t^J$ on a joint measures $\nu^J$ on $x_{1:t}$: $(F_t^J \nu)(x_{t+1}, x_{1:t}) := p(x_{t+1} | x_t) p(y_t | x_t) \nu^J(x_{1:t})$, i.e. $\tilde{p}_{t+1}^J = \frac{1}{W_t} F_t^J p^J_{t}$ (now in the joint sense), then we have $\mu(F_t p_t) = \mu(F_t^J p^J_{t})$, and thus Theorem~\ref{thm:Ct} also holds for the \emph{joint} predictive distribution $p^J_{T}$. 
\end{remark}
\vspace{1em}

\begin{remark}[Bound without $\teleConstant_u$]
The disadvantage of the bound~\eqref{eq:MMDtight} is the presence of the quantity $\teleConstant_u$ for which we did not provide an explicit upper bound (though we would expect it to be close to 1). To get an explicit upper bound for the error, we can repeat a similar argument but always working with the normalized quantities:
$$
\MMD(\hat{p}_{t+1}, p_{t+1}) \leq 
	\underbrace{\MMD(\hat{p}_{t+1}, \frac{1}{\hat{W}_t}F_t\hat{p}_t ) }_{\textrm{(I) FW error} \, := \, \hat{\epsilon}_{t+1}} +
 	\underbrace{\MMD(\frac{1}{\hat{W}_t}F_t\hat{p}_t, \frac{1}{W_{t}}F_t\hat{p}_t)}_{\textrm{(II) Normalization error}} + 
 	\frac{1}{W_{t}} \underbrace{\MMD(F_t\hat{p}_t, F_t p_t)}_{\textrm{(III) Initialization error}}.
$$
The term (II) is the normalization error which can bounded similarly as before as:
$$
\MMD(\frac{1}{\hat{W}_t}F_t\hat{p}_t, \frac{1}{W_{t}}F_t\hat{p}_t) = 
	\underbrace{\|\frac{1}{\hat{W}_t} \mu(F_t \hat{p}_t) \|_\H}_{(A) \,\, \leq R} \frac{1}{W_t} 
	\underbrace{|W_t - \hat{W}_t|}_{(B) \,\, \leq \|o_t\|_\H \MMD(p_t, \hat{p}_t)} \leq  \frac{R \|o_t\|_\H}{W_t} \MMD(\hat{p}_t,p_t) .
$$
Similarly as before, we also have for (III) that by Theorem~\ref{thm:Ct} (with $\nu = \hat{p}_t -  p_t $) that $\MMD(F_t\hat{p}_t, F_t p_t) \leq C_t \, MMD(\hat{p}_t, p_t)$. Combining the three terms, we get:
\begin{equation} \label{eq:MMDloose}
\MMD(\hat{p}_{t+1}, p_{t+1}) \leq \hat{\epsilon}_{t+1} + \frac{R \|o_t\|_\H + C_t}{W_t} \MMD(\hat{p}_{t}, p_{t}) \leq  \sum_{u=1}^{t+1} \hat{\epsilon}_u \left( \prod_{k=u}^t \tilde{\rho}_k \right) ,
\end{equation}
where $\tilde{\rho}_t := \frac{R \|o_t\|_\H + C_t}{W_t}$, by unrolling the recursion for smaller $t$'s.

The problem with bound~\eqref{eq:MMDloose} is that $\tilde{\rho} > 1$ usually due to the extra term $R \|o_t\|$ in its definition, which is why we preferred the tighter form~\eqref{eq:MMDtight}.
\end{remark}
\vspace{1em}
 
\begin{remark}[Removing the $o_t \in \H$ condition in Theorem~\ref{thm:rate}] \label{rem:general:ot}
We note that the condition $o_t \in \H$ is not really necessary in Theorem~\ref{thm:rate}. If $o_t \notin \H$, we can instead re-derive a similar argument as above but using $Z_t = \E_{F_t q_t} [1]$, where $1$ is the constant unit function (here on $x_{t+1}$). We then have $|Z_t - \hat{Z}_t| \leq \|1\|_{\H'} \MMD'(F_t q_t, F_t \hat{q}_t)$, where $\H'$ is an augmented RKHS to ensure that it contains the constant function $1$. We define $\H' = \H$ if $1 \in \H$ already. If $1 \notin \H$, we define $\H'$ to be the Hilbert sum of the RKHS $\H$ and the one generated by the constant kernel $1$ (and thus $\H'$ is a RKHS with kernel $\kernel' = 1 + \kernel$ where $\kernel$ is the original kernel for $\H$; see \citetsup[Thm. 5]{berlinet2004}). We can show that running the Frank-Wolfe algorithm using the kernel $\kernel'$ yields exactly the same objective values and updates, and thus we can use the space $\H'$ to analyze its behavior: Theorem~\ref{thm:Ct} and an analog of Theorem~\ref{thm:rate} then hold, but with all norms defined with respect to $\H'$ instead.
\end{remark}

\section{Faster rates for FW with approximate vertex search for the MMD objective} \label{app:FWrates}

In this section, we provide the proofs for the rate of convergence for FW on the MMD objective $J(g) := \frac{1}{2} \|g - \mu(p)\|^2_\H$ when an approximate vertex search is used (as mentioned in Appendix~\ref{app:rateRandomSearch}) by extending the proofs from~\citet{Chen2010}. We consider the step-size $\gamma_k = \frac{1}{k+1}$.\footnote{We note that the rate extends to the line-search step-size as well as the improvement at each iteration can only be better in this case considering the proof technique that we use.} We note that the standard step-size for Frank-Wolfe optimization to get a $O(1/k)$ rate is $\gamma_k = \frac{2}{k+2}$.\footnote{We also tried the $\gamma_k = \frac{2}{k+2}$ step-size in the mixture of Gaussians experiment of Section~\ref{sub:MoG}, but it gave similar results as the $\gamma_k =\frac{1}{k+1}$ step-size.} The best rate known for general objectives when using FW with $\gamma_k = \frac{1}{k+1}$ is actually $O(\log(k)/k)$~\citepsup[Bound 3.2]{freund2013FW}. We make use of the specific form of the MMD objective here to prove the $O(1/k)$ rate, as well as the faster $O(1/k^2)$ rate under additional assumptions. For the rest of this section, we use $\| \cdot \|$ to mean $\| \cdot \|_\H$.

\begin{theorem}[Rates for FW-Quad with approximate vertex search] \label{thm:FWapp}
Consider the FW-Quad Algorithm~\ref{alg:FWquad} where an approximate vertex search is used: $\innerProd{g_k - \mu_p}{\bar{g}_{k+1}} \leq \min_{g \in \M} \innerProd{g_k-\mu_p}{g}+\delta$, where $\bar{g}_{k+1} := \Phi(x_{k+1})$ and $\delta \geq 0$. Suppose that $\mu_p$ lies in the strict interior of $\M$ with a radius $r > 0$, i.e. a ball of radius $r$ centered at $\mu_p$ lies within $\M$. Recall that $\max_{g \in \M} \|g\| \leq R$. Then we have the faster rate $O(1/k^2)$ for the objective $J$:
\begin{equation} \label{eq:FastFWRate}
\|g_k - \mu_p \| \leq \frac{1}{k} \frac{2 R^2}{r} + \frac{\delta}{r}.
\end{equation}
If $r=0$ (note that $\mu_p \in \M$), then we can still get a standard $O(1/k)$ FW rate:
\begin{equation} \label{eq:slowFWRate}
\|g_k - \mu_p \|^2 \leq \frac{1}{k} 4 R^2 + \delta.
\end{equation}
\end{theorem}

\begin{proof}
If a ball of a radius $r$ centered at $\mu_p$ lies within $\M$, then we have that:
$$
\min_{g \in \M} \innerProd{g_k - \mu_p}{g - \mu_p} \leq -r \|g_k - \mu_p \|.
$$
So the approximate vertex search yields $\bar{g}_{k+1}$ with the property:
\begin{equation} \label{eq:appSearch}
\innerProd{g_k - \mu_p}{\bar{g}_{k+1} - \mu_p} \leq -r \|g_k - \mu_p \| + \delta.
\end{equation}
By using the FW update $g_{k+1} = \gamma_k \bar{g}_{k+1} + (1-\gamma_k) g_k$ with the $\gamma_k = \frac{1}{k+1}$ step-size, we get:
\begin{align*}
\| g_{k+1} - \mu_p \|^2 &= \| \gamma_k \bar{g}_{k+1} + (1-\gamma_k) g_k - \mu_p \|^2
	= \| \frac{1}{k+1} \bar{g}_{k+1} + \frac{k}{k+1}  g_k - \frac{k+1}{k+1} \mu_p \|^2 \\
	&= \frac{1}{(k+1)^2} \| (\bar{g}_{k+1} - \mu_p) + k (g_k - \mu_p) \|^2 .
\end{align*}
Thus if we let $v_k := k (g_k - \mu_p)$, then we get:
\begin{align}
\| v_{k+1} \|^2 &= \| (\bar{g}_{k+1} - \mu_p) + v_k \|^2  \nonumber \\ 
 &= \| \bar{g}_{k+1} - \mu_p ||^2 + \|v_k\|^2 +2 \innerProd{v_k}{\bar{g}_{k+1}-\mu_p} \nonumber \\
 &\leq 4 R^2 + \|v_k\|^2 + 2(k \delta - r \|v_k\|) = \| v_k \|^2 + 2 \|v_k \| \left[ \frac{2R^2 + k \delta}{\|v_k\|} - r \right]. \label{eq:vk}
\end{align}
The last inequality used the crucial strict interior assumption that yielded~\eqref{eq:appSearch}. Now let $C_k := \frac{1}{r} (2R^2+k \delta)$. Note that $C_{k+1} \geq C_k$. We will now proceed to show by induction that $\|v_k \| \leq C_k$ for $k \geq 1$. Note that the bracket in~\eqref{eq:vk} is negative if and only if $\|v_k\| \geq C_k$ (i.e. $\|v_{k+1}\| \leq \|v_k\|$ in this case), giving the inspiration for the $C_k$ threshold.

First, we have that 
$$ \|v_1\| = \|g_1 - \mu_p \| \leq 2 R \leq 2 R \frac{R}{r} \leq C_1,$$
by using the fact that $\frac{R}{r} \geq 1$ since a ball of radius $r$ fitting in $\M$ implies that the maximum norm $R$ of elements in $\M$ is at least $r$.

Now suppose that $\|v_k \| \leq C_k$, i.e. that $\|v_k \| = \alpha C_k$ for some $\alpha \in [0,1]$. Then~\eqref{eq:vk} becomes:
\begin{align*}
\|v_{k+1}\|^2 \leq \alpha^2 C_k^2 + 2 \alpha C_k [\frac{r C_k}{\alpha C_k} - r ] = \alpha^2 C_k^2 + 2 C_k r [1 - \alpha]. 
\end{align*}
The RHS is a convex function of $\alpha$, and so it is maximized at the boundary of its domain. For $\alpha = 0$, we get $\|v_{k+1}\|^2  \leq 2 C_k r = 4 R^2 + 2 k \delta$. For $\alpha =1$, we get $\|v_{k+1}\|^2 \leq C_k^2$. And thus in general, supposing $\alpha \in [0,1]$, we get that $\|v_{k+1}\|^2 \leq \max\{2 C_k r, C_k^2 \} = C_k^2$ as:
$$
C_k^2 = 4 R^2 \left(\frac{R}{r}\right)^2 + 4 k \delta \left(\frac{R}{r} \right)^2 + k^2 \left(\frac{\delta}{r}\right)^2 \geq 2 C_k r = 4 R^2 + 2 k \delta
$$
using $\frac{R}{r}\geq 1$. This completes the induction step as this means that $\|v_{k+1}\| \leq C_k \leq C_{k+1}$.

Thus we conclude that $\|v_k\| \leq C_k$ for all $k\geq 1$, i.e.
$$
\|g_k - \mu_p\| \leq \frac{1}{k} \frac{2 R^2}{r} + \frac{\delta}{r}.
$$
This shows the faster $O(1/k)$ rate~\eqref{eq:FastFWRate}. If we do not have $\mu_p$ in the strict interior of $\M$, i.e. $r=0$, then we can unroll the inequality~\eqref{eq:vk} to get:
\begin{align*}
\| v_{k+1} \|^2 &\leq  4 R^2 + 2 k \delta + \|v_k\|^2 \\
 &\leq  \sum_{l=1}^k \left(4 R^2 + 2 l \delta \right) + \underbrace{\|v_1\|^2}_{\leq 4R^2} \\
 &\leq (k+1) 4R^2 + \frac{k (k+1) }{2} 2 \delta .
\end{align*}
This thus shows~\eqref{eq:slowFWRate}:
\begin{align*}
\| g_k - \mu_p \|^2 &\leq  \frac{1}{k} 4 R^2 +  \delta.
\end{align*}
This translates to a slower $O(1/\sqrt{k})$ rate on $\|g_k-\mu_p\|$ (with $\sqrt{\delta}$ precision), but note that at least it does not have the $\log(k)$ factor from the rate by~\citetsup{freund2013FW}.
\end{proof}

\paragraph{Consequence for SKH with random search points.}
Going back over the argument from Appendix~\ref{app:rateRandomSearch}, we said that using $M$ random search points in FW-Quad was similar to approximately solving (within $\delta_M R_M$) the linear subproblem for the Frank-Wolfe optimization of $J_M(g) := \frac{1}{2} \|g-\mu(p_M)\|_\H^2$ over the marginal polytope of $\X_M$. We note that the marginal polytope of $\X_M$ is at most $M$-dimensional, and thus, by using a similar argument as in Proposition~1 of~\citet{Bach2012}, we could show that it contains a ball of radius $r_M > 0$ centered at $\mu(p_M)$.\footnote{We note that as $\X_M$ is finite, we do not need to make the additional assumption that the kernel $\kernel$ is continuous unlike in Proposition~1 of~\citet{Bach2012}.} From~\eqref{eq:FastFWRate} in Theorem~\ref{thm:FWapp}, we can thus conclude that:
$$
\|g_k - \mu(p_M)\| \leq \frac{R_M}{r_M} \left( \frac{1}{k} 2 R_M + \delta_M \right) .
$$
This seems to give a faster rate, but the problem is that $r_M$ might shrink at an exponential rate with $M$ if $\H$ is infinite dimensional. Thus even though $\delta_M$ is $O(1/\sqrt{M})$, the dependence of $\delta_M \frac{R_M}{r_M}$ might be worse than the previously quoted rate of $O(1/M^{1/4})$; the latter is thus the general worst-case.

On the other hand, under the additional assumption that $\H$ is finite dimensional and that there is a ball of radius $r > 0$ centered around $\mu_p$ in $\M$ (by using Proposition~1 of~\citet{Bach2012} for example), then for a sufficiently large $M$, we will have $r_M$ close to $r$. Thus for large $M$, we will have $\|g_k - \mu(p_M)\| \lesssim \frac{R}{r} \left(\frac{2 R}{k} + \delta_M \right)$, and thus $\|g_N - \mu_p\| = O\left(\frac{R^2}{r}\left(\frac{1}{N} + \frac{1}{\sqrt{M}} \right) \right)$. This gives an asymptotically faster rate than the $O\left( R \left(\frac{1}{\sqrt{N}}+\frac{1}{\sqrt{R}{M^{1/4}}} \right) \right)$ rate given in Appendix~\ref{app:rateRandomSearch} arising from~\eqref{eq:slowFWRate}, but the constant is worse by a factor of $\frac{R}{r}$.

\bibliographysup{skh_refs}
\bibliographystylesup{abbrvnat}